\documentclass[twoside,11pt]{article}
\usepackage{pgm}
\usepackage[utf8]{inputenc}
\inputencoding{utf8}

\usepackage{hyperref,color,soul,booktabs}
\setulcolor{blue}
\newcommand{\BibTeX}{\textsc{B\kern-0.1emi\kern-0.017emb}\kern-0.15em\TeX}
\usepackage{siunitx}
\usepackage{booktabs}
\usepackage[linesnumbered,algoruled,boxed,lined,ruled]{algorithm2e}
\newcommand{\racircle}{\mathrel{\leftarrow\!\!\!\circ}}
\newcommand{\lacircle}{\mathrel{\circ\!\!\!\rightarrow}}
\newcommand{\rcircle}{\mathrel{-\!\circ}}
\newcommand{\lcircle}{\mathrel{\circ\!-}}
\ShortHeadings{Structural Learning of MVR Chain Graphs via Decomposition}{Javidian and Valtorta}

\begin{document}

\title{A Decomposition-Based Algorithm for Learning the Structure of Multivariate Regression Chain Graphs}
\author{\Name{Mohammad Ali Javidian} \Email{javidian@email.sc.edu}\and
	\Name{Marco Valtorta} \Email{mgv@cse.sc.edu}\\
	\addr Department of Computer Science \& Engineering, University of South Carolina, Columbia, SC, 29201, USA.}

\maketitle

\begin{abstract}We extend the decomposition approach for learning Bayesian networks (BN) proposed by \citep{xie} to learning multivariate regression chain graphs (MVR CGs), which include BNs as a special case. The same advantages of this decomposition approach hold in the more general setting: reduced complexity and increased power of computational independence tests. Moreover, latent (hidden) variables can be represented in MVR CGs by using bidirected edges, and our algorithm correctly recovers any independence structure that is faithful to  an MVR CG, thus greatly extending the range of applications of decomposition-based model selection techniques. Simulations under a variety of settings demonstrate the competitive
performance of our method in comparison with the PC-like algorithm \citep{sp}. In fact, the decomposition-based  algorithm usually outperforms the PC-like algorithm except in running time. The performance of both algorithms is much better when the underlying graph is sparse.
\end{abstract}
\section{Introduction}
\begin{figure}[ht]
    \centering
    \fbox{
    \includegraphics[width=.7\linewidth]{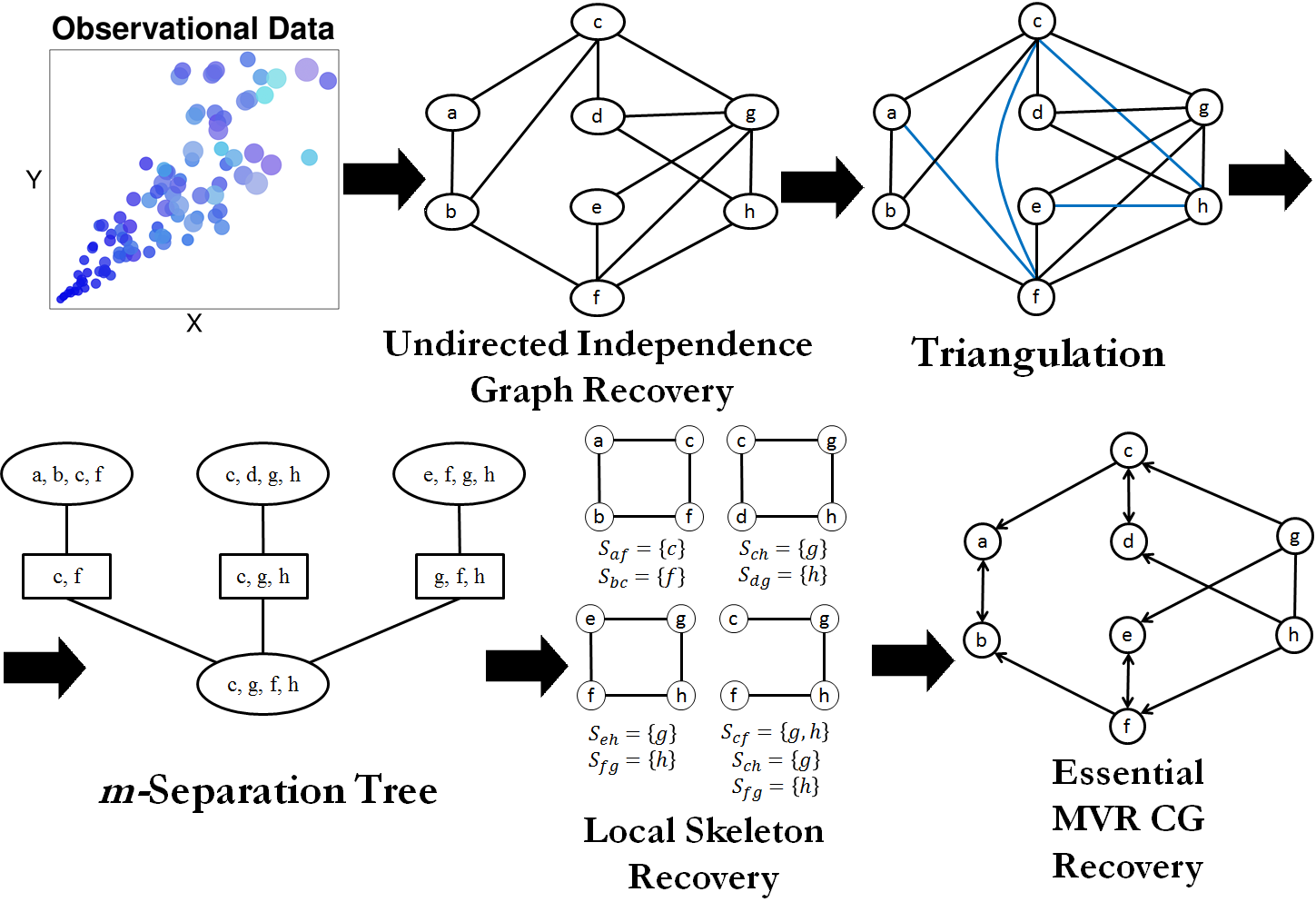}}
    \caption{\small{The procedure for learning the structure of an essential MVR CG from a faithful distribution.}}
    \label{fig:my_label}
\end{figure}
Probabilistic graphical models (PGMs) use graphs, either undirected, directed, bidirected, or mixed, to represent possible dependencies among the variables of a multivariate probability distribution. Two types of graphical representations of distributions are commonly used, namely, Bayesian networks (BNs) and Markov random fields (Markov networks (MNs)), whose graphical parts are, respectively, a directed acyclic graph (DAG) and an undirected graph. Both families encompass the properties of factorization and independencies, but they differ in the set of independencies they can encode and the factorization of the distribution that they induce.

Currently systems containing both causal and non-causal relationships are mostly modeled with directed acyclic graphs (DAGs). An alternative approach is using chain graphs (CGs). Chain graphs may
have both directed and undirected edges under the constraint that there do not exist any semi-directed cycles \citep{d}. So, CGs may contain two types of edges,
the directed type that corresponds to the causal relationship in DAGs and a
second type of edge representing a symmetric relationship \citep{s2}. In
particular, $X_1$ is a direct cause of $X_2$ only if $X_1\to X_2$ (i.e., $X_1$ is a parent
of $X_2$), and $X_1$ is a (possibly indirect) cause of $X_2$ only if there is a directed
path from $X_1$ to $X_2$ (i.e., $X_1$ is an ancestor of $X_2$). So, while the interpretation of the directed edge in a CG is quite clear,
the second type of edge can represent different types of relations and, depending on how we interpret it in the graph, we say that we have different CG interpretations with different separation criteria, i.e. different ways of reading conditional independencies from the graph, and different intuitive meaning behind
their edges. The three following interpretations are the best known in the literature. The first interpretation (LWF) was introduced by Lauritzen,
Wermuth and Frydenberg \citep{lw, f} to combine DAGs and undirected graphs (UGs). The second
interpretation (AMP), was introduced by Andersson, Madigan and Perlman, and also combines DAGs and UGs but with a separation criterion
that more closely resembles the one of DAGs \citep{amp}. The third interpretation,
the multivariate regression interpretation (MVR), was introduced by Cox
and Wermuth \citep{cw1, cw2} to combine DAGs and bidirected (covariance) graphs. 

Unlike in the other CG interpretations, the bidirected edge in MVR CGs has
a strong intuitive meaning. It can be seen to represent one or more hidden
common causes between the variables connected by it. In other words, in an MVR CG any bidirected
edge $X\leftrightarrow Y$ can be replaced by $X\gets H\to Y$ to obtain a Bayesian network representing
the same independence model over the original variables, i.e. excluding the
new variables H. These variables are called hidden, or latent, and have been
marginalized away in the CG model \citep{s}. See \citep{jv1} for details on the properties of MVR chain graphs.

Latent variables, which are often present in practice, cause several complications. First, causal inference based on structural learning (model selection) algorithms such as the PC algorithm \citep{sgs} may be incorrect. Second, if a distribution is faithful\footnote{A distribution $P$ is faithful to DAG $G$ if  any independency in $P$ implies a corresponding $d$-separation property in $G$ \citep{sgs}.} to a DAG, then the distribution obtained by marginalizing on some of the variables may not be faithful to any DAG on the observed variables, i.e., the space of DAGs is not closed under marginalization \citep{cmkr}.  
These problems can be solved by exploiting MVR chain graphs. An example of a situation for which  CG is useful is if we have
a system containing two genes and two diseases caused by these such that Gene1
is the cause of Disease1, Gene2 is the cause of Disease2, and the diseases are correlated. In this case we might suspect the presence of an
unknown factor inducing the correlation between Disease1 and Disease2, such as
being exposed to a stressful environment. Having such a hidden variable results in the
independence model described in the information above. The MVR CG representing the information
above is shown in Figure \ref{Fig:gene} (a)  while the best (inclusion optimal) BN and MN are shown in
Figure \ref{Fig:gene} (b)  and (c), respectively. We can now see that it is only the MVR CG that
describes the relations in the system correctly \citep{Sonntag2015}. 
\begin{figure}[ht]
	\centering
	\includegraphics[scale=.4]{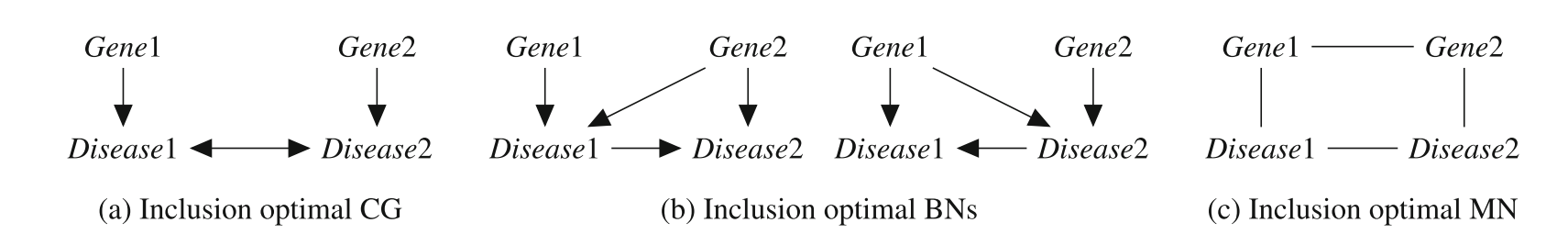}
	\caption{A gene and disease example with MVR CG representation, BN representation and MN
		representation \citep{Sonntag2015}.} \label{Fig:gene}
\end{figure} 

As a result, designing efficient algorithms for learning the structure of MVR chain graphs is an important and desirable task.

Sonntag lists four constraint-based learning algorithms for CGs. All are based on testing if variables
are (conditionally) independent in the data using an independence test, and using this information to deduce the structure of the optimal graph. These algorithms are the PC-like algorithms
\citep{srs, p1, sp}, the answer set programming (ASP) algorithms \citep{p3, sjph}, the LCD algorithm  \citep{mxg} and the
CKES algorithm \citep{psn}. The former two have implementations for all three
CG interpretations, while the latter two are only applicable for LWF CGs \citep{s2}.

In this paper, we propose a decomposition approach for recovering structures of MVR CGs. Our algorithms are natural extensions of algorithms in \citep{xie}. In particular, the
rule in \citep{xie} for combining local structures into a global skeleton is still applicable
and no more careful work (unlike, for example,  algorithms in \citep{mxg}) must be done to ensure a valid combination. Moreover, the method for
extending a global skeleton to a Markov equivalence class is exactly the same as that for Bayesian networks.
The paper is organized as follows:  Section \ref{defs&concepts} gives notation and definitions. In Section 3, we show a condition for decomposing structural learning
of MVR CGs. Construction of $m$-separation trees to be used for decomposition is discussed in Section \ref{construct-trees}. We propose the
main algorithm and then give an example in Section \ref{main-alg} to illustrate our approach for recovering the global structure
of an MVR CG. Section \ref{complexity} discusses the complexity and advantages of the proposed algorithms. Section \ref{evaluation} describes our evaluation setup. Both Gaussian and discrete networks were used.  A comparison with the PC-like algorithm of \citep{sp} was carried out.  Both quality of the recovered networks and running time are reported. Finally, we conclude with
some discussion in Section \ref{discussion}. The proofs of our main results and the correctness of the algorithms are given in Appendices A and B.    

% Section 2 gives notation and definitions. In Section 3, we show a condition for decomposing structural learning
% of MVR CGs. Construction of $m$-separation trees to be used for decomposition is discussed in Section 4. We propose the
% main algorithm and then give an example in Section 5 to illustrate our approach for recovering the global structure
% of  CG. In Section 6, we propose a new algorithm for structural learning of MVR CGs via decomposition by constructing a $JV$-junction tree from domain knowledge or from observed data patterns. Section 7 discusses the complexity and advantages of the proposed algorithms. The proofs of our main results and algorithms are given in Appendix A and B.

\section{Definitions and Concepts}\label{defs&concepts}
In this paper we consider graphs containing both directed ($\to$) and bidirected ($\leftrightarrow$) edges and largely use the terminology of \citep{xie, r2}, where the reader can also find further details. Below we briefly list some of the most central concepts used in this paper.

	If there is an arrow from $a$ pointing towards $b$, $a$ is said to be a parent 
	of $b$. The set of parents of $b$ is denoted as $pa(b)$. If there is a bidirected edge between $a$ and $b$, $a$ and $b$ are said to be neighbors. The set of neighbors of a vertex $a$ is denoted as $ne(a)$. The expressions $pa(A)$ and $ne(A)$ denote the collection of  
	parents and neighbors of vertices in $A$ that are not themselves 
	elements of $A$. The boundary $bd(A)$ of a subset $A$ of vertices is the set of vertices in $V\setminus A$ that are parents or neighbors to vertices in $A$.

	A path of length $n$ from $a$ to $b$ is a sequence $a=a_0,\dots , a_n=b$ of 
	distinct vertices such that $(a_i\to a_{i+1})\in E$, for all $i=1,\dots ,n$. A chain of length $n$ from $a$ to $b$ is a sequence $a=a_0,\dots , a_n=b$ of 
	distinct vertices such that $(a_i\to a_{i+1})\in E$, or $(a_{i+1}\to a_i)\in E$, or $(a_{i+1}\leftrightarrow a_i)\in E$, for all $i=1,\dots ,n$. We say that $u$ is an ancestor of $v$ and $v$
	is a descendant of $u$ if there is a path from $u$ to $v$ in $G$.
	The set of ancestors of $v$ is denoted as $an(v)$, and we define $An(v) = an(v)\cup v$. We apply this definition to sets: $an(X) = \{\alpha | \alpha \textrm{ is an ancestor of } \beta \textrm{ for some } \beta \in X\}$.
	A partially directed cycle in a graph $G$ is a sequence of $n$ distinct vertices $v_1,\dots, v_n (n\ge 3)$,
	and $v_{n+1}\equiv v_1$, such that 
	\begin{itemize}
		\item $\forall i (1\le i\le n)$ either $v_i\leftrightarrow v_{i+1}$ or $v_i\to v_{i+1}$, and
		\item $\exists j (1\le j\le n)$ such that $v_i\to v_{i+1}$.
	\end{itemize}
	A graph with only undirected edges is called an undirected graph (UG). A graph with only
	directed edges and without directed cycles is called a directed acyclic graph (DAG). Acyclic directed mixed graphs, also known as semi-Markov(ian) \citep{pj}
	models contain directed ($\rightarrow$) and bidirected
	($\leftrightarrow$) edges subject to the restriction that there are no directed cycles \citep{r2,er}. A graph that has no partially directed cycles is called \textit{chain graph}.

	A nonendpoint vertex $\zeta$ on a chain is a \emph{collider} on the chain if the edges preceding and succeeding $\zeta$ on the chain have an arrowhead at $\zeta$, that is, $\to \zeta \gets, or \leftrightarrow \zeta \leftrightarrow, or\leftrightarrow \zeta \gets, or\to \zeta \leftrightarrow$. A nonendpoint vertex $\zeta$ on a chain which is not a collider is a noncollider on the chain. A chain between vertices $\alpha$ and $\beta$ in  chain graph $G$ is said to be $m$-connecting given a set $Z$ (possibly empty), with $\alpha, \beta \notin Z$, if: 
	\begin{enumerate}
		\item[(i)] every noncollider on the path is not in $Z$, and 
		\item[(ii)] every collider on the path is in $An_G(Z)$.
	\end{enumerate}
	A chain that is not $m$-connecting given $Z$ is said to be blocked given (or by) $Z$.
	If there is no chain $m$-connecting $\alpha$ and $\beta$ given $Z$, then $\alpha$ and $\beta$ are said to be \emph{m-separated} given $Z$. Sets $X$ and $Y$ are $m$-separated given $Z$, if for every pair $\alpha, \beta$, with $\alpha\in X$ and $\beta \in Y$, $\alpha$ and $\beta$ are $m$-separated given $Z$ ($X$, $Y$, and $Z$ are disjoint sets; $X, Y$ are nonempty). We denote the independence model resulting from applying the $m$-separation criterion to $G$, by $\Im_m$(G). This is an extension of Pearl's $d$-separation criterion \citep{pearl1} to MVR chain graphs in that in a DAG $D$, a chain is $d$-connecting if and only if it is $m$-connecting.

	Two vertices $x$ and $y$ in  chain graph $G$ are said to be collider connected if there is a chain from $x$ to $y$ in $G$ on which every non-endpoint vertex is a collider; such a chain is called a collider chain. Note that a single edge trivially forms a collider chain (path), so if $x$ and $y$ are adjacent in a chain graph then they are collider connected. The augmented graph derived from $G$, denoted $(G)^a$, is an undirected graph with the same vertex set as $G$ such that $$c\--d \textrm{ in } (G)^a \Leftrightarrow c \textrm{ and } d \textrm{ are collider connected in } G.$$

	Disjoint sets $X, Y\ne \emptyset,$ and $Z$ ($Z$ may be empty) are said to be
	$m^\ast$-separated if $X$ and $Y$ are separated by $Z$ in $(G_{an(X\cup Y\cup Z)})^a$. Otherwise $X$ and $Y$ are said to be $m^\ast$-connected
	given $Z$. The resulting independence model is denoted by $\Im_{m^\ast}(G)$.

According to \citep[Theorem 3.18.]{rs} and \citep{jv1}, for  chain graph $G$ we have: $\Im_m(G)=\Im_{m^\ast}(G)$.

Let $\bar{G}_V = (V, \bar{E}_V)$ denote an undirected graph where $\bar{E}_V$ is a set of undirected edges. An undirected edge between
two vertices $u$ and $v$ is denoted by $(u, v)$. For a subset $A$ of $V$, let $\bar{G}_A= (A, \bar{E}_A)$ be the subgraph induced by $A$
and $\bar{E}_A = \{e\in \bar{E}_V | e\in A\times A\} = \bar{E}_V\cap (A\times A)$. An undirected graph is called complete if any pair of vertices is connected by an edge. For an undirected graph, we say that vertices $u$ and $v$ are separated by a set of vertices $Z$ if each path between $u$ and $v$ passes through $Z$. We say that two distinct vertex sets $X$ and $Y$ are separated by $Z$ if and
only if $Z$ separates every pair of vertices $u$ and $v$ for any $u\in X$ and $v\in Y$. We say that an undirected graph $\bar{G}_V$ is
an undirected independence graph (UIG) for  CG $G$ if the fact that a set $Z$ separates $X$ and $Y$ in $\bar{G}_V$ implies that $Z$
$m$-separates $X$ and $Y$ in $G$. Note that the augmented graph derived from  CG $G$, $(G)^a$, is an undirected independence graph for $G$.  We say that $\bar{G}_V$ can be decomposed into subgraphs $\bar{G}_A$ and $\bar{G}_B$ if
\begin{itemize}
	\item[(1)] $A\cup B=V$, and
	\item[(2)] $C=A\cap B$ separates $V\setminus A$ and $V\setminus B$ in $\bar{G}_V$.
\end{itemize}
The above decomposition does not require that the separator $C$ be complete, which is required for weak decomposition defined in \citep{l}. In the next section, we show that a
problem of structural learning of  CG can also be decomposed into problems for its decomposed subgraphs even if
the separator is not complete.

A triangulated (chordal) graph is an undirected graph in which all cycles of four or more vertices have a chord, which is an edge that is not part of the cycle but connects two vertices of the cycle  (see, for example, Figure \ref{Fig:mvr1}). For an
undirected graph $\bar{G}_V$ which is not triangulated, we can add extra (``fill-in") edges to it such that it becomes to be a triangulated
graph, denoted by $\bar{G}_V^t$.

\begin{figure}[ht]
	\centering
	\includegraphics[scale=.4]{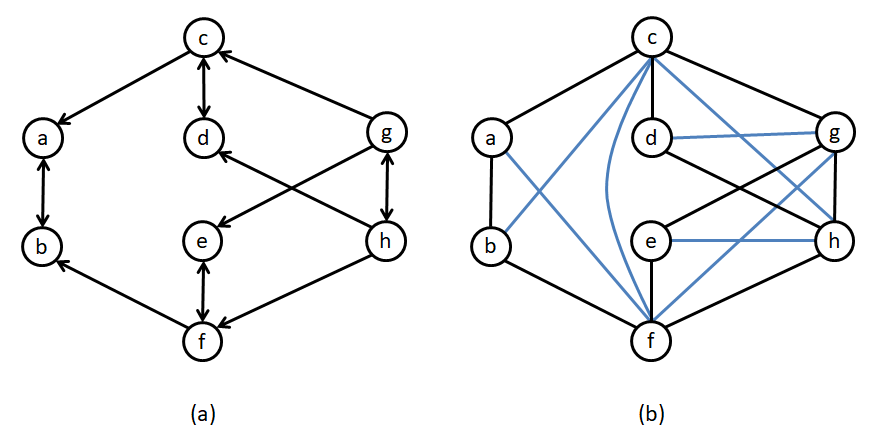}
	\caption{(a) An MVR CG $G$. (b) The augmented graph $G^a$, which is also a triangulated graph $G^t$.} \label{Fig:mvr1}
\end{figure}

Let $X\!\perp\!\!\!\perp Y$ denote
the independence of $X$ and $Y$, and $X\!\perp\!\!\!\perp Y|Z$ (or $\langle X,Y | Z\rangle$) the conditional independence of $X$ and $Y$ given $Z$. In this paper, we assume that all independencies of a
probability distribution of variables in $V$ can be checked by $m$-separations of $G$, called the faithfulness assumption \citep{sgs}. The faithfulness assumption means that all independencies and conditional independencies among variables can be represented by $G$. 

The global skeleton is an undirected graph obtained by dropping direction of  CG. Note that the absence of an
edge $(u, v)$ implies that there is a variable subset $S$ of $V$ such that $u$ and $v$ are independent conditional on $S$, that
is, $u\!\perp\!\!\!\perp v|S$ for some $S\subseteq V\setminus \{u,v\}$ \citep{jv1}. Two MVR CGs over the same variable set are called Markov equivalent if they
induce the same conditional independence restrictions. Two MVR CGs are Markov equivalent if and only if they have the
same global skeleton and the same set of $v$-structures (unshielded colliders) \citep{ws}. An equivalence class of MVR CGs consists of all MVR CGs which
are Markov equivalent, and it is represented as a partially directed graph (i.e., a graph containing directed, undirected, and bidirected edges and no directed cycles) where the directed/bidirected edges represent edges that are common to every MVR CG in it, while the undirected edges represent that any legal orientation of them leads
to a Markov equivalent MVR CG. Therefore the goal of structural learning is to construct a partially directed graph to
represent the equivalence class. A local skeleton for a subset $A$ of variables is an undirected subgraph for $A$ in which
the absence of an edge $(u, v)$ implies that there is a subset $S$ of $A$ such that $u\!\perp\!\!\!\perp v|S$.

Now, we introduce the notion of $m$-separation trees, which is used to facilitate the representation of the decomposition. The concept is similar to the junction tree of cliques and the
independence tree introduced for DAGs as $d$-separation trees in \citep{xie}. Let $C = \{C_1, \dots, C_H \}$ be a collection of distinct variable sets such that for $h = 1,\dots ,H, C_h\subseteq V$.
Let $T$ be a tree where each node corresponds to a distinct variable set in $C$, to be displayed as an oval (see, for example, Figure \ref{Fig:tree1}). The term ‘node’ is used for an $m$-separation tree to distinguish from
the term ‘vertex’ for a graph in general. An undirected edge $e = (C_i,C_j)$ connecting nodes $C_i$ and $C_j$ in $T$ is labeled with a separator $S = C_i\cap C_j$, which is displayed as a rectangle.
Removing an edge $e$ or, equivalently, removing a separator $S$ from $T$ splits $T$ into two subtrees
$T_1$ and $T_2$ with node sets $C_1$ and $C_2$ respectively. We use $V_i$ to denote the union of the
vertices contained in the nodes of the subtree $T_i$ for $i = 1,2$.
 
\begin{definition}\label{septree}
	A tree $T$ with node set $C$ is said to be an $m$-separation tree for  chain graph $G = (V,E)$ if
	\begin{itemize}
		\item $\cup_{C_i\in C}C_i=V$, and
		\item for any separator $S$ in $T$ with $V_1$ and $V_2$ defined as above by removing $S$, we have $\langle V_1\setminus S,V_2\setminus S | S\rangle_G$.
	\end{itemize}
\end{definition}
\begin{figure}[ht]
	\centering
	\includegraphics[scale=.4]{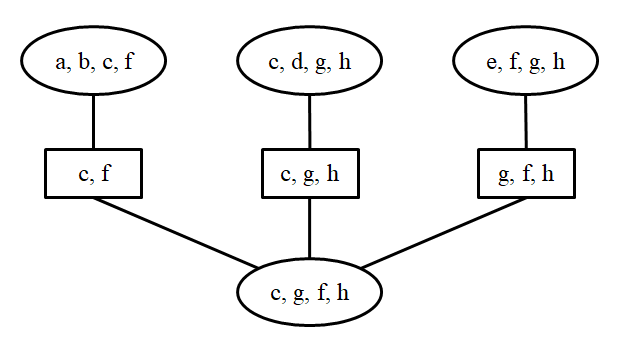}
	\caption{an $m$-separation tree.} \label{Fig:tree1}
\end{figure}

Notice that a separator is defined in terms of a tree whose nodes consist of variable sets, while
the $m$-separator is defined based on  chain graph. In general, these two concepts are not related, though for an $m$-separation tree its separator must be some corresponding $m$-separator in the underlying MVR chain graph. The definition of $m$-separation trees for MVR chain graphs is similar to that of junction trees of cliques,
see \citep{cdls,l}. Actually, it is not difficult to see that a junction tree
of  chain graph $G$ is also an $m$-separation tree. However, as in \citep{mxg}, we point out two differences here: (a) an $m$-separation tree is defined with $m$-separation and it does not require that every node be a clique or
that every separator be complete on the augmented graph; (b) junction trees are mostly used as inference engines, while our interest in $m$-separation trees is mainly derived from their power in facilitating the decomposition of structural learning.

A collection of variable sets $C = \{C_1, \dots, C_H \}$ is said to be a hypergraph on $V$ where each hyperedge $C_h$ is
a nonempty subset of variables, and $\cup_{h=1}^HC_h=V$. A hypergraph is a reduced hypergraph if $C_i\not\subseteq C_j$ for $i\ne j$. In this paper, only reduced hypergraphs are used, and thus simply called hypergraphs. 

% In
% Section \ref{sec4}, a hyperedge will be used to represent the domain knowledge of association among variables (see, for example, Figure \ref{Fig:alg2}, (b)).

\section{Construction of \textit{m}-Separation Trees}\label{construct-trees}

As proposed in \citep{xie}, one can construct a $d$-separation tree from observed data, from
domain or prior knowledge of conditional independence relations or from a collection of databases.
However, their arguments are not valid for constructing an $m$-separation tree from domain knowledge or from observed data patterns when latent common causes are present, as in the current setting. In this section, we first extend
Theorem 2 of \citep{xie}, which guarantees that their method for constructing a separation tree
from data is valid for MVR chain graphs. Then we investigate sufficient conditions for constructing $m$-separation trees from domain or prior knowledge of conditional independence relations or from a collection of databases. 

% We remark that Algorithm 3 of \citep{mxg}, which constructs $c$-separation trees from background knowledge encoded in a labeled block ordering (\citep{roverato06}) of the underlying chain graph, also works in the current context.

\subsection{Constructing an \textit{m}-Separation Tree from Observed Data}
In several algorithms for structural learning of PGMs, the first step is to construct an undirected independence graph in which
the absence of an edge $(u, v)$ implies $u \perp\!\!\!\perp v | V\setminus\{u,v\}$. To construct such an undirected graph, we can start with a complete undirected graph, and then for each pair of variables $u$ and $v$, an undirected edge $(u, v)$ is removed if $u$ and
$v$ are independent conditional on the set of all other variables \citep{xie}. For normally distributed data, the undirected independence graph can be efficiently constructed by removing an edge $(u, v)$ if and only if the corresponding entry in the concentration matrix (inverse covariance matrix) is zero \citep[Proposition 5.2]{l}. For this purpose, performing a conditional independence test for each pair of random variables using the
partial correlation coefficient can be used. If the $p$-value of the test is smaller than the given threshold, then there will be an edge on the output graph. For discrete data, a test of conditional independence given a large number of discrete variables may
be of extremely low power. To cope with such difficulty, a local discovery algorithm called
Max-Min Parents and Children (MMPC) \citep{tas} or  the forward selection procedure described in \citep{ed} can be applied.

An $m$-separation tree can be built by constructing a junction tree from an undirected independence graph. In fact, we generalize Theorem 2 of \citep{xie} as follows.
\begin{theorem}\label{thm2}
	A junction tree constructed from an undirected independence graph for MVR CG $G$ is an $m$-separation tree for $G$.
\end{theorem}

An $m$-separation tree $T$ only requires that all $m$-separation properties of $T$ also hold for MVR CG $G$, but the reverse is
not required. Thus we only need to construct an undirected independence graph that may have fewer conditional
independencies than the moral graph, and this means that the undirected independence graph may have extra edges
added to the augmented graph. As \citep{xie} observe for $d$-separation in DAGs, if all nodes of an $m$-separation tree contain only a few variables, ``the null hypothesis of the absence of an undirected edge may be tested statistically at
a larger significance level."

Since there are standard algorithms for constructing junction trees from UIGs \citep[Chapter 4, Section 4]{cdls}, the construction of separation trees reduces to the construction of
UIGs. In this sense, Theorem \ref{thm2} enables us to exploit various techniques for learning UIGs to serve
our purpose. More suggested methods for learning UIGs from data, in addition to the above mentioned techniques, can be found in \citep{mxg}. 

\begin{example}
To construct an $m$-separation tree for MVR CG $G$ in Figure \ref{Fig:mvr1}(a), at first an undirected independence graph
is constructed by starting with a complete graph and removing an edge $(u, v)$ if $u \perp\!\!\!\perp v | V\setminus\{u,v\}$. An undirected graph
obtained in this way is the augmented graph of MVR CG $G$. In fact, we only need to construct an undirected independence
graph which may have extra edges added to the augmented graph. Next triangulate the undirected graph and finally obtain
the $m$-separation tree, as shown in Figure \ref{Fig:mvr1}(b) and Figure \ref{Fig:tree1} respectively.
\end{example}

\subsection{Constructing an \textit{m}-Separation Tree from Domain Knowledge or from Observed Data Patterns}\label{subsec2} 
Algorithm 2 of \citep{xie} proposes an algorithm for constructing a $d$-separation tree $T$ from domain knowledge or from observed
data patterns such that a correct skeleton can be constructed by combining subgraphs for nodes of $T$. In this subsection, we propose an approach for constructing an $m$-separation tree from domain knowledge or from
observed data patterns without conditional independence tests. Domain knowledge of variable dependencies can be represented as a collection of variable
sets $C = \{C_1,\dots , C_H \}$, in which variables contained in the same set may associate with each other directly but variables
contained in different sets associate with each other through other variables. This means that two variables that are not
contained in the same set are independent conditionally on all other variables. On the other hand, in an application study, observed data may have a collection of different observed patterns, $C = \{C_1,\dots , C_H \}$, where $C_h$ is the set of observed variables for the $h$th group of individuals. In both cases, the condition to make our algorithms correct for structural learning from a
collection $C$ is that $C$ must contain sufficient data such that parameters of the underlying MVR CG are estimable. 

For a DAG, parameters are estimable if, for each variable $u$, there is an observed data pattern $C_h$ in $C$ that contains
both $u$ and its parent set. Thus a collection $C$ of observed patterns has sufficient data for correct structural learning
if there is a pattern $C_h$ in $C$ for each $u$ such that $C_h$ contains both $u$ and its parent set in the underlying DAG. Also, domain knowledge is legitimate if, for each variable $u$, there is a hyperedge $C_h$ in $C$ that contains both $u$ and its parent set \citep{xie}. However, these conditions are not valid in the case of MVR chain graphs. In fact, for MVR CGs domain knowledge is legitimate if for each connected component $\tau$, there is a hyperedge $C_h$ in $C$ that contains both $\tau$ and its parent set $pa_G(\tau)$. Also, a collection $C$ of observed patterns has sufficient data for correct structural learning
if there is a pattern $C_h$ in $C$ for each connected component $\tau$ such that $C_h$ contains both $\tau$ and its parent set $pa_G(\tau)$ in the underlying MVR CG.
\begin{algorithm}
\caption{Construct an $m$-separation tree from a hypergraph}\label{hypergraph}
	\SetAlgoLined
	\KwIn{a hypergraph $C = \{C_1, \dots, C_H \}$, where each hyperedge $C_h$ is a variable set such that for each connected component $\tau$, there is a hyperedge $C_h$ in $C$ that contains both $\tau$ and its parent set $pa_G(\tau)$.}
	\KwOut{$T$, which is an $m$-separation tree for the hypergraph $C$.}
    For each hyperedge $C_h$, construct a complete undirected graph $\bar{G}_h$ with the edge set $\bar{E}_h=\{(u,v)|\forall u,v\in C_h\}=C_h\times C_h$\;
	Construct the entire undirected graph  $\bar{G}_V=(V,\bar{E})$, where $\bar{E}=\bar{E}_1\cup\dots\cup \bar{E}_H$\;
	Construct a junction tree $T$ by triangulating $\bar{G}_V$\;
\end{algorithm}

The correctness of Algorithm \ref{hypergraph} is proven in Appendix B. Note that we do not need any conditional independence
test in Algorithm \ref{hypergraph} to construct an $m$-separation tree. In this algorithm, we can use the proposed algorithm in \citep{bbhp} to construct a minimal triangulated graph. In order to illustrate Algorithm \ref{hypergraph}, see Figure \ref{Fig:hypergraph}.
\begin{figure}[ht]
	\centering
	\includegraphics[scale=.45]{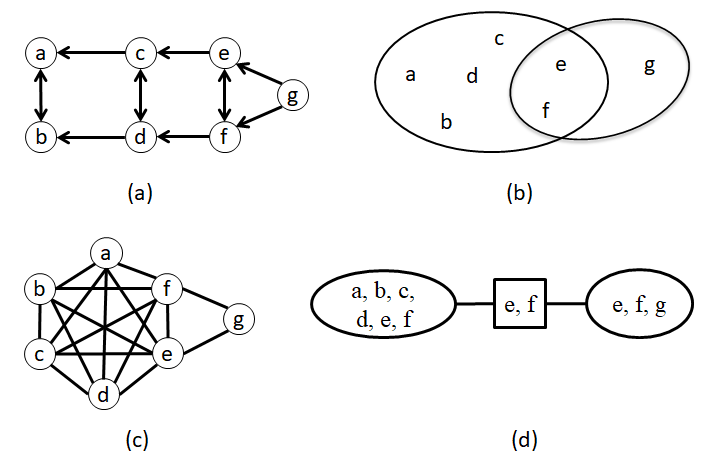}
	\caption{Construction the $m$-separation tree. (a) An MVR CG. (b) Domain knowledge of associations. (c) The undirected graph and triangulation. (d) The $m$-separation tree $T$.} \label{Fig:hypergraph}
\end{figure}

Guaranteeing the presence of both $\tau$ and its parent set $pa(\tau)$ in at least one hyperedge, as required in Algorithm \ref{hypergraph}, is a strong requirement, which may prevent the use of domain knowledge as a practical source of information for constructing MVR chain graphs. In addition, we remark that answering the question "how can one obtain this information?" is beyond the scope of this paper. The two examples that follow show that restricting the hyperedge contents in two natural ways lead to errors.

The example illustrated in Figure \ref{Fig:counterex} shows that, if for each variable $u$ there is a hyperedge $C_h$ in $C$ that contains both $u$ and its parent set, we cannot guarantee the correctness of our algorithm. Note that vertices $a$ and $d$ are separated in the tree $T$ of Figure \ref{Fig:counterex} part (d) by removing vertex $b$, but $a$ and $d$ are not $m$-separated given $b$ as can be verified using \ref{Fig:counterex} part (a). 
\begin{figure}[ht]
	\centering
	\includegraphics[scale=.5]{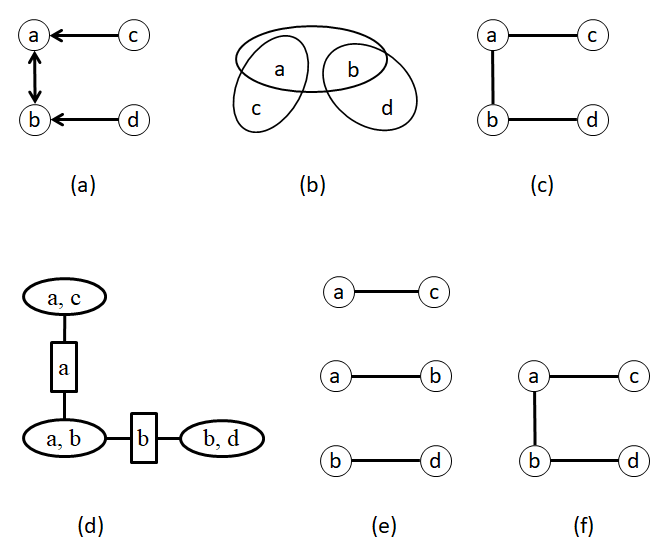}
	\caption{Insufficiency of having a hypergraph that contains both $u$ and its parent set for every $u\in V$. (a) An MVR CG. (b) Domain knowledge of associations. (c) The undirected graph constructed by union of complete graphs corresponding to each hyperedge, which is also a triangulated graph. (d) The junction tree $T$. (e) Local skeleton for every node of $T$. (f) The global skeleton and all $v$-structures.} \label{Fig:counterex}
\end{figure}
The example illustrated in Figure \ref{Fig:jvmethod} shows that, if for each variable $u$ there is a hyperedge $C_h$ in $C$ that contains both $u$ and its boundary set, Algorithm \ref{hypergraph} does not necessarily give an $m$-separation tree because, for example, $S=\{a,b\}$ separates $c$ and $d$ in tree $T$ of Figure \ref{Fig:jvmethod} part (d), but $S$ does not $m$-separate $c$ and $d$ in the MVR CG $G$ in Figure \ref{Fig:jvmethod} part (a).  
\begin{figure}[ht]
	\centering
	\includegraphics[scale=.5]{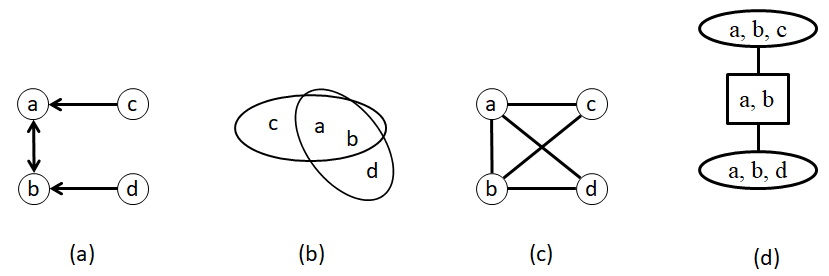}
	\caption{Insufficiency of having a hypergraph that contains both $u$ and its boundary set for every $u\in V$. (a) An MVR CG. (b) Domain knowledge of associations. (c) The undirected graph constructed by union of complete graphs corresponding to each hyperedge, which is also a triangulated graph. (d) The junction tree $T$, which is not an $m$-separation tree.} \label{Fig:jvmethod}
\end{figure}

% The two previous examples show that, in order to build an $m$-separation tree using Algorithm \hyperref[alg2]{2}, it is necessary that, for each connected component $\tau$, there is a $C_h\in C$ that contains both $\tau$ and its parent set $pa_G(\tau)$ in the underlying MVR CG. 
 
% One of the most important differences between our algorithms and algorithms proposed in \citep{xie} is that \citep{xie} assumes that there are no latent (hidden) variables, while in this paper we deal with situations where there are latent variables. 

\section{Decomposition of Structural Learning}\label{main-alg}
Applying the following theorem to structural learning, we can split a problem of searching for $m$-separators and building the skeleton of a CG into small problems for every node of $m$-separation tree $T$.
\begin{theorem}\label{thm1}
	Let $T$ be an $m$-separation tree for  CG $G$. Vertices $u$ and $v$ are $m$-separated by $S\subseteq V$ in $G$ if and
	only if (i) $u$ and $v$ are not contained together in any node $C$ of $T$ or (ii) there exists a node $C$ that contains both $u$
	and $v$ such that a subset $S'$ of $C$ $m$-separates $u$ and $v$.
\end{theorem}
According to Theorem \ref{thm1}, a problem of searching for an $m$-separator $S$ of $u$ and $v$ in all possible subsets of $V$ is
localized to all possible subsets of nodes in an $m$-separation tree that contain $u$ and $v$. For a given $m$-separation tree $T$
with the node set $C = \{C_1,\dots , C_H \}$, we can recover the skeleton and all $v$-structures for a CG as follows. First
we construct a local skeleton for every node $C_h$ of $T$, which is constructed by starting with a complete undirected
subgraph and removing an undirected edge $(u, v)$ if there is a subset $S$ of $C_h$ such that $u$ and $v$ are independent
conditional on $S$. Then, in order to construct the global skeleton, we combine all these local skeletons together and remove
edges that are present in some local skeletons but absent in other local skeletons. Then we determine every $v$-structure
if two non-adjacent vertices $u$ and $v$ have a common neighbor in the global skeleton but the neighbor is not contained
in the $m$-separator of $u$ and $v$. Finally we can orient more undirected edges if none of them creates either a
partially directed cycle or a new $v$-structure (see, for example, Figure \ref{Fig:alg1}). This process is formally described in the following algorithm:

\begin{algorithm}[ht]
\caption{A recovery algorithm for MVR chain graphs}\label{alg1}
	\SetAlgoLined
	\KwIn{a probability distribution $p$ faithful to an
unknown MVR CG $G$.}
	\KwOut{the pattern of MVR CG $G$.}
    Construct an $m$-separation tree $T$ with a node set $C = \{C_1, \dots, C_H \}$ as discussed in Section \ref{construct-trees}\;
    Set $S=\emptyset$\;
    \For{$h\gets 1$ \KwTo $H$}{
    Start from a complete undirected graph $\bar{G}_h$ with         vertex set $C_h$\;
    \For{\textrm{each vertex pair $\{u,v\}\subseteq C_h$ }}{\If{$\exists S_{uv}\subseteq C_h \textrm{ such that } u \perp\!\!\!\perp v | S_{uv}$}{
    Delete the edge $(u,v)$ in $\bar{G}_h$\;
    Add $S_{uv}$ to $S$;
    }
    }
    }
    Initialize the edge set $\bar{E}_V$ of $\bar{G}_V$ as the union of all edge sets of $\bar{G}_h, h=1,\dots, H$\;
    \For{\textrm{each Vertex pair $\{u,v\}$ contained in more than one tree node and $(u,v)\in \bar{G}_V$}}{
    \If{$\exists C_h \textrm{ such that } \{u,v\}\subseteq C_h \textrm{ and } \{u,v\}\not\in \bar{E}_h$}{
    Delete the edge $(u,v)$ in $\bar{G}_V$\;
    }
    }
    \For{\textrm{each $m$-separator $S_{uv}$ in the list $S$}}{\If{\textrm{$u\lcircle w\rcircle v$ appears in the global skeleton
	and $w$ is not in $S_{uv}$}}{
    \tcc{$u\lcircle w$ means $u\gets w$ or $u-w$. Also, $w\rcircle v$ means $w\to v$ or $w-v.$}
    Determine a $v$-structure $u\lacircle w\racircle v$\;
    }
    }
\end{algorithm}
\begin{figure}[ht]
	\centering
	\includegraphics[scale=.4]{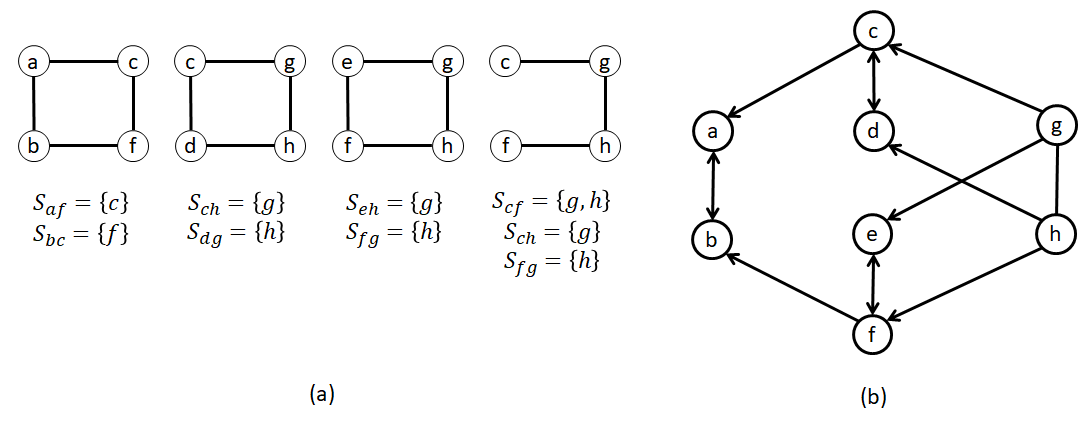}
	\caption{(a) Local skeletons for every node of the $m$-separation tree in Figure \ref{Fig:tree1}. (b) The global skeleton and all $v$-structures.} \label{Fig:alg1}
\end{figure}

The following algorithm returns an MVR chain graph that contains exactly the minimum set
of bidirected edges for its Markov equivalence
class. For the correctness of lines 2-7 in Algorithm \ref{alg2}, see \citep{sp}.
\begin{algorithm}
\caption{A recovery algorithm for MVR chain graphs with minimum set of bidirected edges for its equivalence
class}\label{alg2}
	\SetAlgoLined
	\KwIn{a probability distribution $p$ faithful to an
unknown MVR CG $G$.}
	\KwOut{an MVR CG $G'$ s.t.
$G$ and $G'$ are Markov equivalent and $G'$ has exactly the minimum set of bidirected edges for its equivalence
class.}
Call Algorithm \ref{alg1} to construct $G'$, which is the equivalence class of MVR CGs for $G$\;
Apply rules 1-3 in Figure \ref{Fig:rules} while possible\;
\tcc{After this line, the learned graph is the \textit{essential graph} of MVR CG $G$ i.e.,  it
has the same skeleton as $G$ and contain all and only the arrowheads that
are shared by all MVR CGs in the Markov equivalence class of $G$ \citep{essentialmvrcgs}.}
Let $G'_u$ be the subgraph of $G'$ containing only the
nodes and the undirected edges in $G'$\;
Let $T$ be the junction tree of $G'_u$\; 
\tcc{If $G'_u$ is disconnected, the cliques belonging to different connected components can be linked with empty separators, as described in \cite[Theorem 4.8]{Golumbic}.}
Order the cliques $C_1,\cdots , C_n$ of $G'_u$ s.t. $C_1$ is the root of $T$ and if $C_i$ is closer to the root than $C_j$ in $T$ then $C_i < C_j$\;
Order the nodes such that if $A\in C_i$, $B\in C_j$, and $C_i < C_j$ then $A < B$\;
Orient the undirected edges in $G'$ according to the ordering obtained in line 6.
\end{algorithm}
\begin{figure}[ht]
	\centering
	\includegraphics[scale=.4]{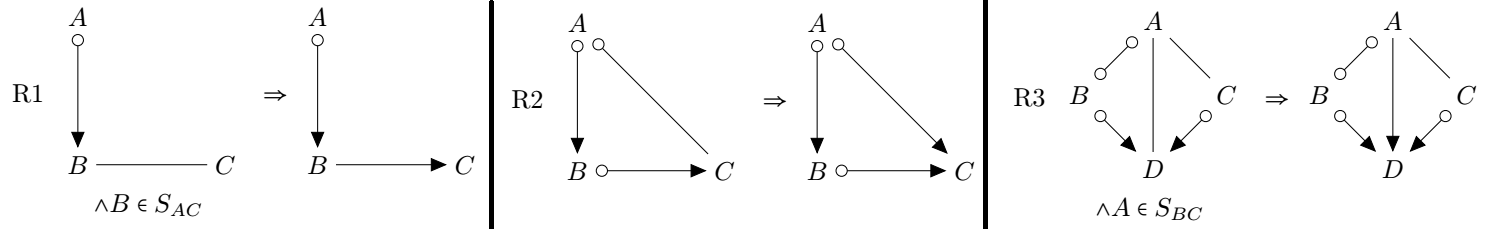}
	\caption{The Rules \citep{sp}} \label{Fig:rules}
\end{figure}

According to Theorem \ref{thm1}, we can prove that the global skeleton and all $v$-structures obtained by applying the decomposition in Algorithm \ref{alg1} are correct, that is, they are the same as those obtained from the joint distribution of $V$, see
Appendix A for the details of proof. Note that separators in an $m$-separation tree may not be complete in the augmented graph.
Thus the decomposition is weaker than the decomposition usually defined for parameter estimation \citep{cdls,l}.

\section{Complexity Analysis and Advantages}\label{complexity}
In this section, we start by comparing our algorithm with the main algorithm in \citep{xie} that is designed
specifically for DAG structural learning when the underlying graph structure is a DAG. We make
this choice of the DAG specific algorithm so that both algorithms can have the same separation tree
as input and hence are directly comparable. 

In a DAG, all chain components are singletons. Therefore, sufficiency of having a hypergraph that contains both $\tau$ and its parent set for every chain component is equivalent with having a hypergraph that contains both $u$ and its parent set for every $u\in V$, when the underlying graph structure is a DAG. Therefore, it is obvious that our algorithm has the same effect and the same complexity as the main algorithm in \citep{xie}.

The same advantages mentioned by \citep{xie} for their BN structural learning algorithm hold for our algorithm when applied to MVR CGs. For the reader convenience, we list them here. 
First, by using the $m$-separation tree, independence tests are performed only conditionally on smaller sets
contained in a node of the $m$-separation tree rather than on the full set of all other variables. Thus our algorithm has
higher power for statistical tests. 
% \textcolor{red}{Second, the theoretical results proposed in this paper can be applied to scheme design of multiple databases.
% Without loss of information on structural learning of MVR CGs, a joint data set can be replaced by a group of incomplete
% data sets based on the domain or prior knowledge of conditional independencies among variables} \citep{xie}. 
Second, the computational complexity can be reduced. This complexity analysis focuses only on the number of
conditional independence tests for constructing the equivalence class. Decomposition of graphs is a computationally
simple task compared to the task of testing conditional independence for a large number of triples of sets of variables. The triangulation of an undirected graph is used in our
algorithms to construct an $m$-separation from an undirected independence graph. Although the problem for optimally
triangulating an undirected graph is NP-hard, sub-optimal triangulation methods \citep{bbhp} may be used provided
that the obtained tree does not contain too large nodes to test conditional independencies. Two of the best known
algorithms are lexicographic search and maximum cardinality search, and their complexities are
$O(|V||E|)$ and $O(|V|+ |E|)$, respectively \citep{bbhp}. Thus in our algorithms, the conditional independence tests dominate the algorithmic
complexity.

The complexity of the Algorithm \ref{alg1} is $O(Hm^22^m)$ as claimed in \citep[Section 6]{xie}, where $H$ is the number of hyperedges (usually $H \ll |V|$) and $m=\max_h|C_h|$ where $|C_h|$ denotes the number of variables in $C_h$ ($m$ usually is much less than $|V|$).

\section{Evaluation}\label{evaluation}
In this section, we evaluate the performance of our algorithms in various setups
using simulated / synthetic data sets. We first compare the performance of our algorithm with the PC-like
learning algorithm \citep{sp} by running them
on randomly generated MVR chain graphs. (A brief description of the PC-like algorithm is provided at the beginning of Section \ref{discussion}.) We then compare our method with the PC-like algorithm  on different discrete Bayesian networks such as \href{http://www.bnlearn.com/bnrepository/}{ASIA, INSURANCE, ALARM, and HAILFINDER} that have
been widely used in evaluating the performance of structural learning algorithms. Empirical simulations show that our algorithm achieves
competitive results with the PC-like learning algorithm; in particular, in the Gaussian case the decomposition-based algorithm outperforms (except in running time) the PC-like algorithm. 
Algorithms \ref{alg1} , \ref{alg2}, and the PC-like algorithm have been implemented in the R language. All the results reported here are
based on our R implementation \citep{jv3}.

\subsection{Performance Evaluation on Random MVR Chain Graphs (Gaussian case)}

To investigate the performance of the decomposition-based learning method, we use the same approach that \citep{mxg} used in
evaluating the performance of the LCD algorithm on LWF chain graphs. We run our algorithms and the PC-like algorithm
on randomly generated MVR chain graphs and then we compare the results and report summary error measures in all cases.

\subsubsection{Data Generation Procedure}
First we explain the way in which the random MVR chain graphs and random samples are generated.
Given a vertex set $V$ , let $p = |V|$ and $N$ denote the average degree of edges (including bidirected
and pointing out and pointing in) for each vertex. We generate a random MVR chain graph on $V$ as
follows:
\begin{itemize}
    \item Choose one element, say $k$, of the vector $c=(0.1, 0.2, 0.3, 0.4, 0.5)$ randomly\footnote{In the case of $p=40,50$ we use $c=(0.1,0.2)$.}.
    \item Use the randDAG function from the \href{https://cran.r-project.org/web/packages/pcalg/index.html}{pcalg} R package and generate an un-weighted random Erdos-Renyi graph, which is a DAG with $p+(k\times p)$ nodes and $N$ expected number of neighbours per node.
    \item Use the AG function from the \href{https://cran.r-project.org/web/packages/ggm/index.html}{ggm} R package and marginalize out $k\times p$ nodes to obtain a random MVR chain graph with $p$ nodes and $N$ expected number of neighbours per node. If the obtained graph is not an MVR chain graph, repeat this procedure until an MVR CG is obtained.
\end{itemize}

The rnorm.cg function from the \href{http://www2.uaem.mx/r-mirror/web/packages/lcd/index.html}{lcd} R package was used to generate a desired number of normal random samples from the canonical DAG \citep{rs} corresponding to the obtained MVR chain graph in the first step. Notice that faithfulness is not necessarily guaranteed by the current sampling procedure \citep{mxg}.

\subsubsection{Experimental Results for Random MVR Chain Graphs (Gaussian case)}
We evaluate the performance of the decomposition-based and PC-like algorithms in terms of five measurements: (a) the true positive
rate (TPR)\footnote{Also known as sensitivity, recall, and hit rate.}, (b) the false positive rate (FPR)\footnote{Also known as fall-out.}, (c) accuracy (ACC) for the skeleton, (d) the structural Hamming distance (SHD)\footnote{This is the metric described in \citep{Tsamardinos2006} to  compare the
structure of the learned and the original graphs.}, and (e) run-time for the pattern recovery algorithms. In short, $TPR=\frac{\textrm{true positive } (TP)}{\textrm{the number of real positive cases in the data } (Pos)}$ is the ratio of  the number of correctly identified edges over total number of edges, $FPR=\frac{\textrm{false positive }(FP)}{\textrm{the number of real negative cases in the data }(Neg)}$ is the ratio of the number of incorrectly identified edges over total number of gaps, $ACC=\frac{\textrm{true positive }(TP) +\textrm{ true negative }(TN)}{Pos+Neg}$ and
$SHD$ is the number of legitimate operations needed to change the current pattern to the true one,
where legitimate operations are: (a) add or delete an edge and (b) insert, delete or reverse an edge
orientation. In principle, a large TPR and ACC, a small FPR and SHD indicate good performance.

In our simulation, we change three parameters $p$ (the number of vertices), $n$ (sample size) and
$N$ (expected number of adjacent vertices) as follows:
\begin{itemize}
    \item $p\in\{10, 20, 30, 40, 50\}$,
    \item $n\in\{300, 1000, 3000, 10000\}$, and
    \item $N\in\{2,3,5,8,10\}$.
\end{itemize}

For each $(p,N)$ combination, we first generate 25 random MVR chain graphs. We then generate a
random Gaussian distribution based on each corresponding canonical DAG and draw an identically independently distributed
(i.i.d.) sample of size $n$ from this distribution for each possible $n$, and finally we remove those columns (if any exist) that correspond to the hidden variables. For each sample, three different
significance levels $\alpha=0.05/0.01/0.005$ are used to perform the hypothesis tests. For decomposition-based algorithm we consider two different versions: The first version  uses Algorithm \ref{alg1} and the three rules in Algorithm \ref{alg2}, while the second version uses both Algorithm \ref{alg1} and \ref{alg2}. Since the learned graph of the first version may contain some undirected edges, we call it the \textit{essential recovery algorithm}. However, removing all directed and bidirected edges from the learned graph results in a chordal graph \citep{sp}. Furthermore, the learned graph has exactly the (unique) minimum set of bidirected edges for its Markov equivalence class \citep{sp}. The second version of the decomposition-based algorithm returns an MVR chain graph that has exactly the
minimum set of bidirected edges for its equivalence
class. A similar approach is used for the PC-like algorithm. We then
compare the results to access the performance of the decomposition-based algorithm against the PC-like algorithm. The entire plots of the error measures and running times can be seen in the \href{https://www.dropbox.com/sh/iynnlwyu8il7m3v/AACk8SyIEn7s-W9NRlLnz0DDa?dl=0}{supplementary document} \citep{jv3}.
% Figure \ref{fig:tprfpracc1}, \ref{fig:tprfpracc2}, \ref{fig:tprfpracc3}, \ref{fig:tprfpracc4}, \ref{fig:tprfpracc5}, \ref{fig:tprfpracc6}, \ref{fig:tprfpracc7}, \ref{fig:tprfpracc8}, \ref{fig:tprfpracc9}, \ref{fig:tprfpracc10}, \ref{fig:tprfpracc11}, \ref{fig:tprfpracc12}, \ref{fig:tprfpracc13}, \ref{fig:shd1}, \ref{fig:shd2}, \ref{fig:shd3}, \ref{fig:shd4}, and \ref{fig:shd5}.
From the plots, we infer that: (a) both
algorithms yield better results on sparse graphs $(N = 2,3)$ than on dense graphs $(N = 5,8,10)$, for example see Figures \ref{fig:tprfpracc1} and \ref{fig:shd1}; (b) for both algorithms, typically the TPR and ACC increase with sample size, for example see Figure \ref{fig:tprfpracc1}; (c) for both algorithms, typically the SHD decreases with sample size for sparse graphs $(N = 2,3)$. For $N=5$ the SHD decreases with sample size for the decomposition-based algorithm while the SHD has no clear dependence on the sample size for the PC-like algorithm in this case. Typically, for the PC-like algorithm the SHD increases with sample size for dense graphs $(N = 8,10)$ while the SHD has no clear dependence on the sample size for the decomposition-based algorithm in these cases, for example see Figure \ref{fig:shd1}; (d) a large significance level $(\alpha=0.05)$ typically yields large
TPR, FPR, and SHD, for example see Figures \ref{fig:tprfpracc1} and \ref{fig:shd1}; (e) in almost all cases, the performance of the decomposition-based algorithm based on all error measures i.e., TPR, FPR, ACC, and SHD is better than the performance of the PC-like algorithm, for example see Figure \ref{fig:tprfpracc1} and \ref{fig:shd1}; (f) In most cases, error measures based on $\alpha=0.01$ and $\alpha=0.005$ are very close, for example see Figure \ref{fig:tprfpracc1} and \ref{fig:shd1}. Generally, our empirical results suggests that in order to obtain a better performance, we can choose a small value (say $\alpha=0.005$ or 0.01) for
the significance level of individual tests along with large sample (say $n=3000$ or 10000). However, the optimal value for a desired overall error rate may depend on the sample size, significance level, and the sparsity of the underlying graph.
\begin{figure}
	\centering
	\includegraphics[scale=.25,page=9]{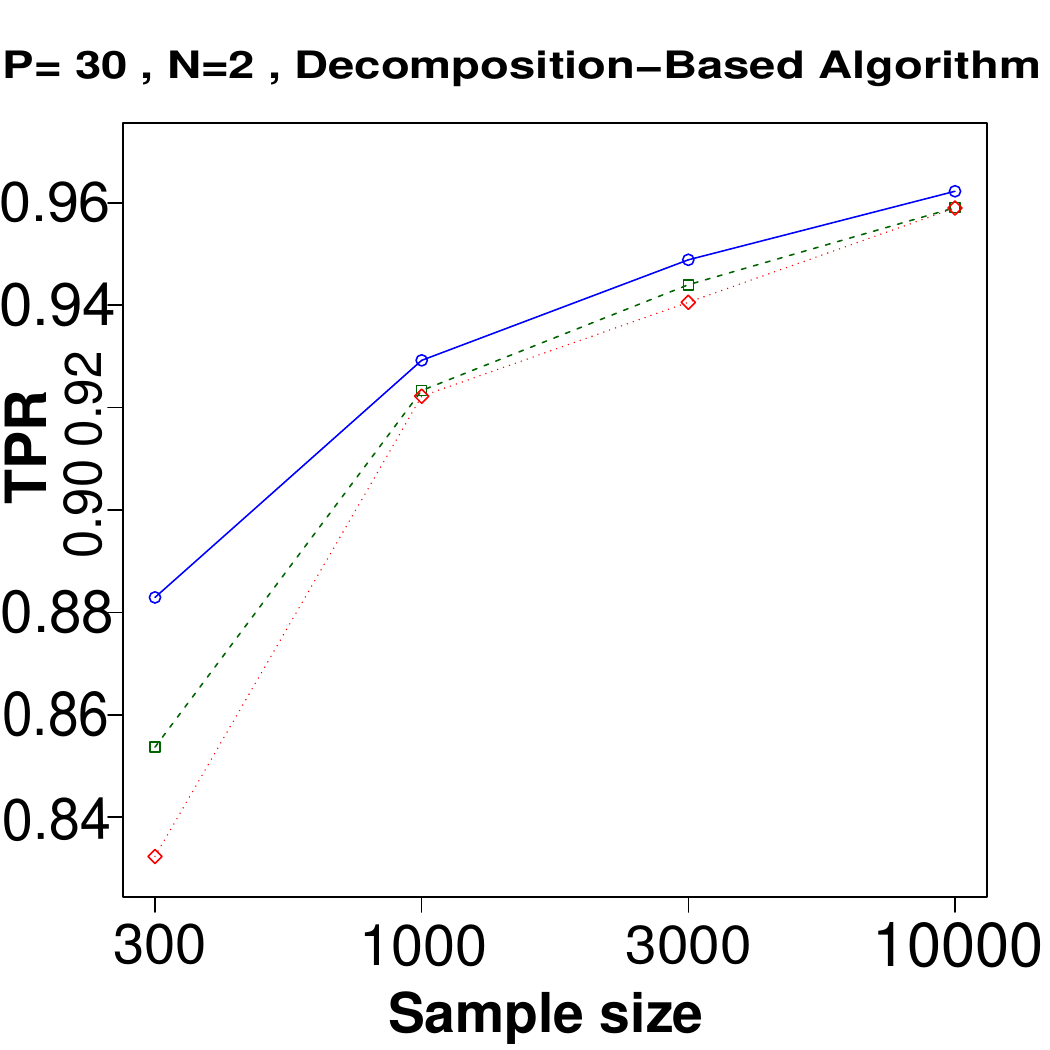}
	\includegraphics[scale=.25,page=9]{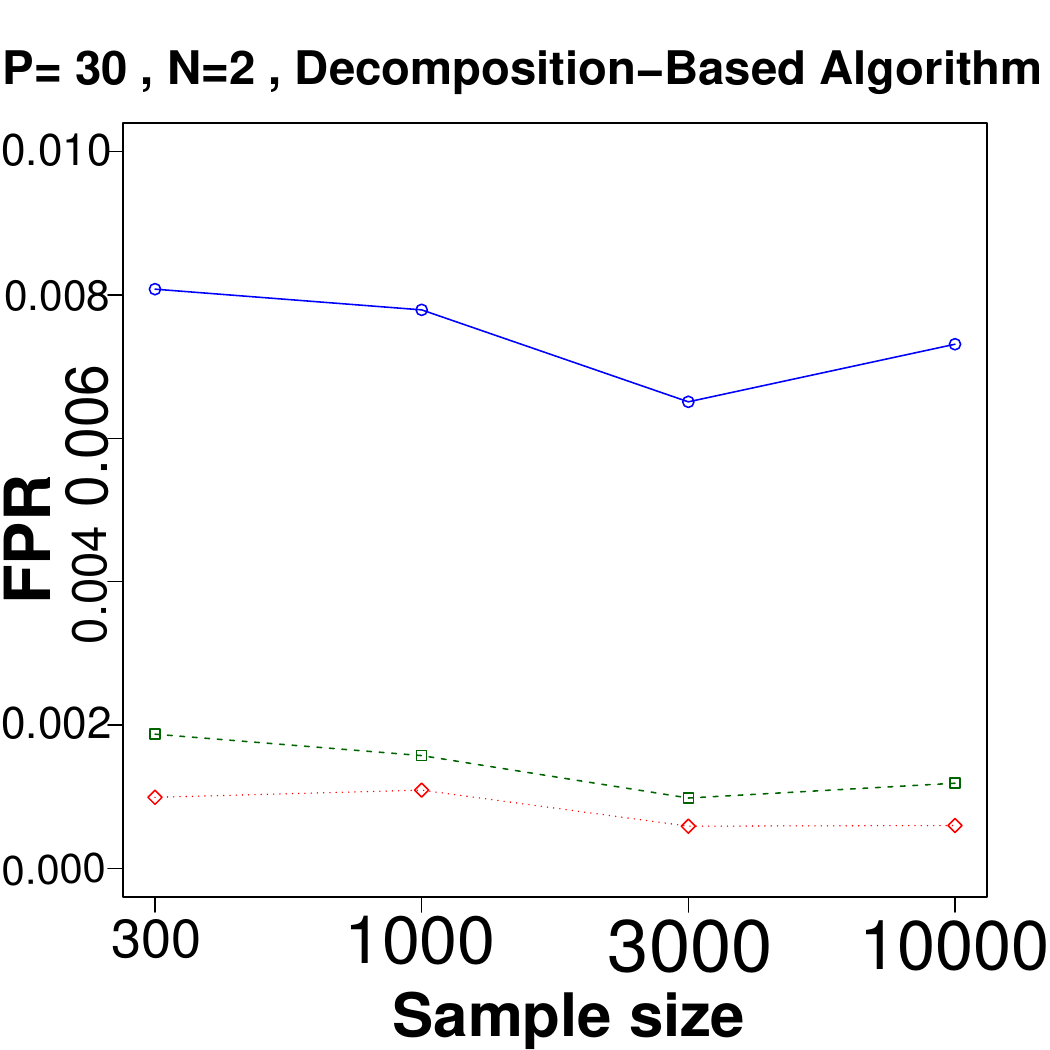}
	\includegraphics[scale=.25,page=9]{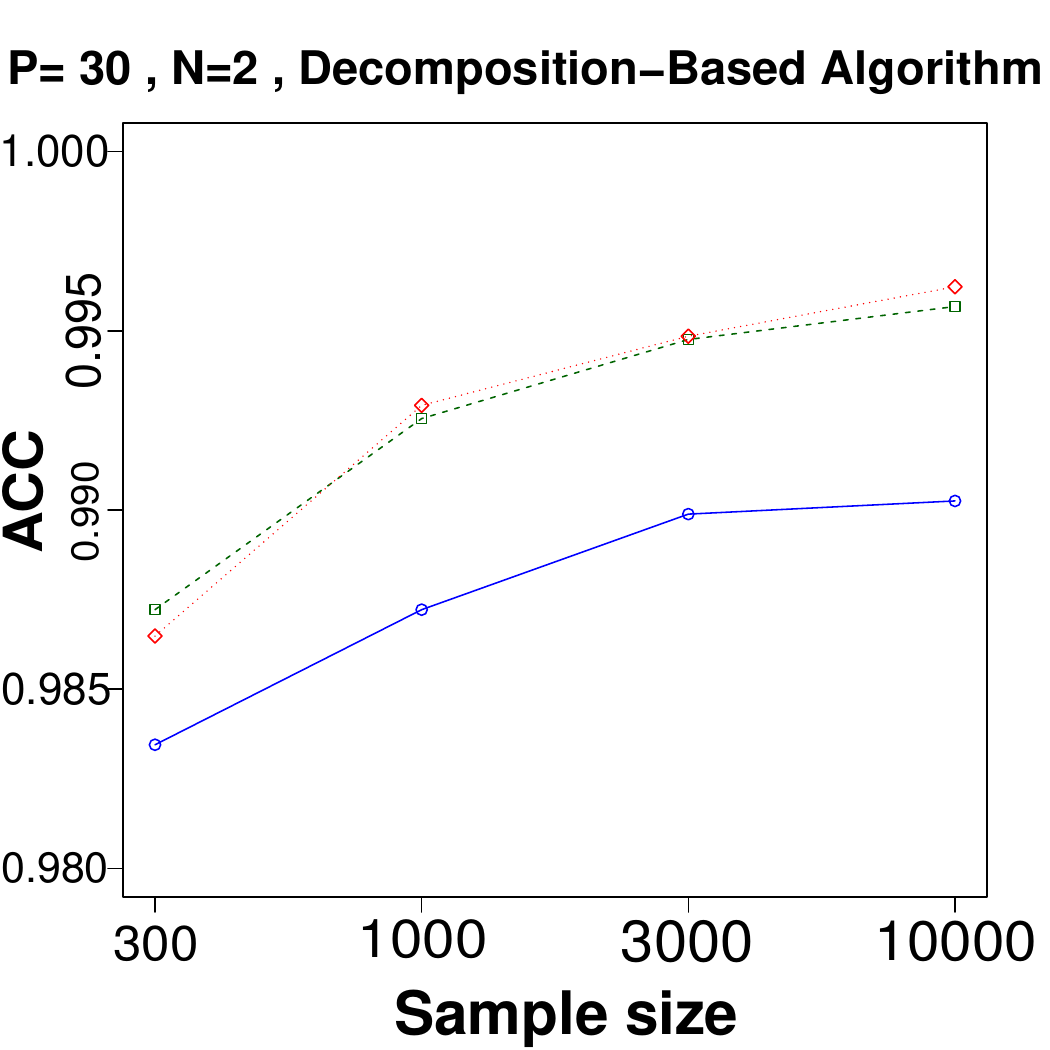}
	\includegraphics[scale=.25,page=9]{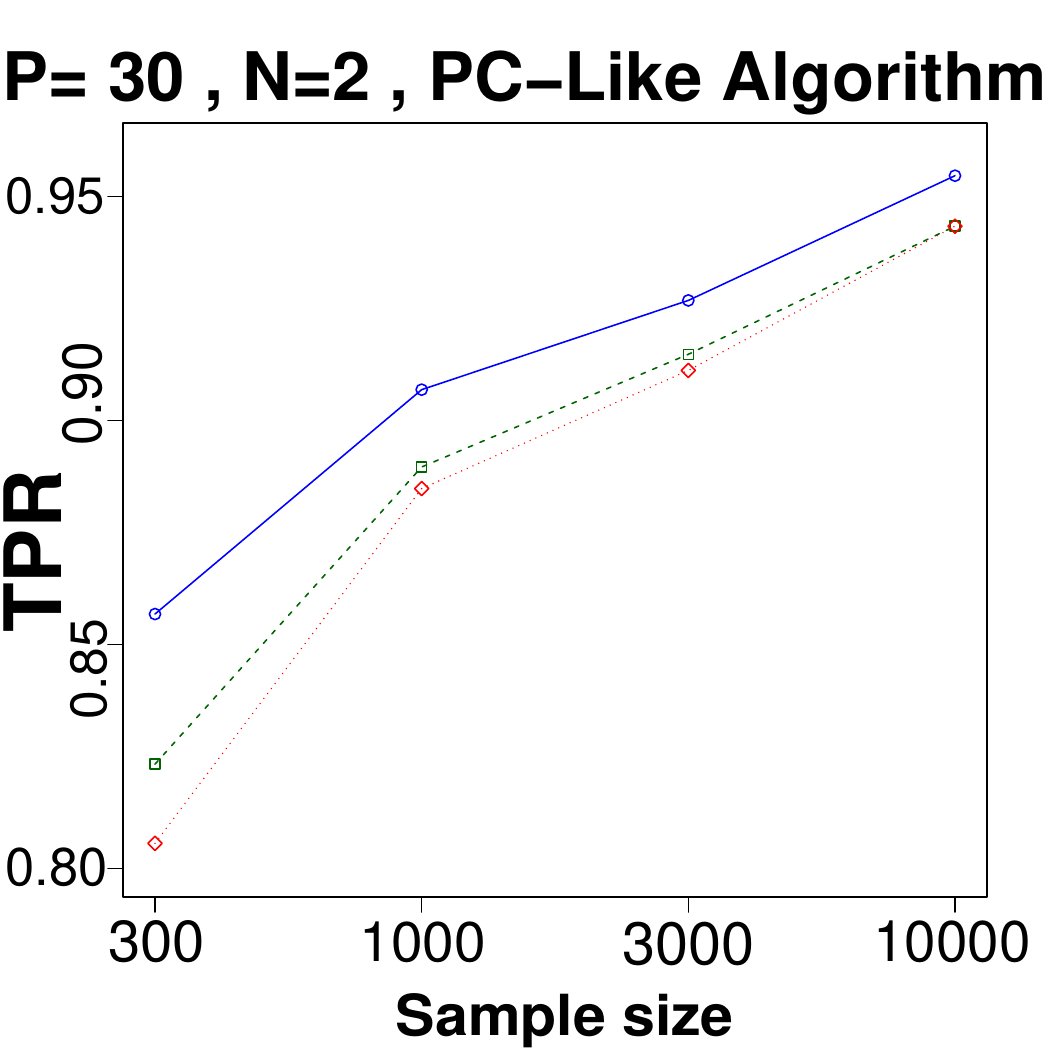}
	\includegraphics[scale=.25,page=9]{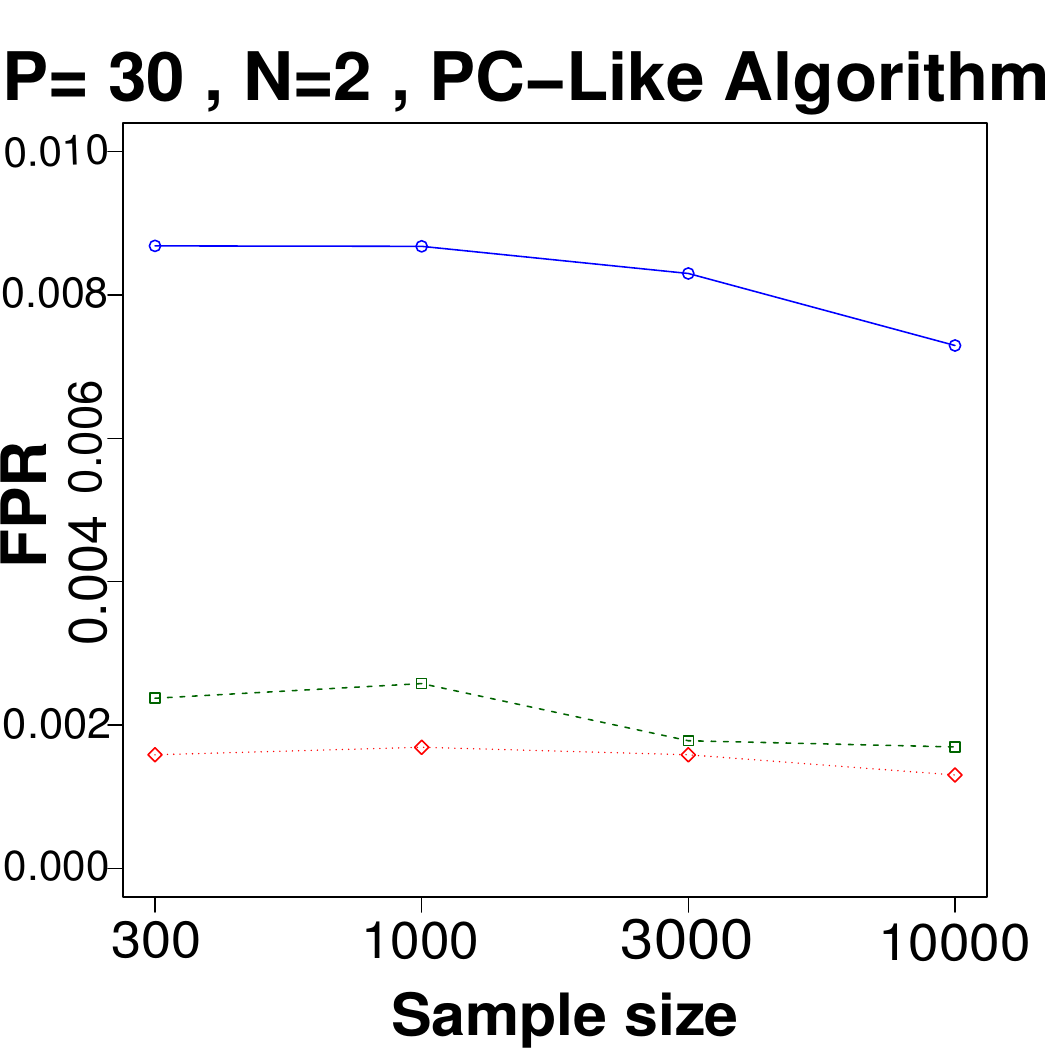}
	\includegraphics[scale=.25,page=9]{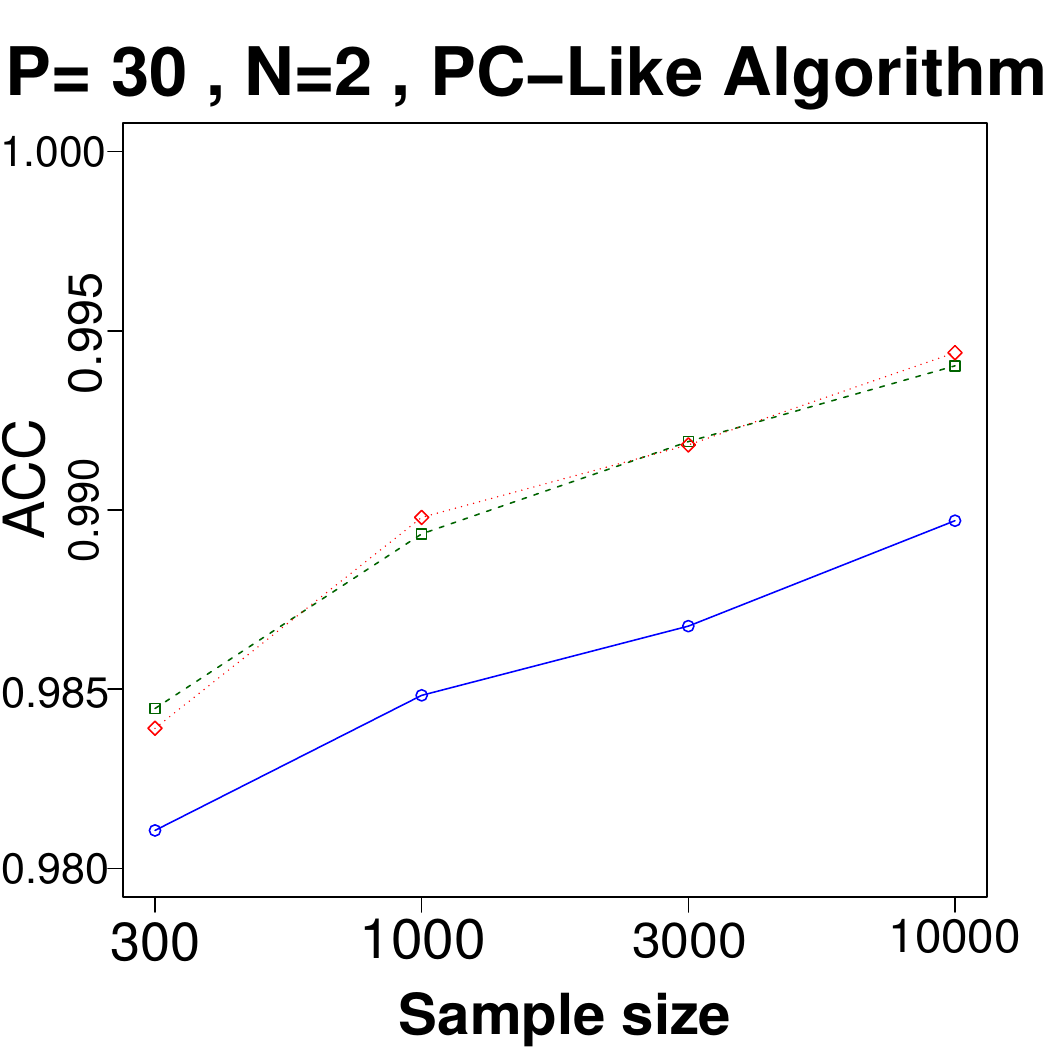}
	\includegraphics[scale=.25,page=12]{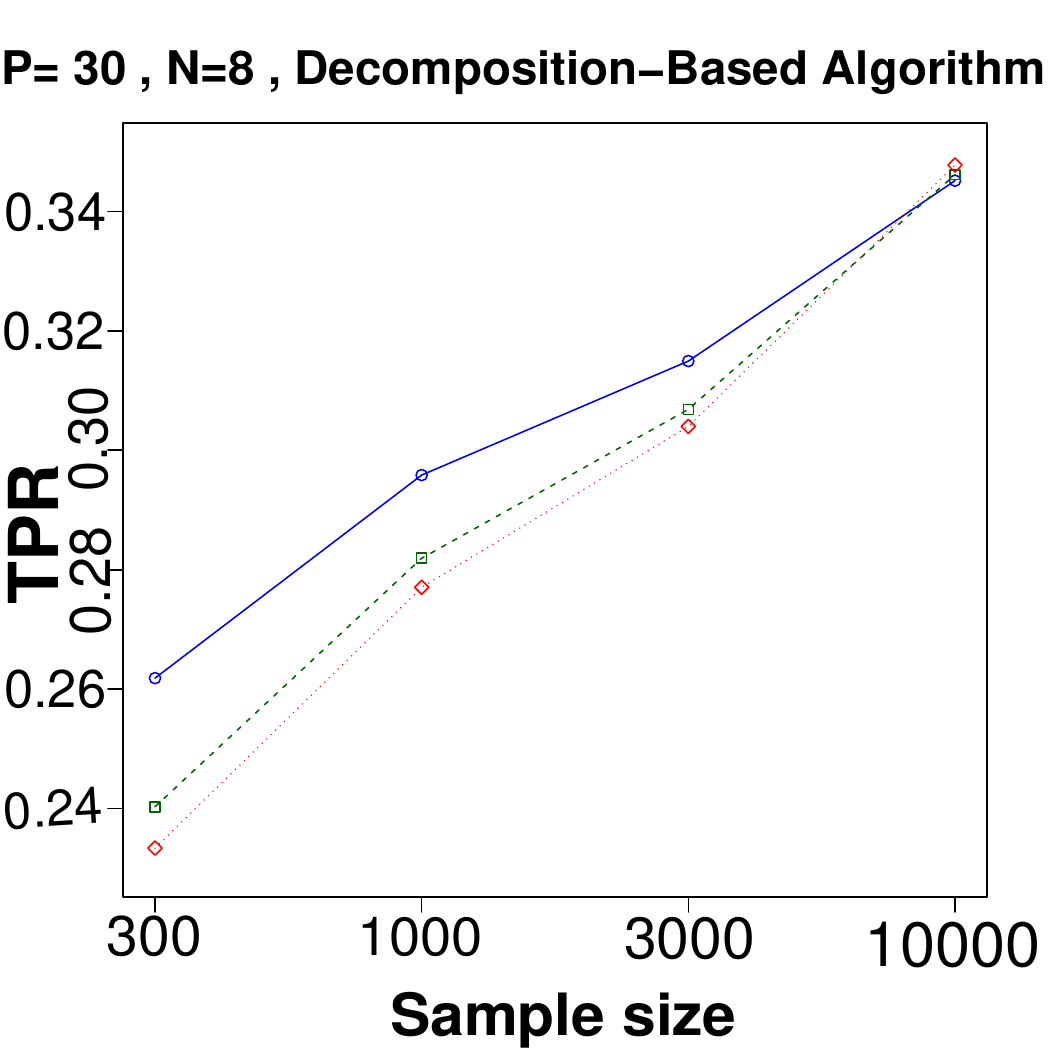}
	\includegraphics[scale=.25,page=12]{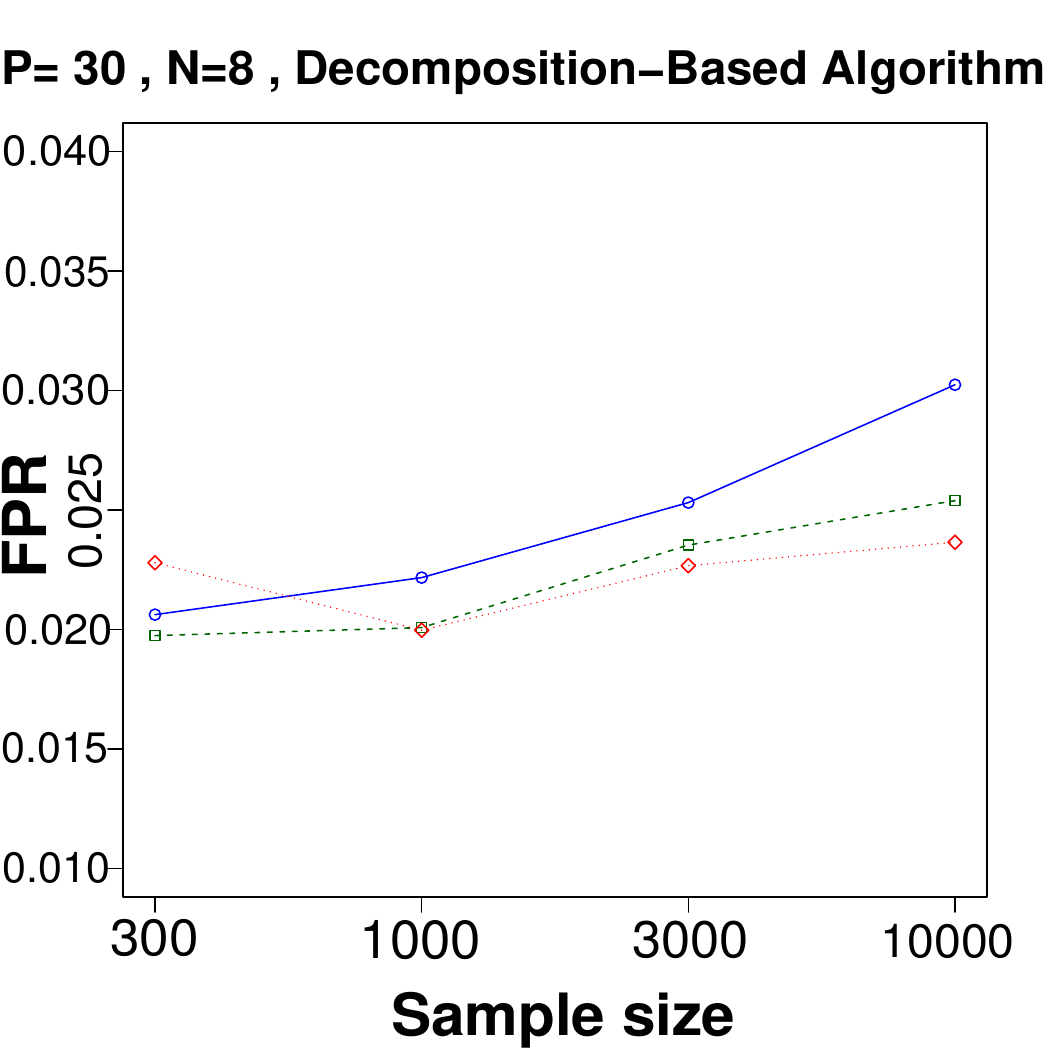}
	\includegraphics[scale=.25,page=12]{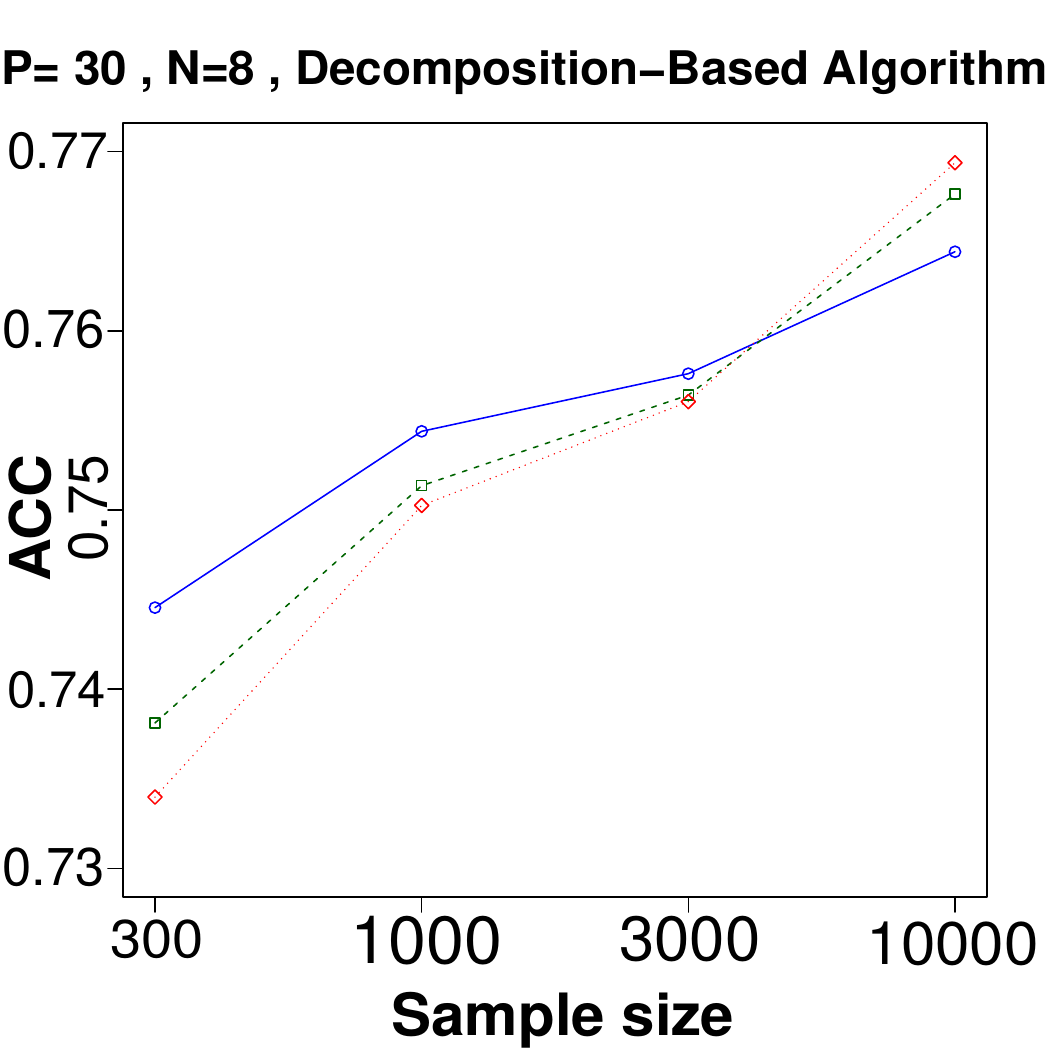}
	\includegraphics[scale=.25,page=12]{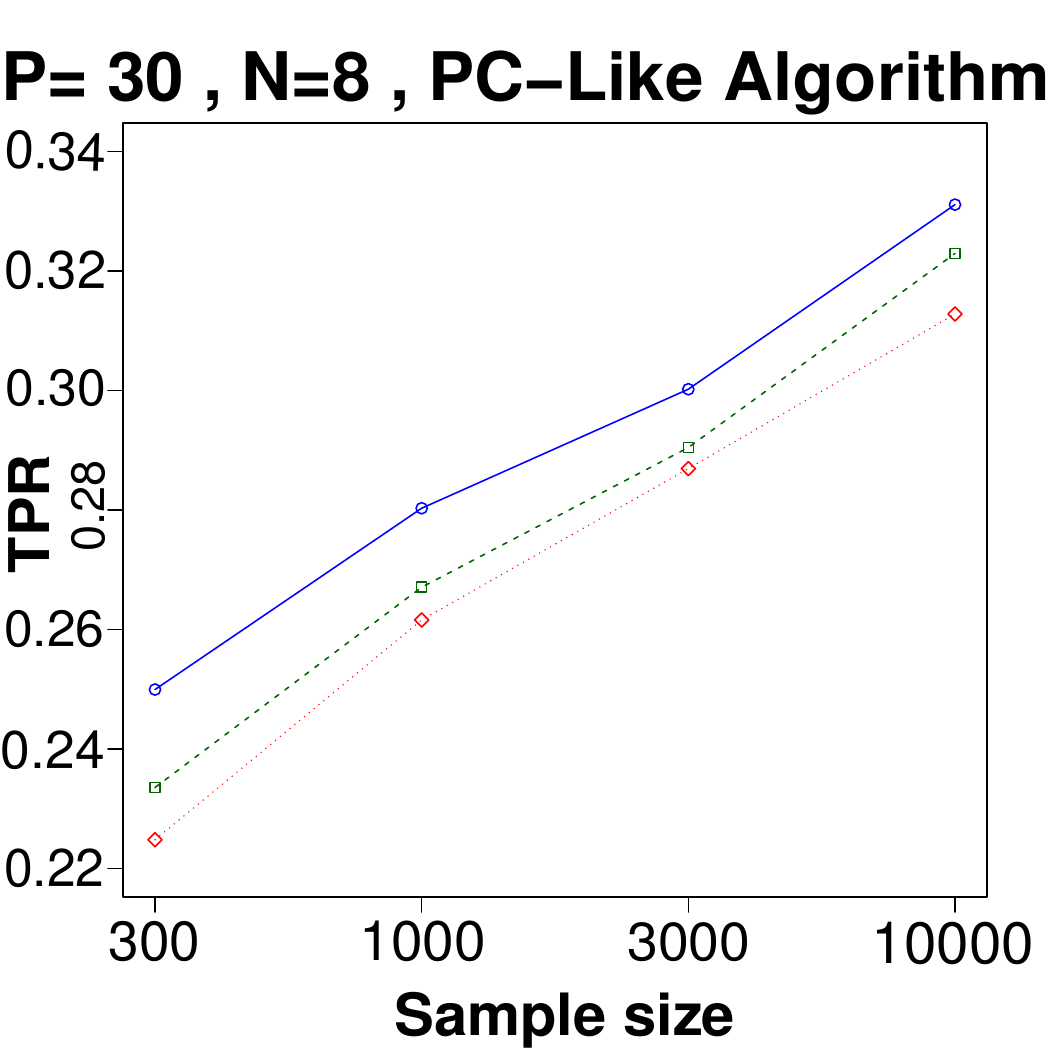}
	\includegraphics[scale=.25,page=12]{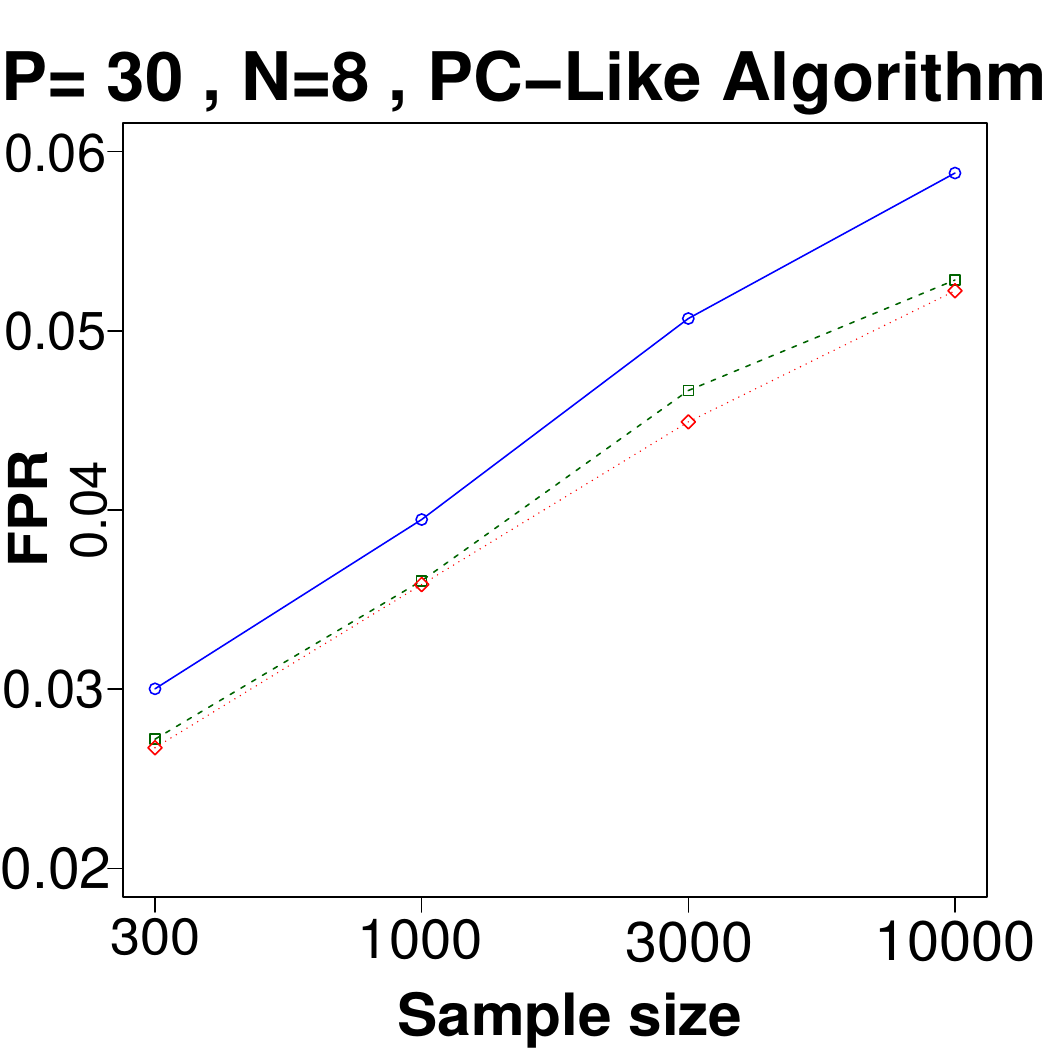}
	\includegraphics[scale=.25,page=12]{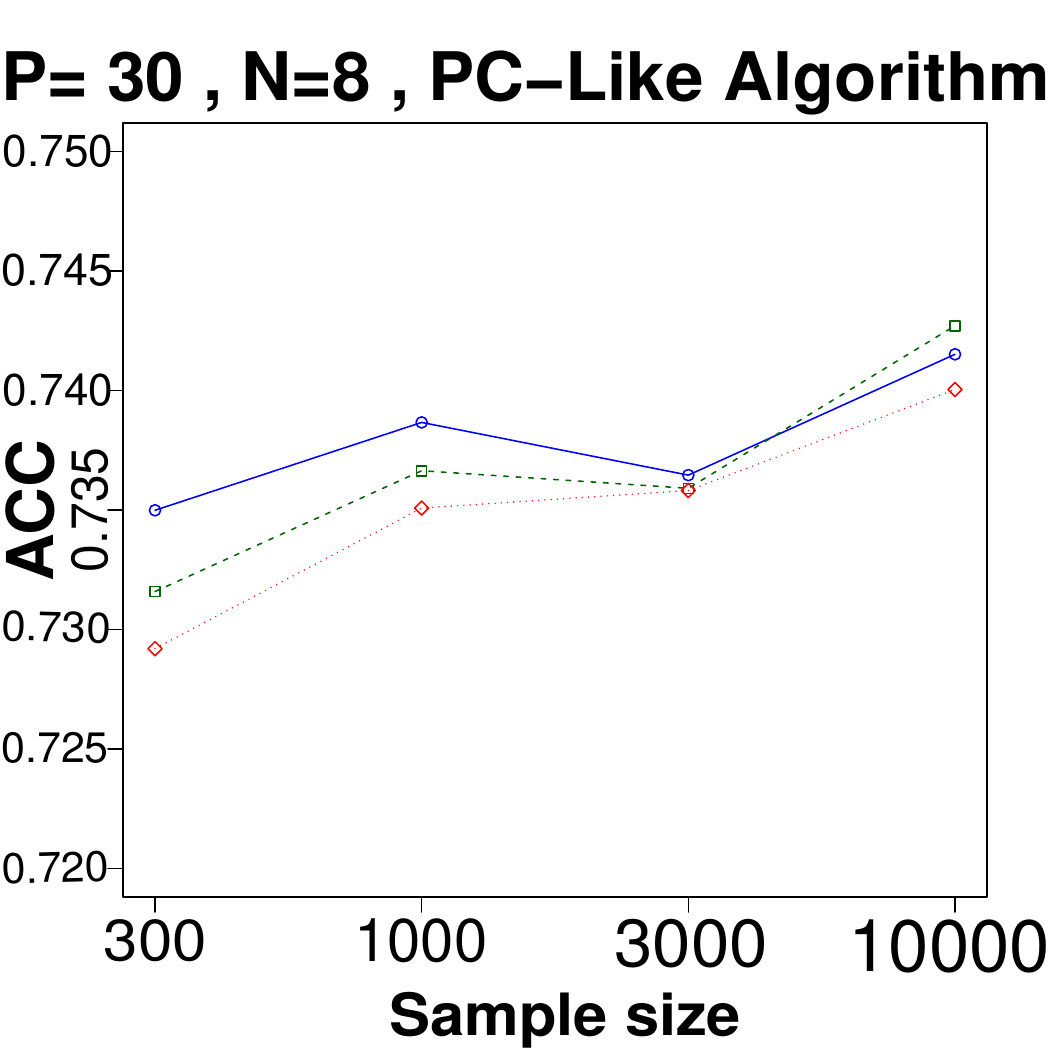}
	\caption{Error measures of the decomposition-based and PC-like algorithms for randomly generated Gaussian chain graph models:
		average over 25 repetitions with 30 variables. The four rows correspond to N = 2 and 8.  The three columns give three error measures: TPR, FPR and ACC in each
		setting respectively. In each plot, the solid (blue)/dashed (green)/dotted (red) lines correspond to significance
		levels $\alpha=0.05/0.01/0.005$.}
	\label{fig:tprfpracc1}
\end{figure}

\begin{figure}
	\centering
	\includegraphics[scale=.21,page=9]{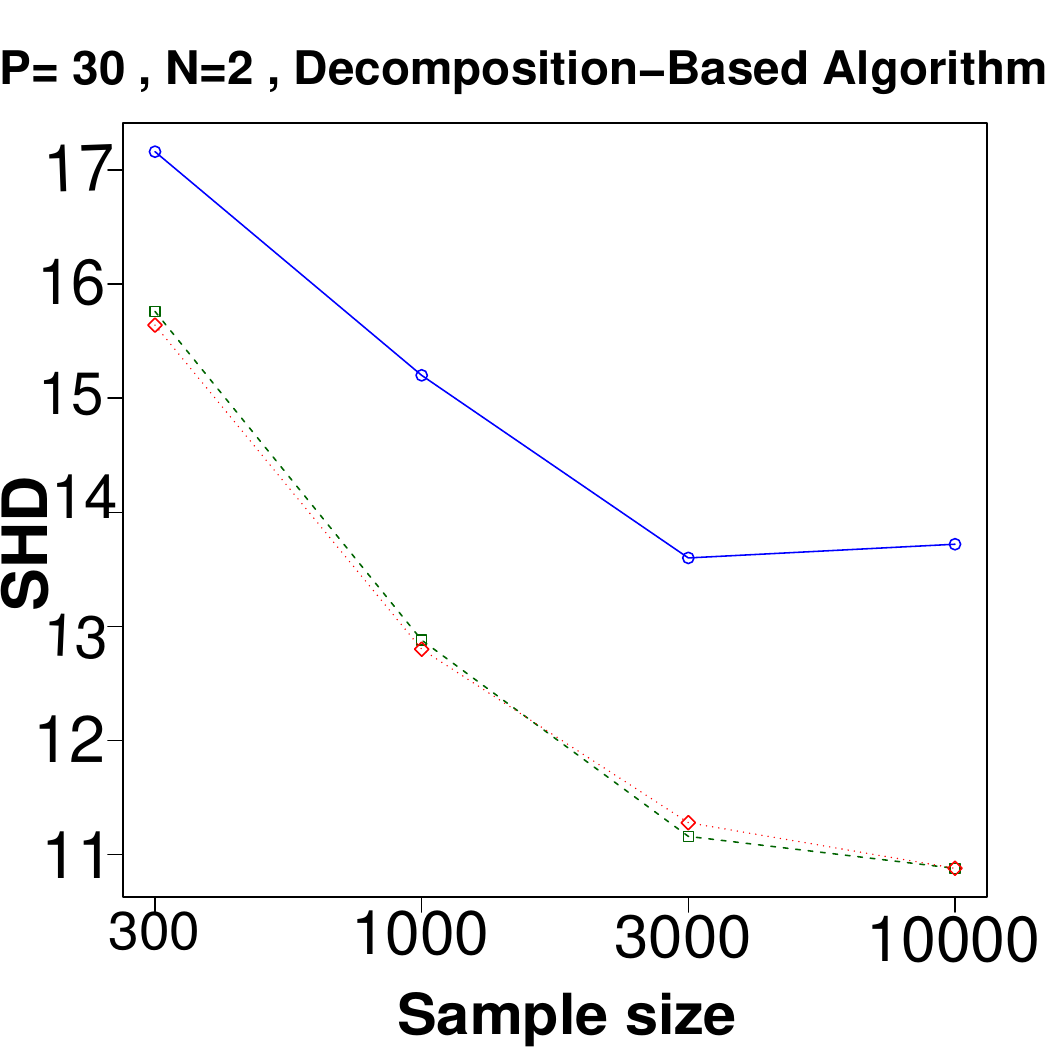}
	\includegraphics[scale=.21,page=9]{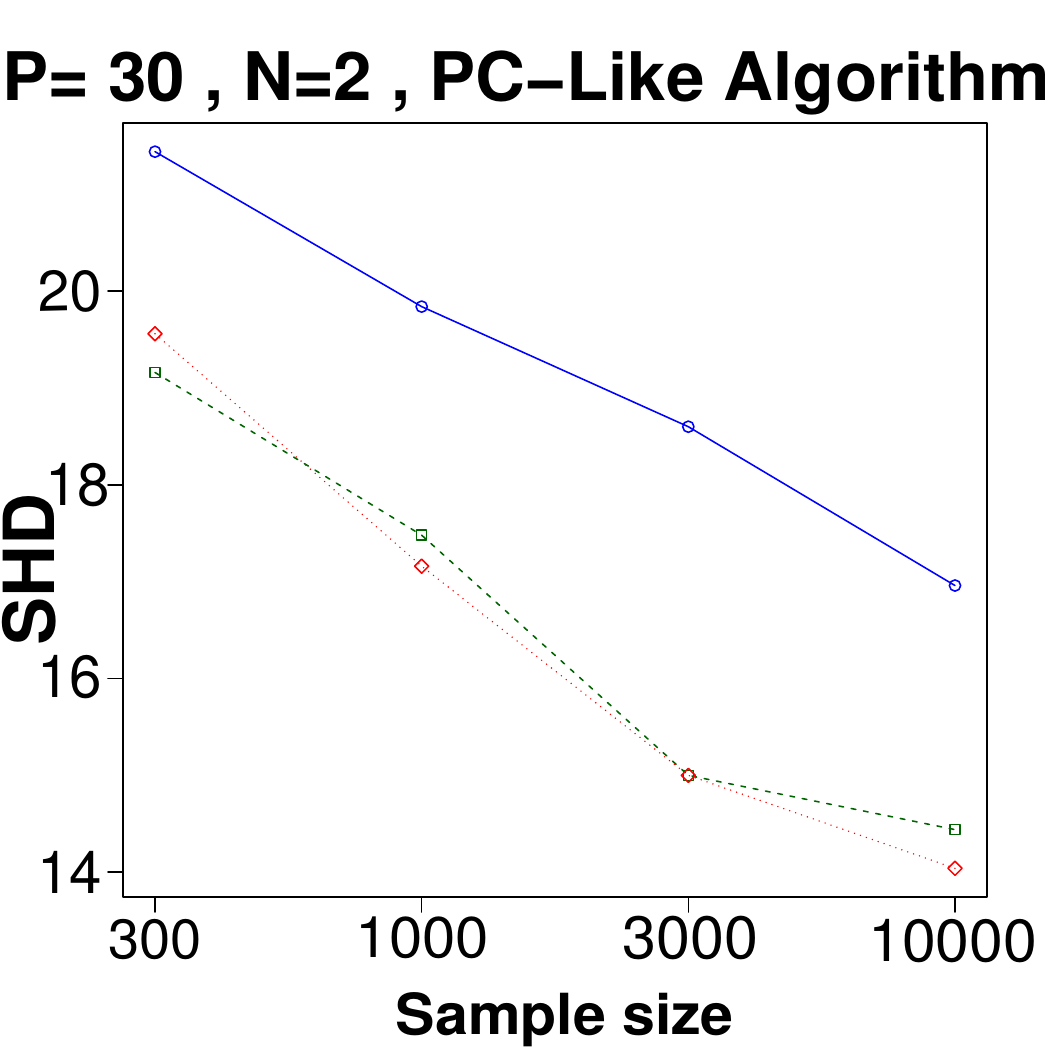}
	\includegraphics[scale=.21,page=9]{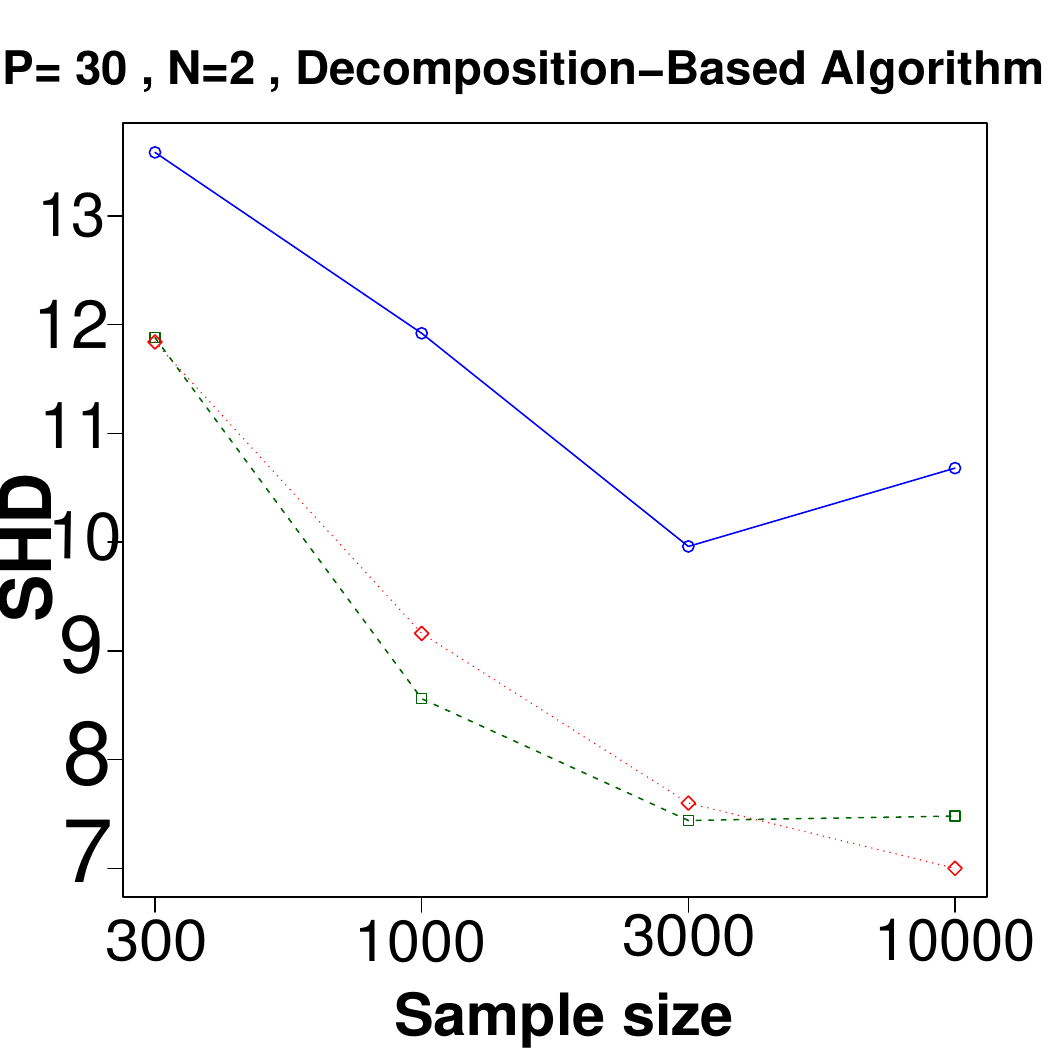}
	\includegraphics[scale=.21,page=9]{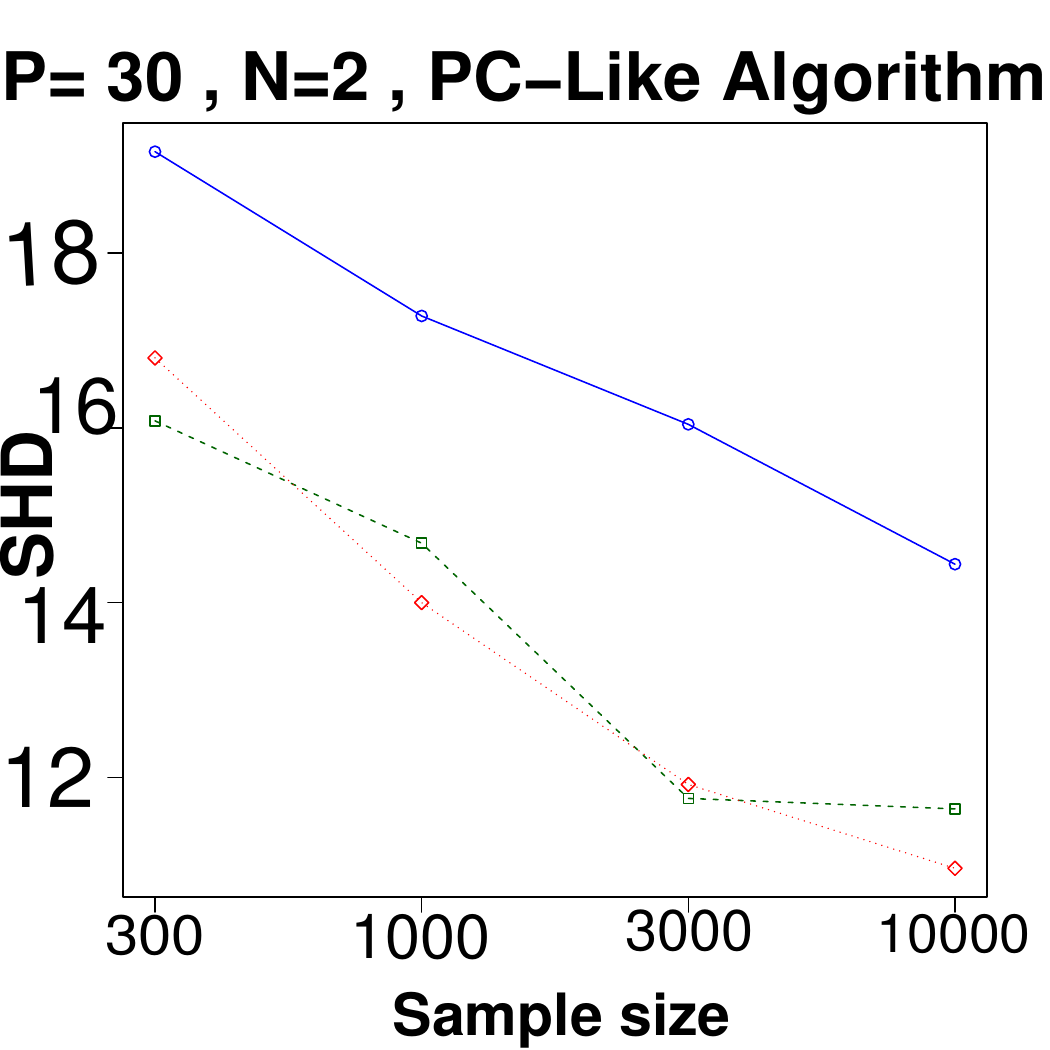}
	\includegraphics[scale=.21,page=11]{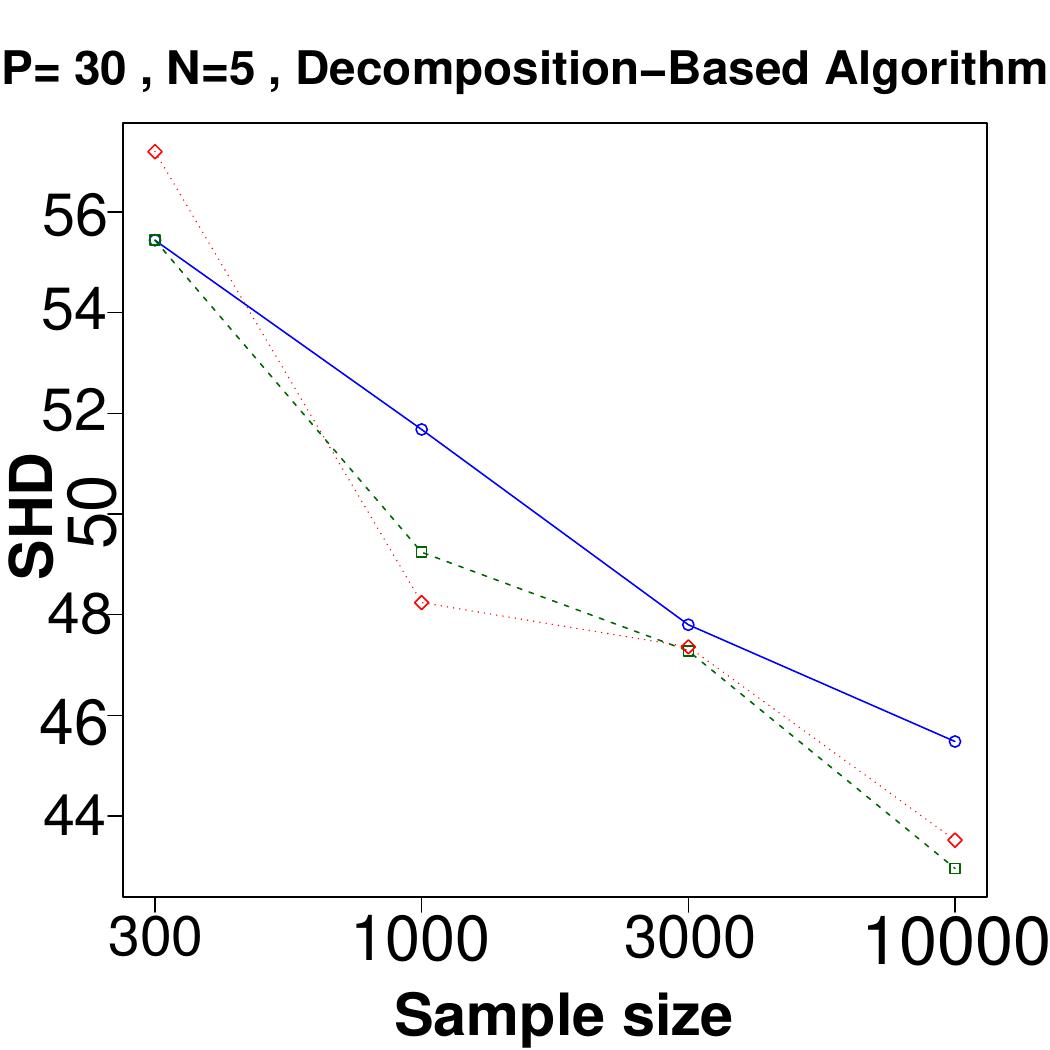}
	\includegraphics[scale=.21,page=11]{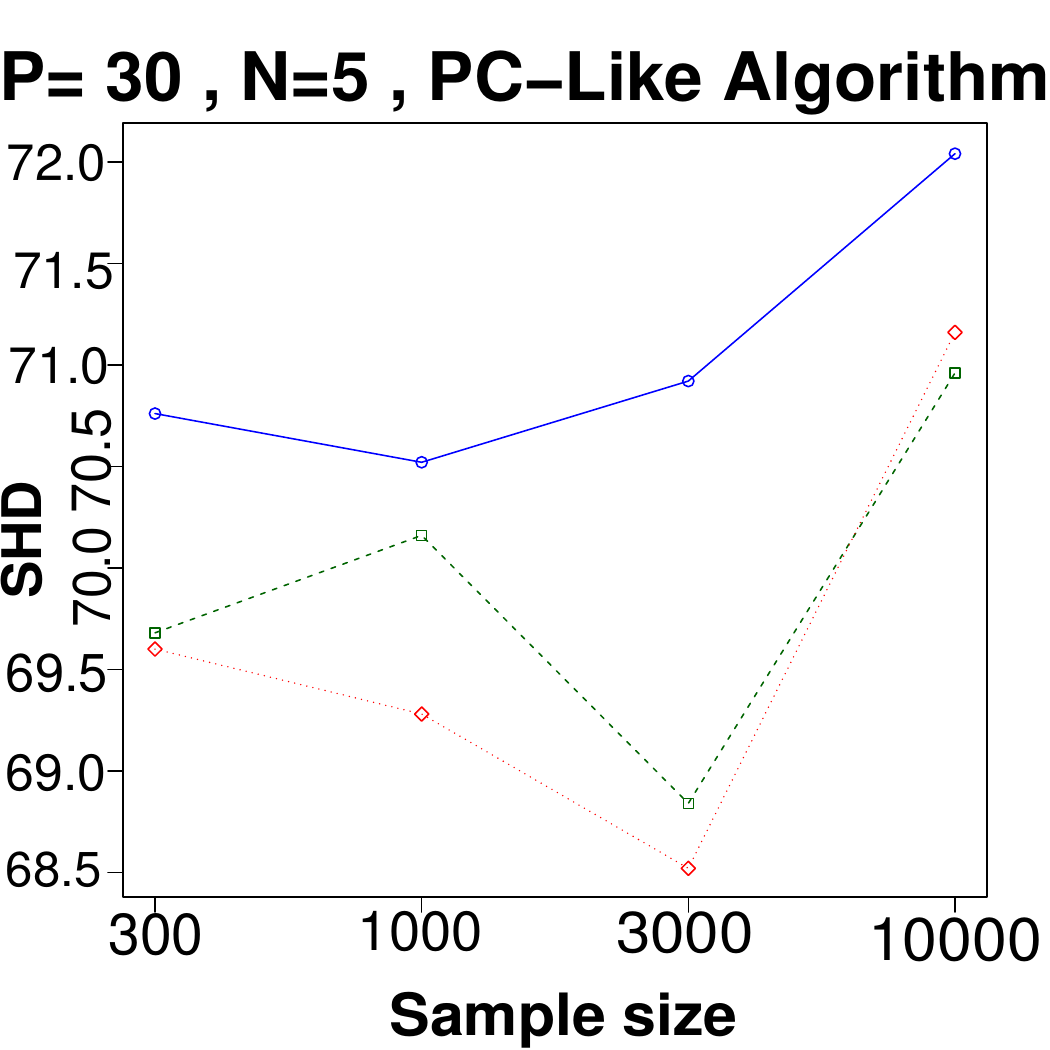}
	\includegraphics[scale=.21,page=11]{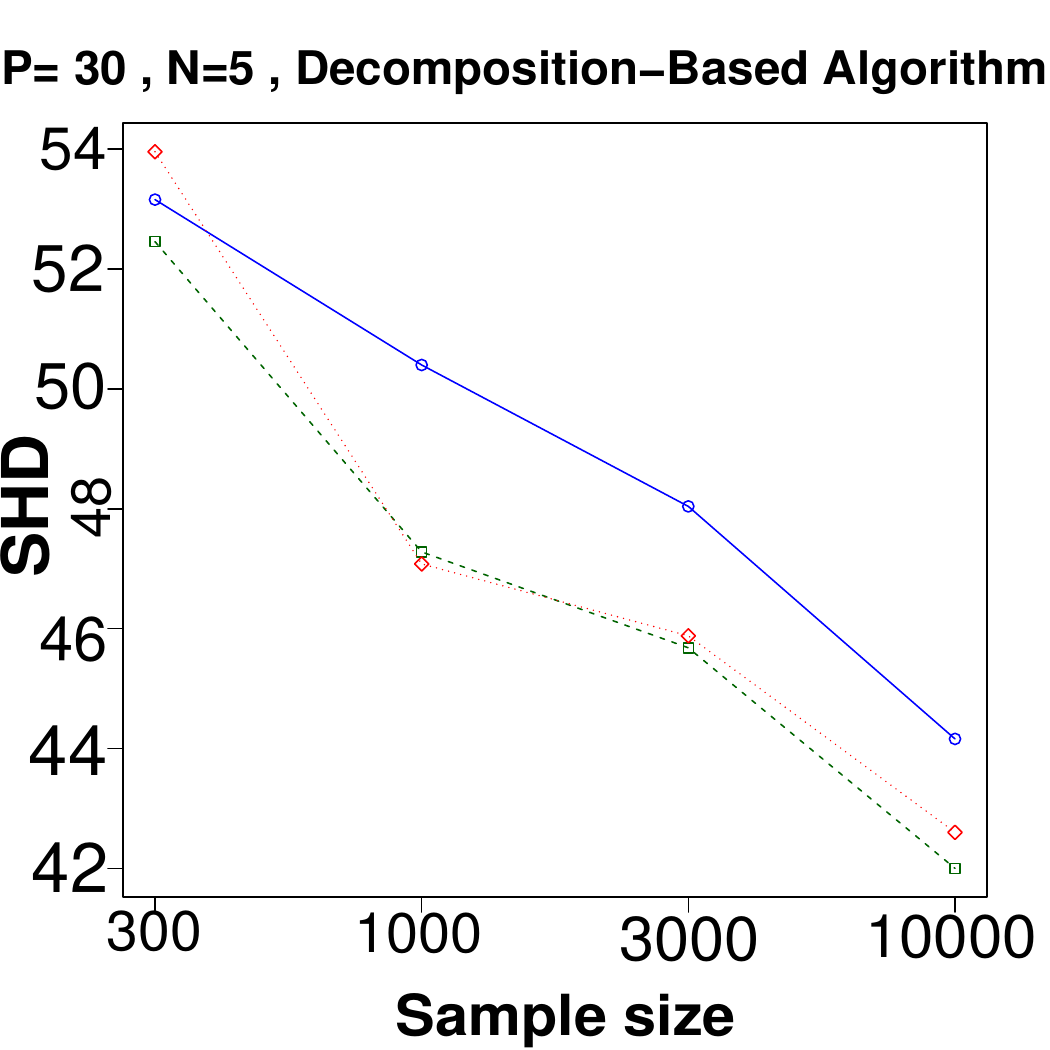}
	\includegraphics[scale=.21,page=11]{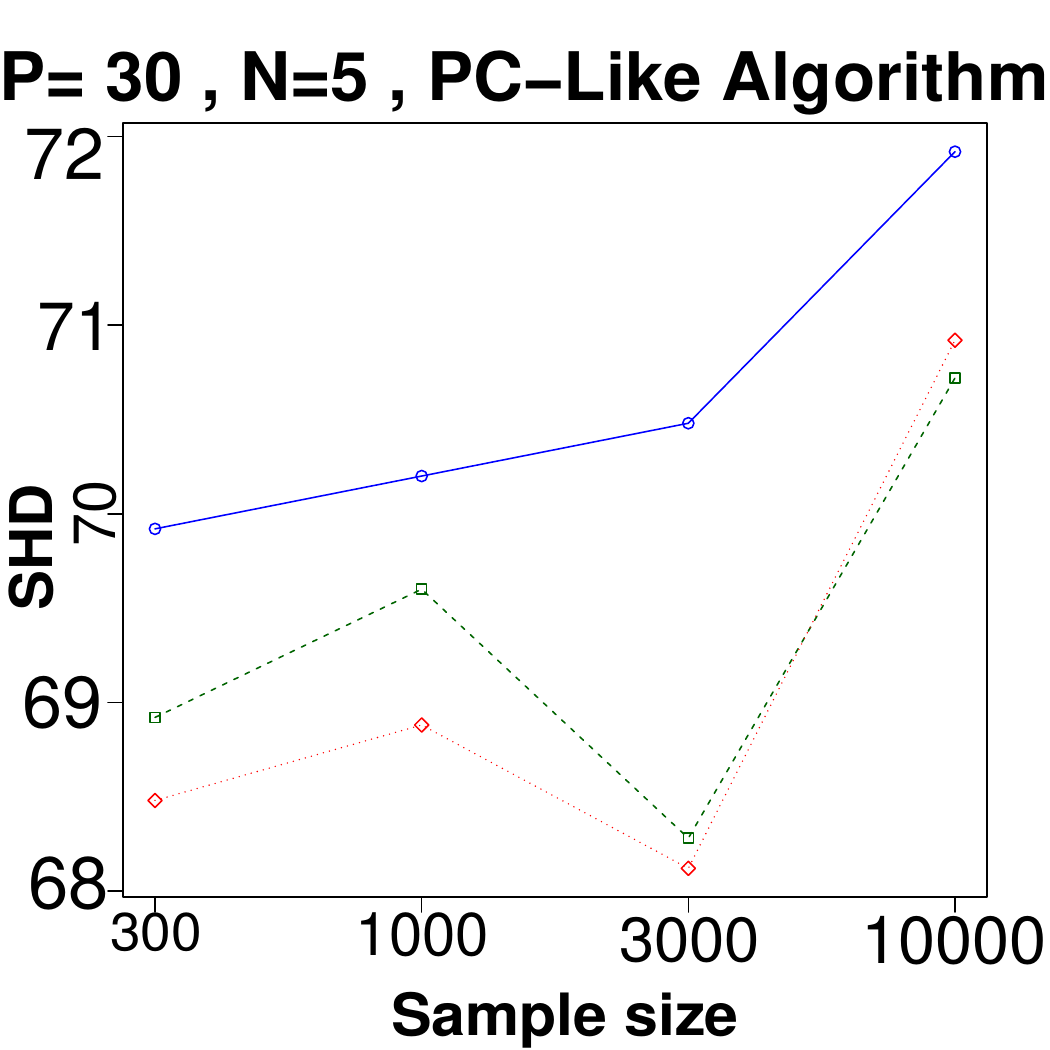}
	\includegraphics[scale=.21,page=12]{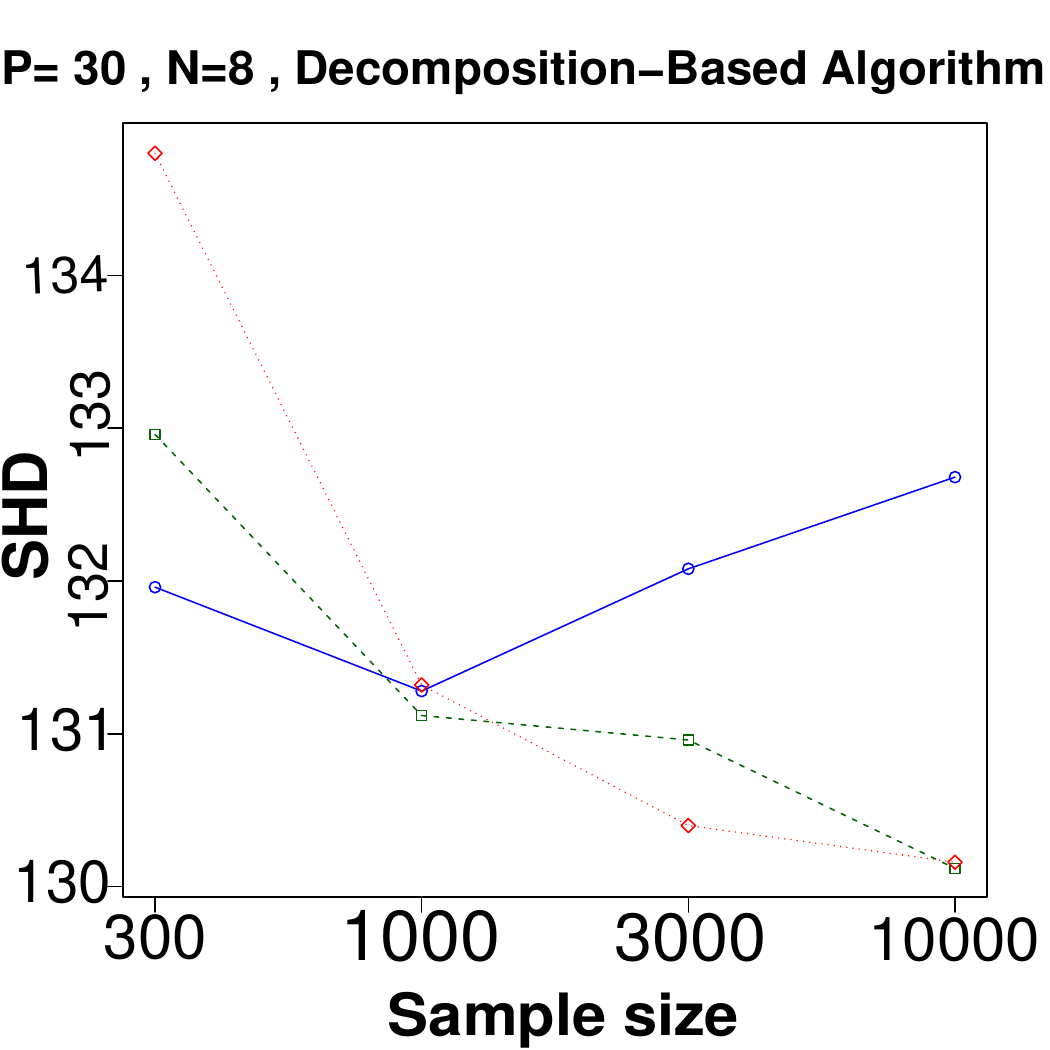}
	\includegraphics[scale=.21,page=12]{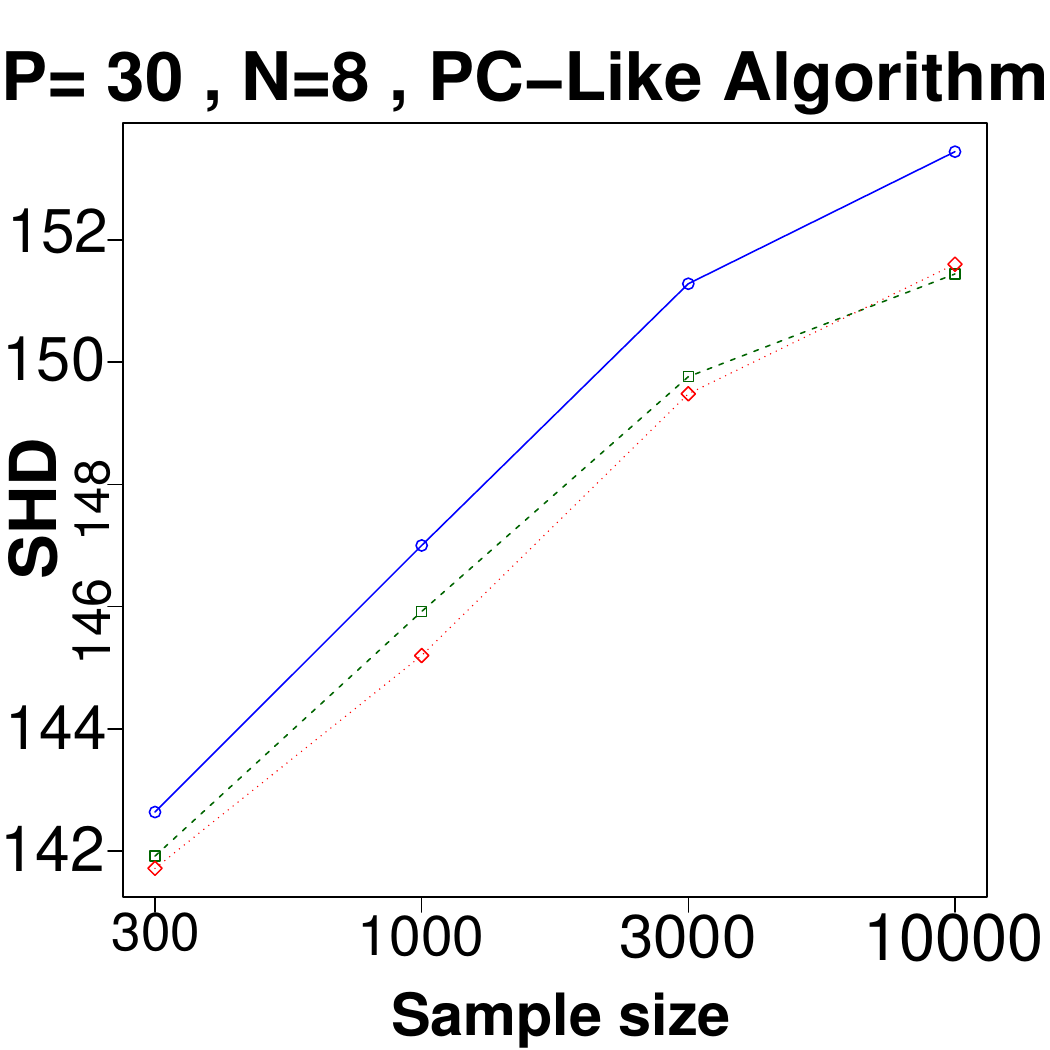}
	\includegraphics[scale=.21,page=12]{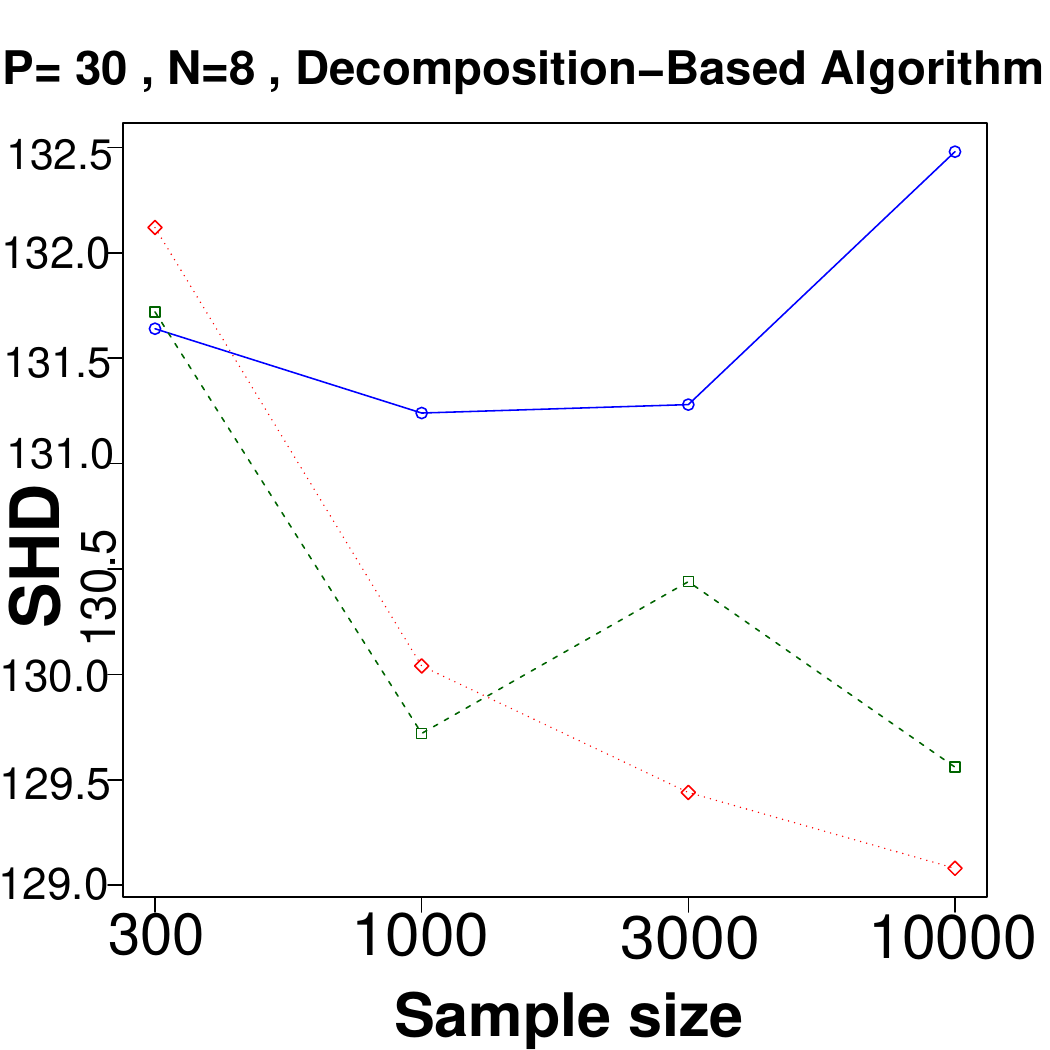}
	\includegraphics[scale=.21,page=12]{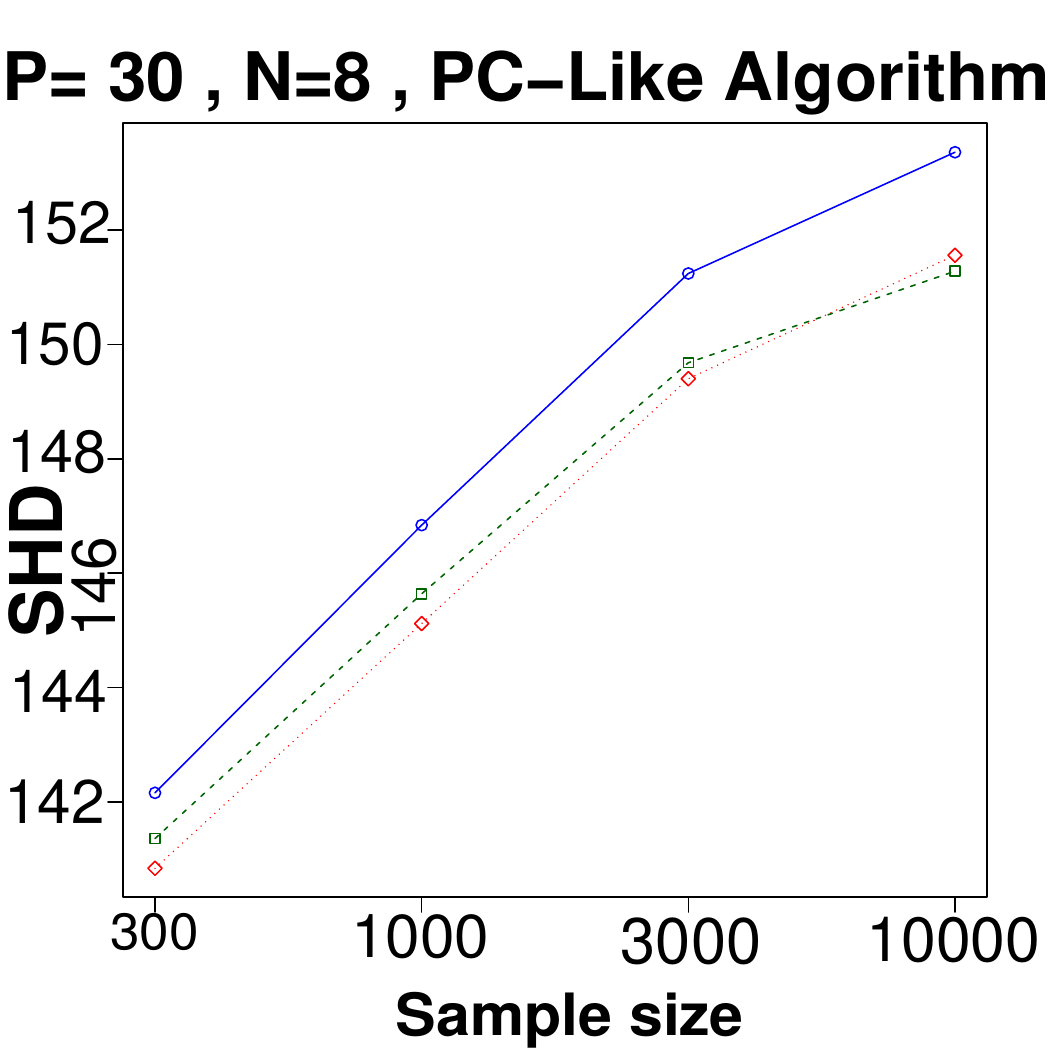}
	\caption{Error measure SHD of the decomposition-based and PC-like algorithms for randomly generated Gaussian chain graph models:
		average over 25 repetitions with 30 variables. The first row correspond to N = 2, the second row correspond to N=5, and the third row correspond to N=8. The first two columns correspond to the essential  recovery while the last two columns correspond to the minimum bidirected recovery respectively. In each plot, the solid (blue)/dashed (green)/dotted (red) lines correspond to significance
		levels $\alpha=0.05/0.01/0.005$.}
	\label{fig:shd1}
\end{figure}

Considering average running times vs. sample sizes, it can be
seen that, for example see Figure \ref{fig:time1}: (a) the average run time increases with sample size; (b) the average run times based on $\alpha=0.01$ and $\alpha=0.005$ are very close and in all cases are better than $\alpha=0.05$,  while
choosing $\alpha=0.005$ yields a consistently (albeit slightly) lower average run time across all the settings in
the current simulation; (c) generally, the average run time for the PC-like algorithm is better than that for the decomposition-based algorithm. One possible justification is related to the details of the implementation. The PC algorithm implementation in the pcalg R package is very well optimized, while we have not concentrated on optimizing our implementation of the LCD algorithm; therefore the comparison on run time may be unfair to the new algorithm.  For future work, one may consider both optimization of the LCD implementation and  instrumentation of the code to allow counting characteristic operations and therefore reducing the dependence of run-time comparison on program optimization.   The simulations were run on an Intel(R) Core(TM) i7-7700HQ  CPU @ 2.80GHz. An R language package that implements our algorithms is available in the \href{https://www.dropbox.com/sh/iynnlwyu8il7m3v/AACk8SyIEn7s-W9NRlLnz0DDa?dl=0}{supplementary document} \citep{jv3}.
\begin{figure}
	\centering
	\includegraphics[scale=.21,page=9]{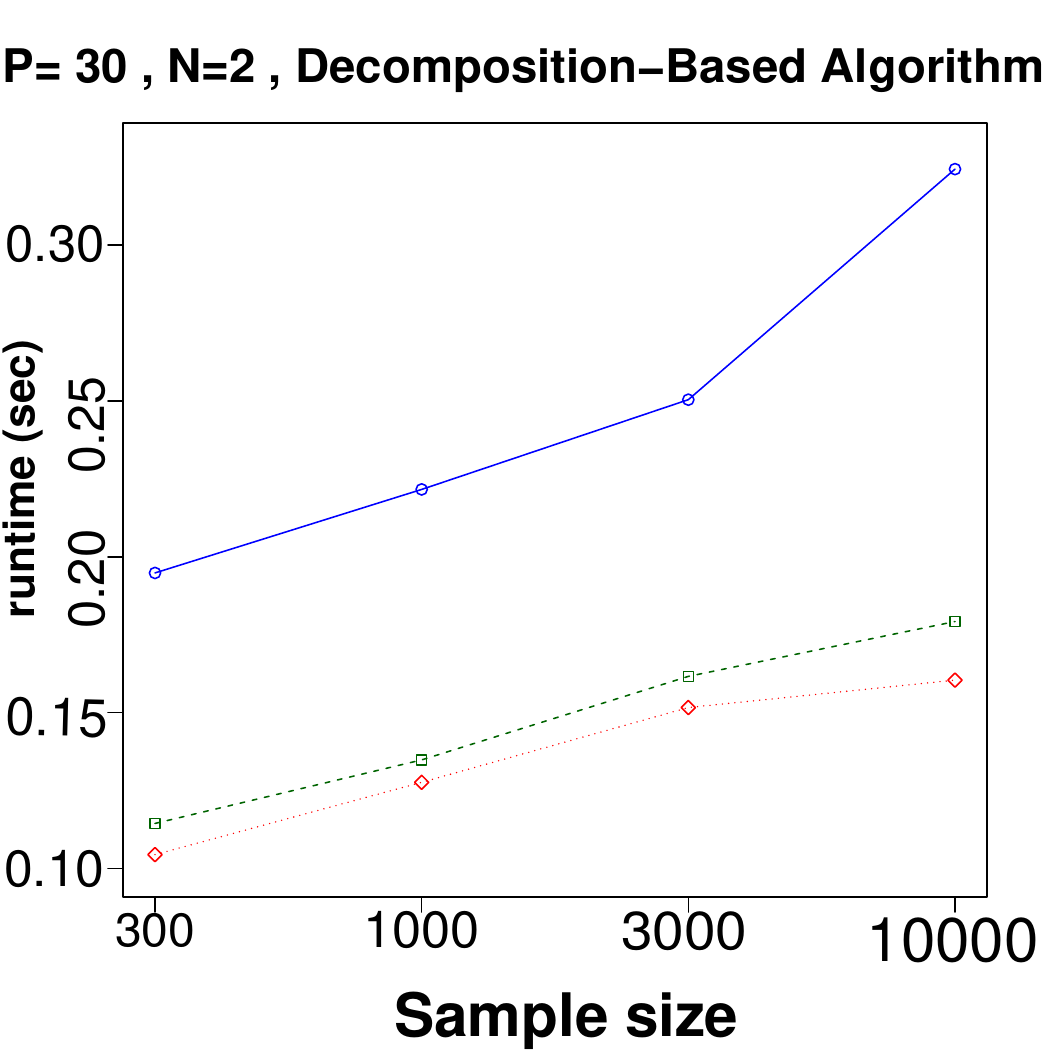}
	\includegraphics[scale=.21,page=9]{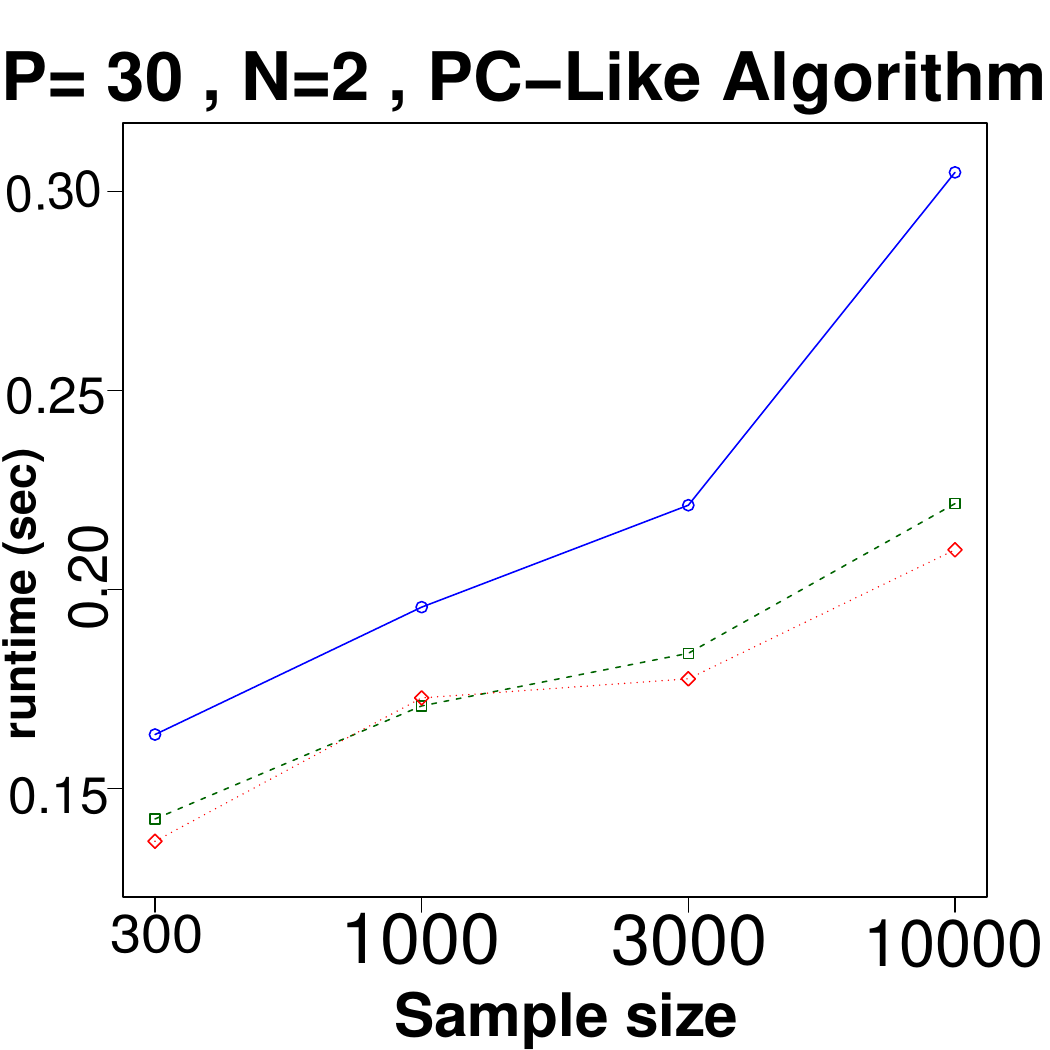}
	\includegraphics[scale=.21,page=9]{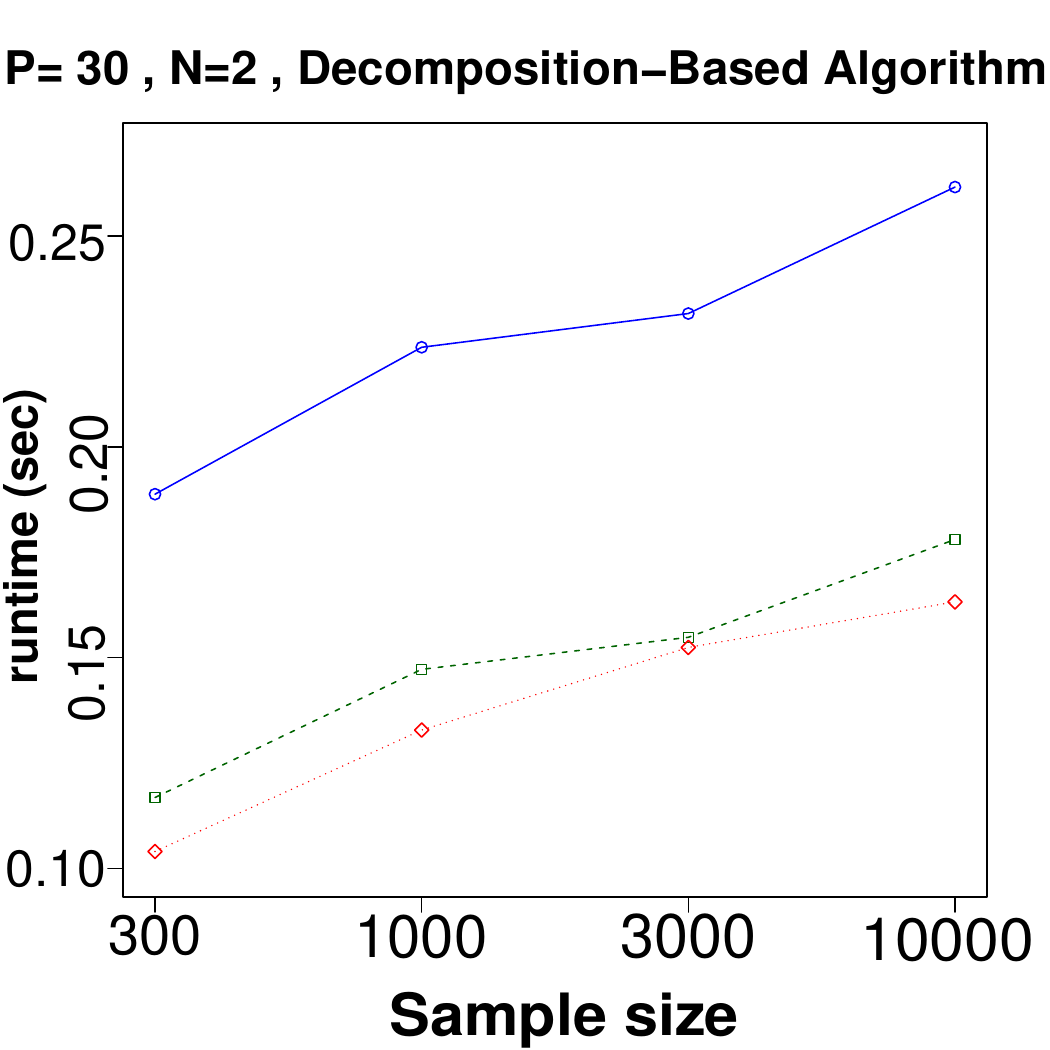}
	\includegraphics[scale=.21,page=9]{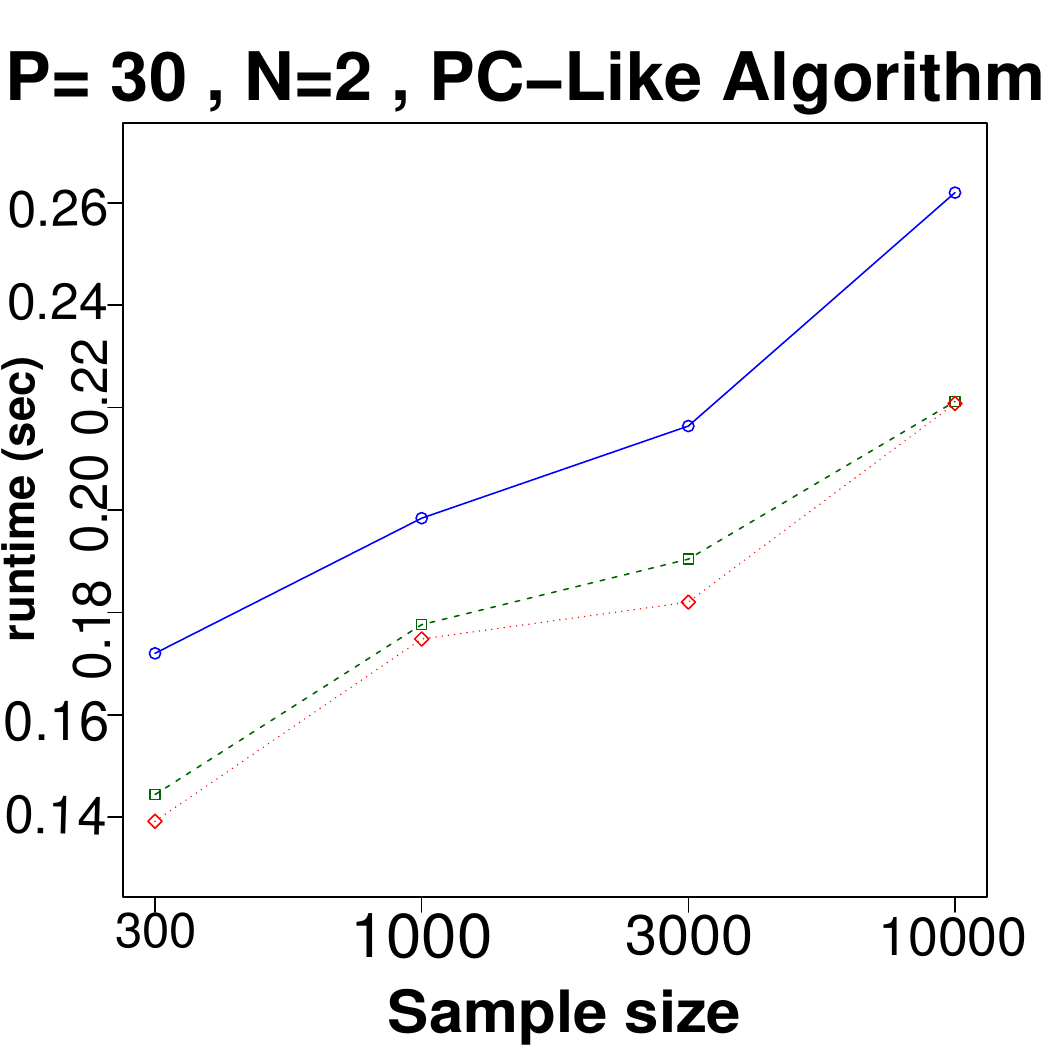}
	\caption{Running times of the decomposition-based and PC-like algorithms for randomly generated Gaussian chain graph models:
		average over 25 repetitions with 30 variables  correspond to N = 2. The first two columns correspond to the essential recovery algorithm while the last two columns correspond to the minimum bidirected recovery respectively. In each plot, the solid (blue)/dashed (green)/dotted (red) lines correspond to significance
		levels $\alpha=0.05/0.01/0.005$.}
	\label{fig:time1}
\end{figure}

It is worth noting that since our implementation of the decomposition-based algorithms is based on the LCD R package, the generated normal random samples from a given MVR chain graph is not guaranteed to be faithful to it. So, one can expect a better performance if we only consider faithful probability distributions in the experiments. Also, the LCD R package uses ${\chi}^2$ test which is an asymptotic test for $G^2$ \citep{mxg}. Again, one can expect a better results if we replace the asymptotic test used in the LCD R package with an exact test. However, there is a trade-off between accuracy and computational time \citep{mxg}.

\subsection{Performance on Discrete Bayesian Networks}
Bayesian networks are special cases of MVR chain graphs. It is of
interest to see whether the decomposition-based algorithms still work well when the data are actually generated
from a Bayesian network. For this purpose, in this subsection, we perform simulation studies for four well-known Bayesian networks from \href{http://www.bnlearn.com/bnrepository/}{Bayesian Network Repository} (Figures \ref{fig:asia}, \ref{fig:insurance}, \ref{fig:alarm}, and \ref{fig:hailfinder}):
\begin{itemize}
    \item ASIA \citep{asia}: with 8 nodes, 8 edges, and 18 parameters, it describes the diagnosis of a patient at a chest clinic who may have just come back from a trip to Asia and may be showing dyspnea. Standard learning algorithms are not able to recover the true structure of the network because of the presence of a functional node (either, representing  logical or)\footnote{\href{https://cran.r-project.org/web/packages/bnlearn/bnlearn.pdf}{Package 'bnlearn'}}. 
    \item INSURANCE \citep{insurance}: with 27 nodes, 52 edges, and 984 parameters, it evaluates car insurance risks.
    \item ALARM \citep{alarm}: with 37 nodes, 46 edges and 509 parameters, it was designed by medical experts to provide an alarm message system for intensive care unit patients based on the output a number of vital signs monitoring devices.
    \item HAILFINDER \citep{Hailfinder}: with 56 nodes, 66 edges, and 2656 parameters, it was designed to forecast severe summer hail in northeastern Colorado.
\end{itemize}

We compare the performance of our algorithms against the PC-like algorithm for these Bayesian networks for three different significance levels $(\alpha=0.05/0.01/0.005)$.

The results of all learning methods are summarized in Table \ref{asia}, \ref{insurance}, \ref{alarm}, and \ref{hailfinder}. 
For the decomposition-based methods, all the three error measures: TPR, FPR and SHD are similar to those of the PC-like algorithms, but the results indicate that  the decomposition-based method outperforms the PC-like algorithms as the size of Bayesian network become larger, especially in terms of TPR and SHD. 

\section{Discussion and Conclusion}\label{discussion}
In this paper, we presented a computationally feasible algorithm for learning the structure of MVR chain graphs via decomposition. We compared the performance of our algorithm with that of the PC-like algorithm proposed by \citep{sp}, in the Gaussian and discrete cases. The PC-like algorithm is a constraint-based algorithm that learns the structure of the underlying MVR chain graph in four steps: (a) determining the skeleton: the resulting undirected graph in this phase contains an undirected edge $u-v$ iff there is no set $S\subseteq V\setminus\{u,v\}$ such that $u\!\perp\!\!\!\perp v|S$; (b) determining the v-structures (unshielded colliders); (c) orienting some of the undirected/directed edges into directed/bidirected edges according to a set of rules applied iteratively; (d) transforming the resulting graph in the previous step into an MVR CG. The essential recovery algorithm obtained after step (c) contains all directed and bidirected edges that are present in every MVR CG of the same Markov equivalence class.  The decomposition-based algorithm is also a constraint-based algorithm that is based on a divide and conquer approach and contains four steps: (a) determining the skeleton by a divide-and-conquer approach; (b) determining the v-structures (unshielded colliders) with localized search for $m$-separators; continuing with steps (c) and (d) exactly as in the PC-like algorithm. The correctness of both algorithms lies upon the assumption that the probability distribution $p$ is faithful
to some MVR CG.  
As for the PC-like algorithms, unless the probability distribution $p$ of the data is faithful to some MVR CG the learned CG cannot be ensured to factorize $p$ properly. Empirical simulations in the Gaussian case show that both algorithms yield good results when the underlying graph is sparse.  The decomposition-based algorithm achieves
competitive results with the PC-like learning algorithm in both Gaussian and discrete cases. 
In fact, the decomposition-based method usually outperforms the PC-like algorithm in all four error measures i.e., TPR, FPR, ACC, and SHD.
Such simulation results confirm that our method is reliable both when latent variables are present (and the underlying graph is an MVR CG) and when there are no such variables (and the underlying graph is a DAG.  The algorithm works reliably when latent variables are present and only fails when selection bias variables are presents. Our algorithm allows relaxing half of the causal sufficiency assumption, because only selection bias needs to be represented explicitly. Since our implementation of the decomposition-based algorithm is based on the LCD R package, with fixed
number of samples, one can expect a better performance if we replace the asymptotic test used in the LCD R package with an exact test. However, there is a trade-off between accuracy and computational time. Also, one can expect a better results if we only consider faithful probability distributions in the experiments. 

The natural continuation of the work presented here would be to develop a learning algorithm with weaker assumptions than the one presented. This could for example be a learning
algorithm that only assumes that the probability distribution satisfies the composition property. It should be mentioned that \citep{psn} developed an algorithm for learning LWF CGs under the composition property. However, \citep{Addendum} proved that the same technique cannot be used for MVR chain graphs. 
We believe that our approach is extendable to the structural learning of AMP chain graphs \citep{amp}. So, the natural continuation of the work presented here would be to develop a learning algorithm via decomposition for AMP chain graphs under the faithfulness assumption.

\begin{table}\centering
\begin{tabular}{c|c|c|c|c}
%\toprule
 & TPR & FPR&ACC& SHD\\
\midrule
    & 0.625&	0.2&	0.75&	9\\
   Decomposition-Based essential recovery algorithm  & 0.625&	0.2&	0.75&	9\\% That's the rule you're looking for.
    &0.625&	0.2&	0.75&	9\\
    \midrule
    &0.625&	0&	0.893&	6\\
  PC-Like essential recovery algorithm Algorithm  &0.625&	0&	0.893&	6 \\
  &0.625&	0&	0.893&	6\\
\midrule
&0.625&	0.2&	0.75&	8\\
Decomposition-Based Algorithm with Minimum bidirected Edges & 0.625&	0.2&	0.75&	7 \\
&0.625&	0.2&	0.75&	8 \\
\midrule
&0.625&	0&	0.893&	4\\
PC-Like Algorithm with Minimum bidirected Edges &0.625&	0&	0.893&	4\\
&0.625&	0&	0.893&	4\\
\bottomrule
\end{tabular}
\caption{Results for discrete samples from the ASIA network. Each row corresponds to the significance
level: $\alpha=0.05/0.01/0.005$ respectively.}\label{asia}
\end{table}

\begin{table}\centering
\begin{tabular}{c|c|c|c|c}
%\toprule
 & TPR & FPR&ACC& SHD\\
\midrule
    &0.635&	0.0167&	0.932&	31\\
   Decomposition-Based essential recovery algorithm   & 0.635&	0.020&	0.926&	32\\% That's the rule you're looking for.
    &0.654&	0.0134&	0.937&	28\\
    \midrule
    &0.558&	0&	0.934&	37\\
  PC-Like essential recovery algorithm Algorithm  &0.519&	0&	0.929&	37\\
  &0.519&	0&	0.929&	37\\
\midrule
&0.635&	0.0167&	0.932&	30\\
Decomposition-Based Algorithm with Minimum bidirected Edges & 0.635&	0.020&	0.926&	32 \\
&0.654&	0.0134&	0.937&	27 \\
\midrule
&0.558&	0&	0.934&	27\\
PC-Like Algorithm with Minimum bidirected Edges &0.519&	0&	0.929&	29\\
&0.519&	0&	0.929&	29\\
\bottomrule
\end{tabular}
\caption{Results for discrete samples from the INSURANCE network. Each row corresponds to the significance
level: $\alpha=0.05/0.01/0.005$ respectively.}\label{insurance}
\end{table}

\begin{table}\centering
\begin{tabular}{c|c|c|c|c}
%\toprule
 & TPR & FPR&ACC& SHD\\
\midrule
    &0.783&	0.0194&	0.967&34\\
   Decomposition-Based essential recovery algorithm  &0.783&	0.0161&	0.967&32\\% That's the rule you're looking for.
    &0.761&	0.021&	0.964&	36\\
    \midrule
    &0.457&	0&	0.962&	38\\
  PC-Like essential recovery algorithm Algorithm  &0.435&	0&	0.961&	38\\
  &0.413&	0&	0.959&	41\\
\midrule
&0.783&	0.0194&	0.967&30\\
Decomposition-Based Algorithm with Minimum bidirected Edges &0.783&	0.0161&	0.967&28 \\
&0.761&	0.021&	0.964&	35\\
\midrule
&0.457&	0&	0.962&	33\\
PC-Like Algorithm with Minimum bidirected Edges &0.435&	0&	0.961&	33\\
&0.413&	0&	0.959&	36\\
\bottomrule
\end{tabular}
\caption{Results for discrete samples from the ALARM network. Each row corresponds to the significance
level: $\alpha=0.05/0.01/0.005$ respectively.}\label{alarm}
\end{table}

\begin{table}\centering
\begin{tabular}{c|c|c|c|c}
%\toprule
 & TPR & FPR&ACC& SHD\\
\midrule
    &0.758&	0.003&	0.986&	26\\
   Decomposition-Based essential recovery algorithm  &0.742&	0.002&	0.987&	24\\% That's the rule you're looking for.
    &0.757&	0.002&	0.988&	22\\
    \midrule
    &0.457&	0&	0.962&	38\\
  PC-Like essential recovery algorithm Algorithm  &0.515&	0.0007&	0.979&	40\\
  &0.515&	0.0007&	0.979&	40\\
\midrule
&0.758&	0.003&	0.986&	42\\
Decomposition-Based Algorithm with Minimum bidirected Edges &0.742&	0.002&	0.987&41 \\
&0.757&	0.002&	0.988&	24\\
\midrule
&0.457&	0&	0.962&	38\\
PC-Like Algorithm with Minimum bidirected Edges &0.515&	0.0007&	0.979&	38\\
&0.515&	0.0007&	0.979&39\\
\bottomrule
\end{tabular}
\caption{Results for discrete samples from the HAILFINDER network. Each row corresponds to the significance
level: $\alpha=0.05/0.01/0.005$ respectively.}\label{hailfinder}
\end{table}

\begin{figure}
    \centering
    \includegraphics[scale=.25]{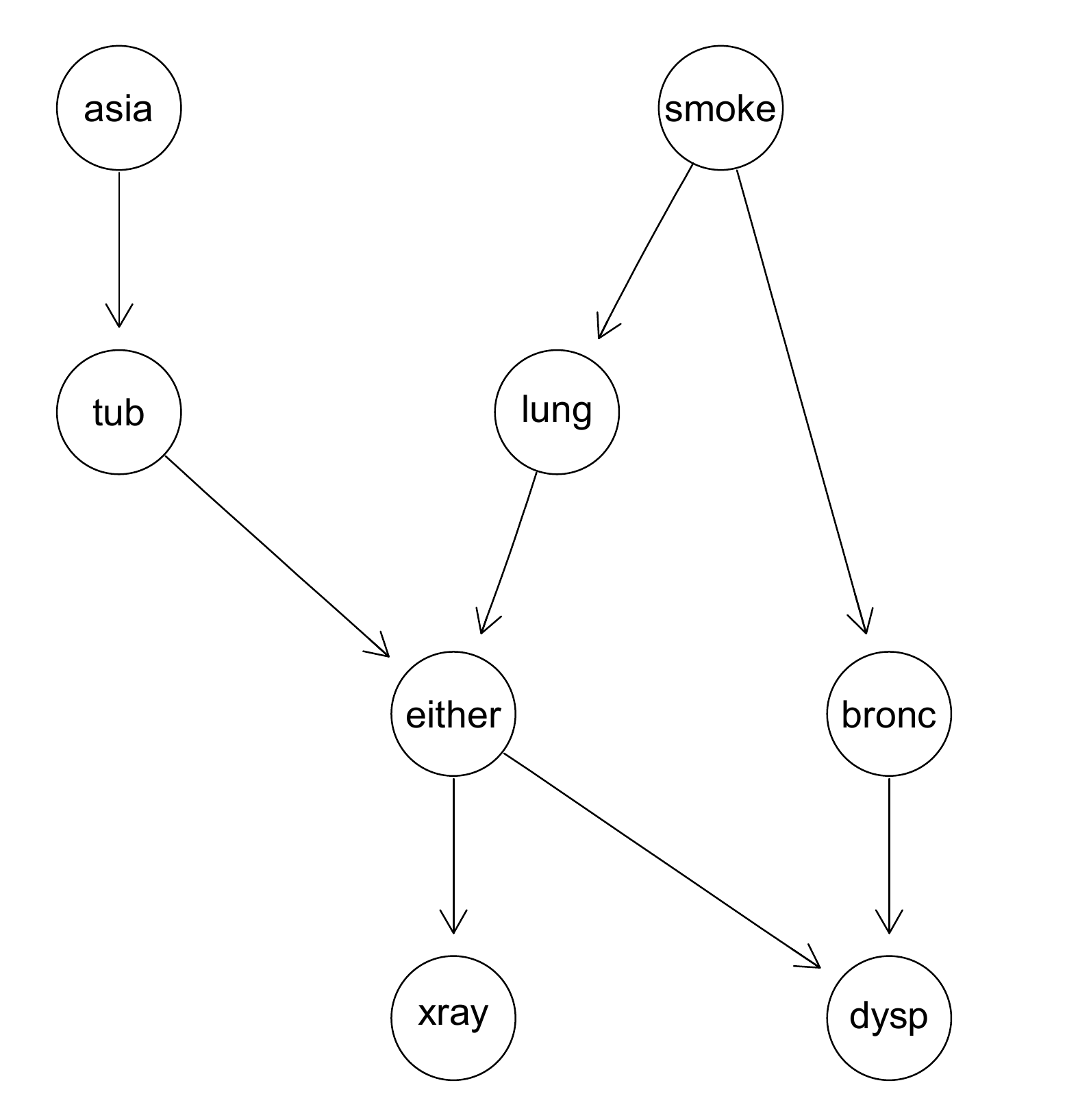}
    \caption{\href{http://www.bnlearn.com/bnrepository/}{ASIA (sometimes called LUNG CANCER or CHEST CLINIC)} ,
Number of nodes: 8,
Number of arcs: 8,
Number of parameters: 18,
Average Markov blanket size: 2.50,
Average degree: 2.00,
Maximum in-degree: 2.}
    \label{fig:asia}
\end{figure}

\begin{figure}
    \centering
    \includegraphics[scale=.5]{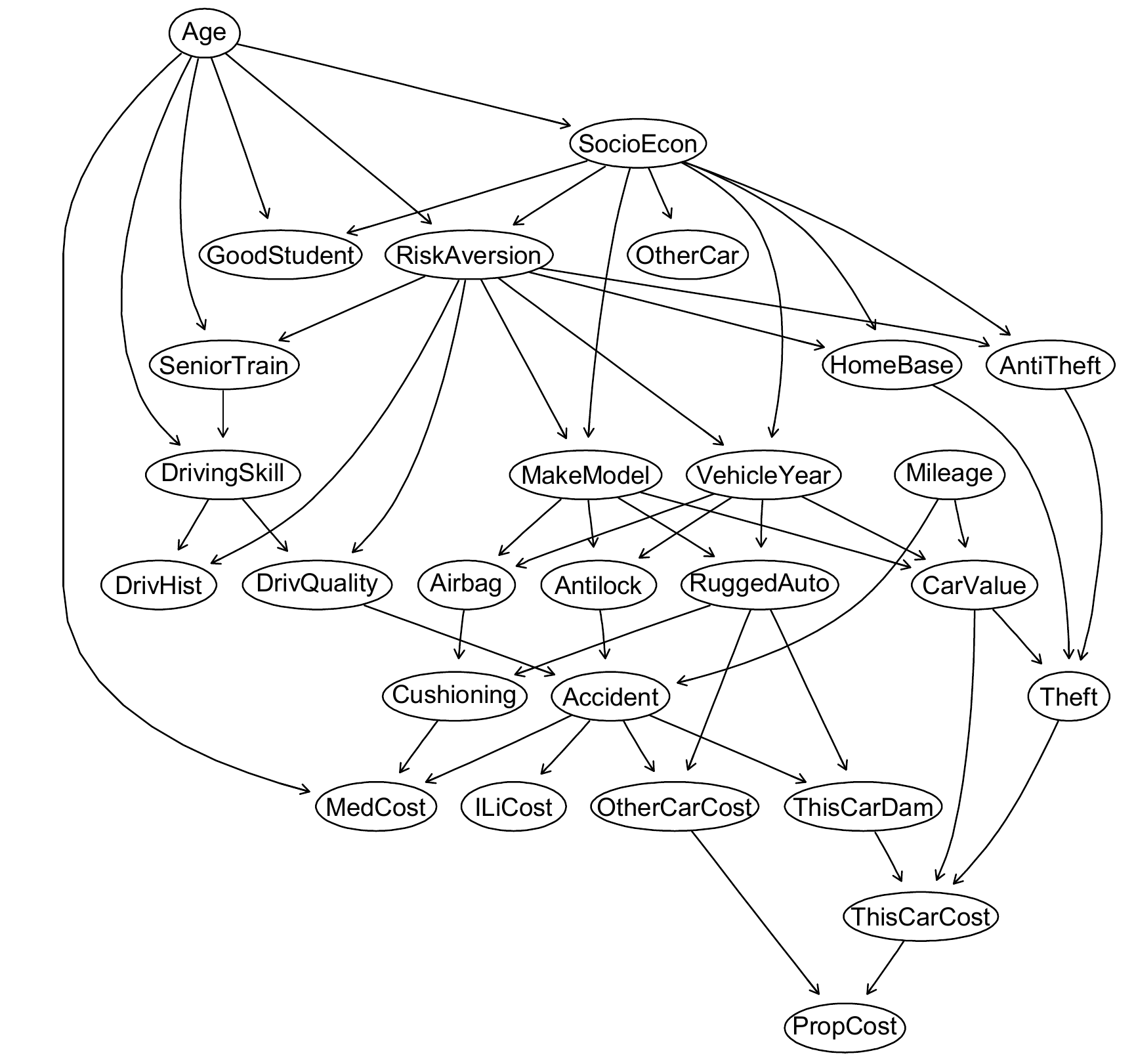}
    \caption{\href{http://www.bnlearn.com/bnrepository/}{INSURANCE} ,
Number of nodes: 27,
Number of arcs: 52,
Number of parameters: 984,
Average Markov blanket size: 5.19,
Average degree: 3.85
Maximum in-degree: 3.}
    \label{fig:insurance}
\end{figure}

\begin{figure}
    \centering
    \includegraphics[scale=.5]{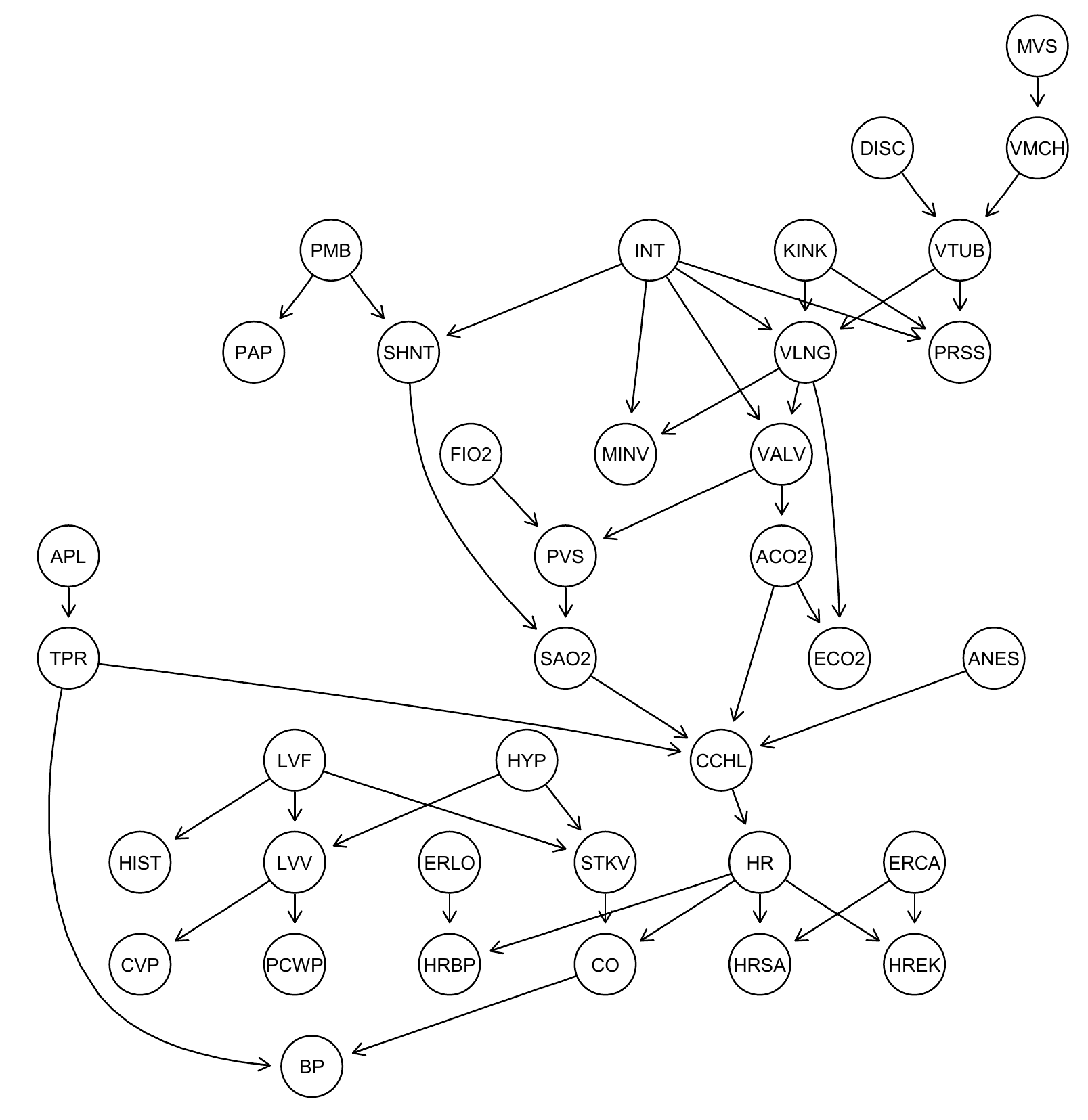}
    \caption{\href{http://www.bnlearn.com/bnrepository/}{ALARM} ,
Number of nodes: 37,
Number of arcs: 46,
Number of parameters: 509,
Average Markov blanket size: 3.51,
Average degree: 2.49,
Maximum in-degree: 4.}
    \label{fig:alarm}
\end{figure}

\begin{figure}
    \centering
    \includegraphics[scale=.9]{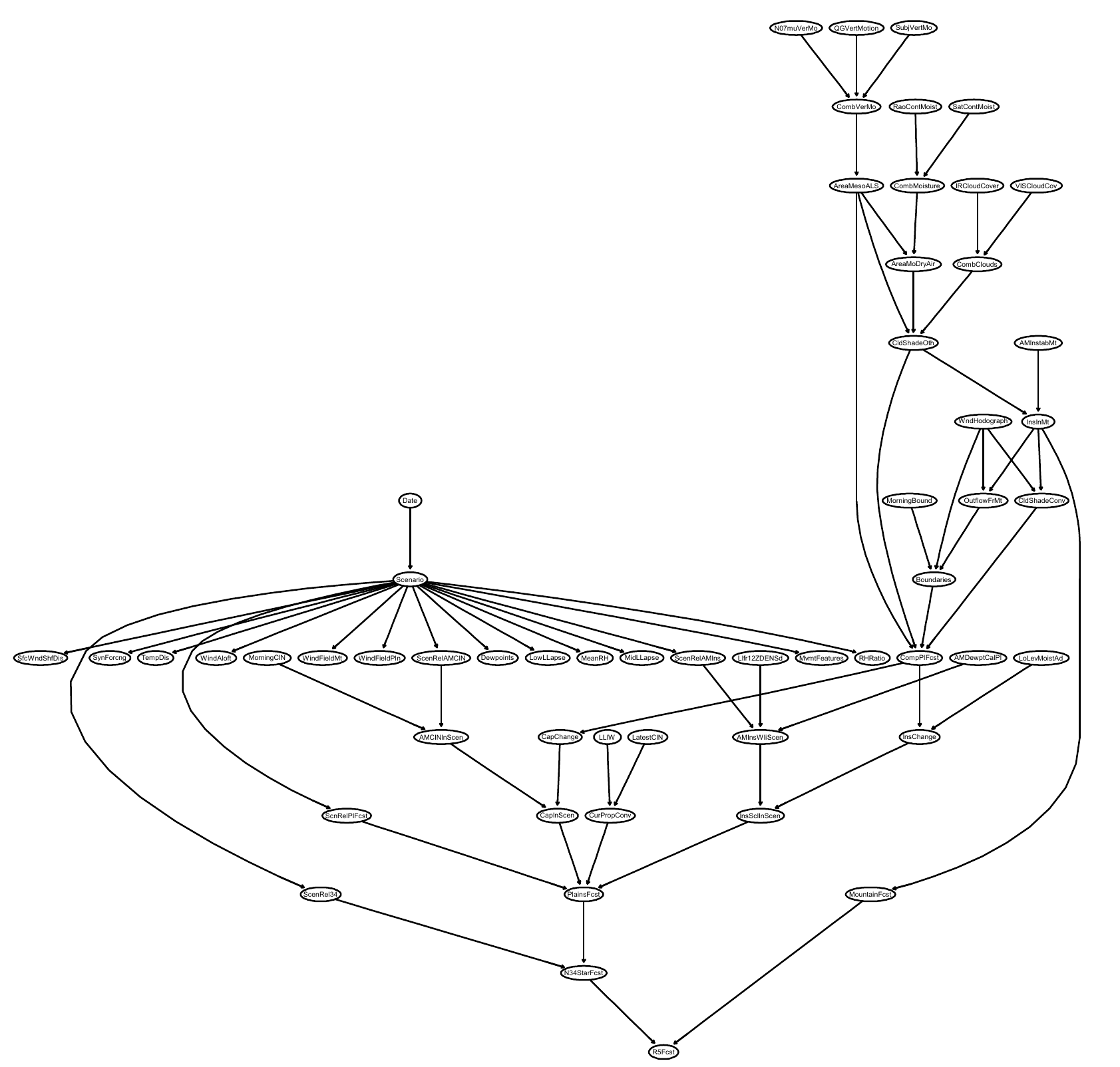}
    \caption{\href{http://www.bnlearn.com/bnrepository/}{HAILFINDER} ,
Number of nodes: 56,
Number of arcs: 66,
Number of parameters: 2656,
Average Markov blanket size: 3.54,
Average degree: 2.36
Maximum in-degree: 4.}
    \label{fig:hailfinder}
\end{figure}

\section*{Appendix A. Proofs of Theoretical Results}

\begin{lemma}\label{lem1}
	Let $\rho$ be a chain from $u$ to $v$, and $W$ be the set of all vertices on $\rho$ ($W$ may or may not contain $u$ and $v$).
	Suppose that (the endpoints of) a chain $\rho$ is (are) blocked by $S$. If $W\subseteq S$, then the chain $\rho$ is blocked by $W$ and by any set containing $W$.
\end{lemma}
\begin{proof}
	Since the blocking of the chain $\rho$ depends on those vertices between $u$ and $v$ that are contained in the $m$-separator, 
	and since $W$ contains all vertices on $\rho$, $\rho$ is also blocked by $S \cap W = W$ if $\rho$ is blocked by $S$. Since all colliders on $\rho$
	have already been activated conditionally on $W$, adding other vertices into the conditional set does not make any
	new collider active on $\rho$. This implies that $\rho$ is blocked by any set containing $W$.	
\end{proof}
\begin{lemma}\label{lem2}
	Let $T$ be an $m$-separation tree for  CG $G$, and $K$ be a separator of $T$ that separates $T$ into two
	subtrees $T_1$ and $T_2$ with variable sets $V_1$ and $V_2$ respectively. Suppose that $\rho$ is a chain from $u$ to $v$ in $G$ where $u\in V_1\setminus K$ and $v\in V_2\setminus K$. Let $W$ denote the set of all vertices on $\rho$ ($W$ may or may not contain $u$ and $v$). Then the
	chain $\rho$ is blocked by $W\cap K$ and by any set containing $W\cap K$. 
\end{lemma}
\begin{proof}
	Since $u\in V_1\setminus K$ and $v\in V_2\setminus K$, there is a sequence from $s$ (may be $u$) to $y$ (may be $v$) in $\rho =(u,\dots,s,t,\dots,x,y,\dots,v)$ such that $s\in V_1\setminus K$ and $y\in V_2\setminus K$ and all vertices from $t$ to $x$ are contained in $K$. Let $\rho'$ be the sub-chain of $\rho$ from $s$ to $y$ and $W'$ the vertex set from $t$ to $x$, so $W'\subseteq K$. Since $s\in V_1\setminus K$ and $y\in V_2\setminus K$, we have from definition of $m$-separation tree that $K$  $m$-separates $s$ and $y$ in $G$, i.e., $K$ blocks $\rho'$. By lemma \ref{lem1}, we obtain that $\rho'$ is blocked by $W'(\subseteq K)$ and any set containing $W'$. Since $W'\subseteq (K\cap W)$, $\rho'$ is blocked by $K\cap W$ and by any set containing $K\cap W$. Thus $\rho(\supseteq \rho')$ is also blocked by them.
\end{proof}

\begin{remark}\label{rem1}
	Javidian and Valtorta showed that if we find a separator over $S$ in $(G_{An(u\cup v)})^a$ then it is an $m$-separator in $G$. On the other hand, if there exists an $m$-separator over $S$ in $G$ then there must exist a separator over $S$ in $(G_{An(u\cup v)})^a$ by removing all nodes which are not in $An(u\cup v)$ from it \citep{jv2}.
\end{remark} 
Observations in Remark \ref{rem1} yield the following results.

\begin{lemma}\label{lem3}
	Let $u$ and $v$ be two non-adjacent vertices in MVR CG $G$, and let $\rho$ be a chain from $u$ to $v$. If $\rho$ is not contained in
	$An(u\cup v)$, then $\rho$ is blocked by any subset S of $an(u\cup v)$.
\end{lemma}
\begin{proof}
	Since $\rho \not\subseteq An(u\cup v)$, there is a sequence from $s$ (may be $u$) to $y$ (may be $v$) in $\rho=(u,\dots,s,t,\dots,x,y,\dots,v)$ such that $s$ and $y$ are contained in $An(u\cup v)$ and all vertices from $t$ to $x$ are out of $An(u\cup v)$.Then the edges $s-t$ and $x-y$ must be oriented as $s\lacircle t$ and $x\racircle y$, otherwise $t$ or $x$ belongs to $an(u\cup v)$. Thus there exist at least one collider between $s$ and $y$ on $\rho$. The middle vertex $w$ of the collider closest to $s$ between $s$ and $y$ is not contained in
	$an(u\cup v)$, and any descendant of $w$ is not in $an(u\cup v)$, otherwise there is a (partially) directed cycle. So $\rho$ is blocked
	by the collider, and it cannot be activated conditionally on any vertex in $S$ where $S\subseteq an(u\cup v)$.
\end{proof}

\begin{lemma}\label{lem4}
	Let $T$ be an $m$-separation tree for  CG $G$. For any vertex $u$ there exists at least one node of $T$ that contains $u$ and $bd(u)$.
\end{lemma}
\begin{proof}
	If $bd(u)$ is empty, it is trivial. Otherwise let $C$ denote the node of $T$ which contains $u$ and the most elements
	of $u$'s boundary. Since no set can separate $u$ from a parent (or neighbor), there must be a node of $T$ that contains $u$ and the parent (or neighbor). If $u$ has only
	one parent (or neighbor), then we obtain the lemma. If $u$ has two or more elements in its boundary, we choose two arbitrary elements $v$ and $w$ of $u$'s boundary that are not contained in a single node but are contained in two different nodes of $T$,
	say $\{u,v\}\subseteq C$ and  $\{u,w\}\subseteq C'$ respectively, since all vertices in $V$ appear in $T$. On the chain from $C$ to $C'$ in $T$, all
	separators must contain $u$, otherwise they cannot separate $C$ from $C'$. However, any separator containing $u$ cannot
	separate $v$ and $w$ because $v\lacircle u\racircle w$ is an active chain between $v$ and $w$ in $G$. Thus we got a contradiction.
\end{proof}

\begin{lemma}\label{lem5}
	Let $T$ be an $m$-separation tree for  CG $G$ and $C$ a node of $T$. If $u$ and $v$ are two vertices in $C$ that
	are non-adjacent in $G$, then there exists a node $C'$ of $T$ containing $u, v$ and a set $S$ such that $S$ $m$-separates $u$ and $v$ in $G$.
\end{lemma}
\begin{proof}
	Without loss of generality, we can suppose that $v$ is not a descendant of the vertex $u$ in $G$, i.e., $v\not\in nd(u)$. According to the local Markov property for MVR chain graphs proposed by Javidian and Valtorta in \citep{jv1}, we know that $u\perp\!\!\!\perp [nd(u)\setminus bd(u)]|pa_G(u).$ By Lemma \ref{lem4}, there is a
	node $C_1$ of $T$ that contains $u$ and $bd(u)$. If $v\in C_1$, then $S$ defined as the parents of $u$ $m$-separates $u$ from $v$.
	
	If $v\not\in C_1$, choose the node $C_2$ that is the closest node in $T$ to the node $C_1$ and that contains $u$ and $v$. Consider that there is at least one parent (or neighbor) $p$ of $u$ that is not contained
	in $C_2$. Thus there is a separator $K$ connecting $C_2$ toward $C_1$ in $T$ such that $K$ $m$-separates $p$ from all vertices in $C_2\setminus K$. Note that on the chain from $C_1$ to $C_2$ in $T$, all
	separators must contain $u$, otherwise they cannot separate $C_1$ from $C_2$. So, we have $u\in K$ but $v\not\in K$ (if $v\in K$, then $C_2$ is not the closest node of $T$ to the node $C_1$). In fact, for every parent (or neighbor) $p'$ of $u$ that is contained in $C_1$ but not in $C_2$,  $K$ separates $p'$ from all vertices in $C_2\setminus K$, especially the vertex $v$.
	 
	Define $S=(an(u\cup v)\cap C_2)$, which is a subset of $C_2$. We need to show that $u$ and $v$ are $m$-separated by $S$, that is, every chain between $u$
	and $v$ in $G$ is blocked by $S$.
	
	If $\rho$ is not contained in $An(u\cup v)$, then we obtain from Lemma \ref{lem3} that $\rho$ is blocked by $S$.
	
	When $\rho$ is contained in $An(u\cup v)$, let $x$ be adjacent to $u$ on $\rho$, that is, $\rho =
	(u, x, y, \dots , v)$. We consider the three possible orientations of the edge between $u$ and $x$.  We now show that $\rho$ is blocked in all three cases.
	\begin{itemize}
		\item[i:] $u\gets x$, so we know that $x$ is not a collider and we have two possible sub-cases:
			\begin{enumerate}
				\item $x\in C_2$. In this case the chain $\rho$ is blocked at $x$.
				\item $x\not\in C_2$. In this case $K$ $m$-separates $x$ from $v$. By
				Lemma \ref{lem2}, we can obtain that the sub-chain $\rho'$ from $x$ to $v$ can be blocked by $W\cap K$ where $W$ denotes the set of
				all vertices between $x$ and $v$ (not containing $x$ and $v$) on $\rho'$. Since $S\supseteq (W\cap K)$, we obtain from Lemma \ref{lem2} that $S$ also blocks $\rho'$. Hence the chain $\rho$ is blocked by $S$.
			\end{enumerate}
		\item[ii:] $u\to x$. We have the following sub-cases:
			\begin{enumerate}
				\item $x\in an(u)$. This case is impossible because a directed cycle would occur.
				\item $x\in an(v)$. This case is impossible because $v$ cannot be a descendant of $u$.
			\end{enumerate}
		\item[iii:] $u\leftrightarrow x$. We have the following sub-cases:
			\begin{enumerate}
				\item $x\in an(u)$. This case is impossible because a partially directed cycle would occur.
				\item $x\in an(v)$ and $v$ is in the same chain component $\tau$ that contains $u, x$. This is impossible, because in this case we have a partially directed cycle. 
				\item $x\in an(v)$ and $v$ is not in the same chain component $\tau$ that contains $u, x$. We have the following sub-cases:
				\begin{itemize}
					\item $x\not\in C_2$. In this case $K$ $m$-separates $x$ from $v$. By
					Lemma \ref{lem2}, we can obtain that the sub-chain $\rho'$ from $x$ to $v$ can be blocked by $W\cap K$ where $W$ denotes the set of
					all vertices between $x$ and $v$ (not containing $x$ and $v$) on $\rho'$. Since $S\supseteq (W\cap K)$, we obtain from Lemma \ref{lem2} that $S$ also blocks $\rho'$. Hence the chain $\rho$ is blocked by $S$.
					\item $x\in C_2$. We have the three following sub-cases:
					\begin{itemize}
						\item $u\leftrightarrow x\to y$. In this case $x\in S$ blocks the chain. Note that in this case it is possible that $y=v$.
						\item $u\leftrightarrow x\gets y$. So, $y$ ($\ne v$ o.w., a directed cycle would occur) is not a collider. If $y\in C_2$ then the chain $\rho$ is blocked at $y$. Otherwise, we have the two following sub-cases:
						\begin{itemize}
							\item There is a node $C'$ between $C_1$ and $C_2$ that contains $y$ (note that it is possible that $C'=C_1$), so $K$ $m$-separates $y$ from $v$ and the same argument used for case i.2 holds.
							\item In this case $K$ $m$-separates $y$ from $p$ ($p\in bd(u)\cap C_1$ and $p\not\in C_2$), which is impossible because the chain $p\lacircle u\leftrightarrow x\gets y$ is active (note that $u,x\in K$).  
						\end{itemize}
						\item $u\leftrightarrow x\leftrightarrow y$. If there is an outgoing ($\to$) edge from $y$ ($\ne v$ o.w., a partially directed cycle would occur) then the same argument in the previous sub-case ($u\leftrightarrow x\gets y$) holds. Otherwise, $y$ is a collider. If $y\not\in C_2$ then the chain $\rho$ is blocked at $y$. If $y\in C_2$, there must be a non-collider vertex on the chain $\rho$ between $y$ and $v$ to prevent a (partially) directed cycle. The same argument as in the previous sub-case ($u\leftrightarrow x\gets y$) holds.
					\end{itemize}
				\end{itemize} 
			\end{enumerate}
	\end{itemize}  
\end{proof}

\begin{proof}[Proof of Theorem \ref{thm2}]
	From \citep{cdls}, we know that any separator $S$ in junction tree $T$ separates $V_1\setminus S$ and $V_2\setminus S$ in the triangulated graph $\bar{G}_V^t$, where $V_i$ denotes the variable set of the subtree $T_i$ induced by removing the edge with
	a separator $S$ attached, for $i = 1, 2$. Since the edge set of $\bar{G}_V^t$ contains that of undirected independence graph $\bar{G}_V$ for $G$, $V_1\setminus S$ and $V_2\setminus S$ are also separated in $\bar{G}_V$. Since $\bar{G}_V$ is an undirected independence graph for $G$, using  Definition \ref{septree} we obtain that $T$ is an $m$-separation tree for $G$.
\end{proof}

\begin{proof}[Proof of Theorem \ref{thm1}] 
	\noindent ($\Rightarrow$) If condition (i) is the case, nothing remains to prove. Otherwise, Lemma \ref{lem5} implies condition (ii).
	
	\noindent ($\Leftarrow$) Assume that $u$ and $v$ are not contained together in any node $C$ of $T$. Also, assume that $C_1$ and $C_2$ are two nodes of $T$ that contain $u$ and $v$, respectively. Consider that $C_1'$ is the most distant node from $C_1$, between $C_1$ and $C_2$, that contains $u$ and $C_2'$ is the most distant node from $C_2$, between $C_1$ and $C_2$, that contains $v$. Note that it is possible that $C_1'=C_1$ or $C_2'=C_2$. By the condition (i) we know that $C_1'\ne C_2'$. Any separator between $C_1'$ and $C_2'$ satisfies the assumptions of Lemma \ref{lem2}. The sufficiency of condition (i) is given by Lemma \ref{lem2}.
	
	The sufficiency of
	conditions (ii) is trivial by the definition of $m$-separation.
\end{proof}

\section*{Appendix B. Proofs for Correctness of the Algorithms}
\begin{proof}
    [Correctness of Algorithm \ref{hypergraph}] Since an augmented graph for  CG $G$ is an undirected independence graph, by definition of
	an undirected independence graph, it is enough to show that $\bar{G}_V$ defined in step 3 contains all edges of $(G_V)^a$. It
	is obvious that $\bar{E}$ contains all edges obtained by dropping directions of directed edges in $G$ since any set cannot
	$m$-separate two vertices that are adjacent in $G$.
	
	Now we show that $\bar{E}$ also contains any augmented edge that connects vertices $u$ and $v$ having a collider chain between them, that
	is, $(u, v)\in \bar{E}$. Any chain graph yields a directed acyclic graph $D$ of its chain components having $\mathcal{T}$ as a node set and an edge $T_1\to T_2$ whenever there exists in the chain graph $G$ at least one edge $u\rightarrow v$ connecting a node \textit{u} in $T_1$ with a node \textit{v} in $T_2$ \citep{ml2}. So, there is a collider chain between two nodes $u$ and $v$ if and only if there is a chain component $\tau\in \mathcal{T}$ such that
	\begin{enumerate}
		\item $u,v\in \tau$, or
		\item $u\in \tau$ and $v\in pa_G(\tau)$ or vice versa, or
		\item $u, v\in pa_G(\tau)$
	\end{enumerate}
	Since for each connected component $\tau$ there is a $C_h\in C$ containing both $\tau$ and its parent set $pa_G(\tau)$, in all of above mentioned cases we have an $(u,v)$ edge in step 2. Therefore, $\bar{G}_V$ defined in step 3 contains all edges of $(G_V)^a$. 
\end{proof}

\begin{proof}
    [Correctness of Algorithm \ref{alg1}] By the sufficiency of Theorem \ref{thm1}, the initializations at steps 2 and 3 for
	creating edges guarantee that no edge is created between any two variables which are not in the same node of the
	$m$-separation tree. Also, by the sufficiency of Theorem \ref{thm1}, deleting edges at steps 2 and 3 guarantees that any other edge
	between two $m$-separated variables can be deleted in some local skeleton. Thus the global skeleton obtained at step 3 is
	correct. In a maximal ancestral graph, every missing edge corresponds to at least one independency in the corresponding
	independence model \citep{rs}, and MVR CGs are a subclass of maximal ancestral graphs \citep{jv1}. Therefore, according to the necessity of Theorem \ref{thm1}, each augmented edge $(u, v)$ in the undirected independence graph must be deleted at some subgraph over a node of the $m$-separation tree. Furthermore, according to Lemma \ref{lem4}, for every $v$-structure $(u\lacircle w\racircle v)$ there is a node in $m$-separation tree $T$ that contains $u, v$ and $w$, and obviously $w\not\in S_{uv}$. Therefore, we can determine all $v$-structures at step 4, which
	completes our proof. 
\end{proof}

% \begin{proof}[Correctness of \hyperref[alg3]{Main Algorithm}]
% 	The correctness of \hyperref[alg1]{Algorithm 1} and \hyperref[alg2]{Algorithm 2} has been proved above. Thus we
% 	only need to prove that step 3 is correct for obtaining an $m$-separation tree. With an argument similar to that used in the proof of correctness of \hyperref[alg2]{Algorithm 2}, we have that the
% 	entire undirected independence graph is constructed correctly at the step 3. Then the $m$-separation tree can be obtained
% 	correctly by using an algorithm for constructing a junction tree.
% \end{proof}

\section*{Acknowledgements}
We are grateful to Professor Jose M.  Pe\~{n}a and Dr. Dag Sonntag for providing us with code that helped in the design of the algorithm that we implemented in R. 

% This work has been partially supported by Office of Naval Research grant ONR N00014-17-1-2842. This research is based upon work supported in part by the Office of the Director of National
% Intelligence (ODNI), Intelligence Advanced Research Projects Activity (IARPA), award/contract
% number 2017-16112300009. The views and conclusions contained therein are those of the authors
% and should not be interpreted as necessarily representing the official policies, either expressed or implied,
% of ODNI, IARPA, or the U.S. Government. The U.S. Government is authorized to reproduce
% and distribute reprints for governmental purposes, notwithstanding annotation therein.

% \section*{conflict of interest}
% You may be asked to provide a conflict of interest statement during the submission process. Please check the journal's author guidelines for details on what to include in this section. Please ensure you liaise with all co-authors to confirm agreement with the final statement.

%\printendnotes

% Submissions are not required to reflect the precise reference formatting of the journal (use of italics, bold etc.), however it is important that all key elements of each reference are included.
\bibliography{sample}

\begin{thebibliography}{42}
\providecommand{\natexlab}[1]{#1}
\providecommand{\url}[1]{\texttt{#1}}
\expandafter\ifx\csname urlstyle\endcsname\relax
  \providecommand{\doi}[1]{doi: #1}\else
  \providecommand{\doi}{doi: \begingroup \urlstyle{rm}\Url}\fi

\bibitem[Abramson et~al.(1996)Abramson, Brown, Edwards, Murphy, and
  Winkler]{Hailfinder}
B.~Abramson, J.~Brown, W.~Edwards, A.~Murphy, and R.~L. Winkler.
\newblock Hailfinder: A bayesian system for forecasting severe weather.
\newblock \emph{International Journal of Forecasting}, 12\penalty0
  (1):\penalty0 57 -- 71, 1996.
\newblock Probability Judgmental Forecasting.

\bibitem[Andersson et~al.(1996)Andersson, Madigan, and Perlman]{amp}
S.~A. Andersson, D.~Madigan, and M.~D. Perlman.
\newblock An alternative {M}arkov property for chain graphs.
\newblock In E.~Horvitz and F.~V. Jensen, editors, \emph{Proceedings of the
  Twelfth Conference on Uncertainty in artificial intelligence}, pages 40--48,
  1996.

\bibitem[Beinlich et~al.(1989)Beinlich, Suermondt, Chavez, and Cooper]{alarm}
I.~A. Beinlich, H.~J. Suermondt, R.~M. Chavez, and G.~F. Cooper.
\newblock The alarm monitoring system: A case study with two probabilistic
  inference techniques for belief networks.
\newblock In J.~Hunter, J.~Cookson, and J.~Wyatt, editors, \emph{AIME 89},
  pages 247--256, Berlin, Heidelberg, 1989. Springer Berlin Heidelberg.

\bibitem[Berry et~al.(2004)Berry, Blair, Heggernes, and Peyton]{bbhp}
A.~Berry, J.~Blair, P.~Heggernes, and B.~Peyton.
\newblock Maximum cardinality search for computing minimal triangulations of
  graphs.
\newblock \emph{Algorithmica}, 39:\penalty0 287--298, 2004.

\bibitem[Binder et~al.(1997)Binder, Koller, Russell, and Kanazawa]{insurance}
J.~Binder, D.~Koller, S.~Russell, and K.~Kanazawa.
\newblock Adaptive probabilistic networks with hidden variables.
\newblock \emph{Machine Learning}, 29\penalty0 (2):\penalty0 213--244, Nov
  1997.

\bibitem[Colombo et~al.(2012)Colombo, Maathuis, Kalisch, and Richardson]{cmkr}
D.~Colombo, M.~H. Maathuis, M.~Kalisch, and T.~S. Richardson.
\newblock Learning high-dimensional directed acyclic graphs with latent and
  selection variables.
\newblock \emph{The Annals of Statistics}, 40\penalty0 (1):\penalty0 294--321,
  2012.

\bibitem[Cowell et~al.(1999)Cowell, Dawid, Lauritzen, and Spiegelhalter]{cdls}
R.~Cowell, A.~P. Dawid, S.~Lauritzen, and D.~J. Spiegelhalter.
\newblock \emph{Probabilistic networks and expert systems. Statistics for
  Engineering and Information Science}.
\newblock Springer-Verlag, 1999.

\bibitem[Cox and Wermuth(1993)]{cw1}
D.~R. Cox and N.~Wermuth.
\newblock Linear dependencies represented by chain graphs.
\newblock \emph{Statistical Science}, 8\penalty0 (3):\penalty0 204--218, 1993.

\bibitem[Cox and Wermuth(1996)]{cw2}
D.~R. Cox and N.~Wermuth.
\newblock \emph{Multivariate Dependencies-Models, Analysis and Interpretation}.
\newblock Chapman and Hall, 1996.

\bibitem[Drton(2009)]{d}
M.~Drton.
\newblock Discrete chain graph models.
\newblock \emph{Bernoulli}, 15\penalty0 (3):\penalty0 736--753, 2009.

\bibitem[Edwards(2000)]{ed}
D.~Edwards.
\newblock \emph{Introduction to Graphical Modelling. 2nd Ed.}
\newblock Springer-Verlag, New York, 2000.

\bibitem[Evans and Richardson(2014)]{er}
R.~Evans and T.~S. Richardson.
\newblock Markovian acyclic directed mixed graphs for discrete data.
\newblock \emph{The Annals of Statistics}, 42\penalty0 (4):\penalty0
  1452--1482, 2014.

\bibitem[Frydenberg(1990)]{f}
M.~Frydenberg.
\newblock The chain graph markov property.
\newblock \emph{Scandinavian Journal of Statistics}, 17\penalty0 (4):\penalty0
  333--353, 1990.

\bibitem[Golumbic(1980)]{Golumbic}
M.~C. Golumbic.
\newblock \emph{Algorithmic Graph Theory and Perfect Graphs}.
\newblock Academic Press, 1980.

\bibitem[Javidian and Valtorta(2018{\natexlab{a}})]{jv1}
M.~A. Javidian and M.~Valtorta.
\newblock On the properties of \textrm{MVR} chain graphs.
\newblock In \emph{Workshop proceedings of the 9th International Conference on
  Probabilistic Graphical Models}, pages 13--24, 2018{\natexlab{a}}.

\bibitem[Javidian and Valtorta(2018{\natexlab{b}})]{jv2}
M.~A. Javidian and M.~Valtorta.
\newblock Finding minimal separators in ancestral graphs.
\newblock In \emph{Seventh Causal Inference Workshop at the 34th Conference on
  Artifical Intelligence (UAI-18)}, 2018{\natexlab{b}}.

\bibitem[Javidian and Valtorta(2019)]{jv3}
M.~A. Javidian and M.~Valtorta.
\newblock Supplementary materials for "structural learning of multivariate
  regression chain graphs via decomposition".
\newblock
  \underline{\textcolor{blue}{\href{https://www.dropbox.com/sh/iynnlwyu8il7m3v/AACk8SyIEn7s-W9NRlLnz0DDa?dl=0}{link}}},
  2019.

\bibitem[Lauritzen(1996)]{l}
S.~Lauritzen.
\newblock \emph{Graphical Models}.
\newblock Oxford Science Publications, 1996.

\bibitem[Lauritzen and Wermuth(1989)]{lw}
S.~Lauritzen and N.~Wermuth.
\newblock Graphical models for associations between variables, some of which
  are qualitative and some quantitative.
\newblock \emph{The Annals of Statistics}, 17\penalty0 (1):\penalty0 31--57,
  1989.

\bibitem[Lauritzen and Spiegelhalter(1988)]{asia}
S.~L. Lauritzen and D.~J. Spiegelhalter.
\newblock Local computations with probabilities on graphical structures and
  their application to expert systems.
\newblock \emph{Journal of the Royal Statistical Society. Series B
  (Methodological)}, 50\penalty0 (2):\penalty0 157--224, 1988.

\bibitem[Ma et~al.(2008)Ma, Xie, and Geng]{mxg}
Z.~Ma, X.~Xie, and Z.~Geng.
\newblock Structural learning of chain graphs via decomposition.
\newblock \emph{Journal of Machine Learning Research}, 9:\penalty0 2847--2880,
  2008.

\bibitem[Marchetti and Lupparelli(2011)]{ml2}
G.~Marchetti and M.~Lupparelli.
\newblock Chain graph models of multivariate regression type for categorical
  data.
\newblock \emph{Bernoulli}, 17\penalty0 (3):\penalty0 827--844, 2011.

\bibitem[Pearl(1988)]{pearl1}
J.~Pearl.
\newblock \emph{Probabilistic Reasoning in Intelligent Systems: Networks of
  Plausible Inference}.
\newblock Morgan Kaufmann Publishers Inc. San Francisco, CA, USA, 1988.

\bibitem[Pearl(2009)]{pj}
J.~Pearl.
\newblock \emph{Causality. Models, reasoning, and inference}.
\newblock Cambridge University Press, 2009.

\bibitem[Pe{\~n}a(2014{\natexlab{a}})]{Addendum}
J.~M. Pe{\~n}a.
\newblock Learning multivariate regression chain graphs under faithfulness:
  Addendum.
\newblock Available at the author's website, 2014{\natexlab{a}}.

\bibitem[Pe{\~n}a(2014{\natexlab{b}})]{p1}
J.~M. Pe{\~n}a.
\newblock Learning marginal {AMP} chain graphs under faithfulness.
\newblock \emph{European Workshop on Probabilistic Graphical Models PGM:
  Probabilistic Graphical Models}, pages 382--395, 2014{\natexlab{b}}.

\bibitem[Pe{\~n}a(2018)]{p3}
J.~M. Pe{\~n}a.
\newblock Reasoning with alternative acyclic directed mixed graphs.
\newblock \emph{Behaviormetrika}, pages 1--34, 2018.

\bibitem[Pe{\~n}a et~al.(2014)Pe{\~n}a, Sonntag, and Nielsen]{psn}
J.~M. Pe{\~n}a, D.~Sonntag, and J.~Nielsen.
\newblock An inclusion optimal algorithm for chain graph structure learning.
\newblock \emph{In Proceedings of the 17th International Conference on
  Artificial Intelligence and Statistics}, pages 778--786, 2014.

\bibitem[Richardson(2003)]{r2}
T.~S. Richardson.
\newblock Markov properties for acyclic directed mixed graphs.
\newblock \emph{Scandinavian Journal of Statistics}, 30\penalty0 (1):\penalty0
  145--157, 2003.

\bibitem[Richardson and Spirtes(2002)]{rs}
T.~S. Richardson and P.~Spirtes.
\newblock Ancestral graph markov models.
\newblock \emph{The Annals of Statistics}, 30\penalty0 (4):\penalty0 962--1030,
  2002.

\bibitem[Sonntag(2014)]{s}
D.~Sonntag.
\newblock \emph{A Study of Chain Graph Interpretations (Licentiate
  dissertation)[\url{https://doi.org/10.3384/lic.diva-105024}]}.
\newblock Link{\"o}ping University, 2014.

\bibitem[Sonntag(2016)]{s2}
D.~Sonntag.
\newblock \emph{Chain Graphs: Interpretations, Expressiveness and Learning
  Algorithms}.
\newblock PhD thesis, Link{\"o}ping University, 2016.

\bibitem[Sonntag and Pe{\~n}a(2012)]{sp}
D.~Sonntag and J.~M. Pe{\~n}a.
\newblock Learning multivariate regression chain graphs under faithfulness.
\newblock \emph{Proceedings of the 6th European Workshop on Probabilistic
  Graphical Models}, pages 299--306, 2012.

\bibitem[Sonntag and Pe{\~{n}}a(2015)]{Sonntag2015}
D.~Sonntag and J.~M. Pe{\~{n}}a.
\newblock Chain graphs and gene networks.
\newblock In A.~Hommersom and P.~J. Lucas, editors, \emph{Foundations of
  Biomedical Knowledge Representation: Methods and Applications}, pages
  159--178. Springer, 2015.

\bibitem[Sonntag et~al.(2015{\natexlab{a}})Sonntag, J{\~a}rvisalo, Pe{\~n}a,
  and Hyttinen]{sjph}
D.~Sonntag, M.~J{\~a}rvisalo, J.~M. Pe{\~n}a, and A.~Hyttinen.
\newblock Learning optimal chain graphs with answer set programming.
\newblock In \emph{Proceedings of the 31st Conference on Uncertainty in
  Artificial Intelligence}, pages 822--831, 2015{\natexlab{a}}.

\bibitem[Sonntag et~al.(2015{\natexlab{b}})Sonntag, Peña, and
  Gómez-Olmedo]{essentialmvrcgs}
D.~Sonntag, J.~M. Peña, and M.~Gómez-Olmedo.
\newblock Approximate counting of graphical models via mcmc revisited.
\newblock \emph{International Journal of Intelligent Systems}, 30\penalty0
  (3):\penalty0 384--420, 2015{\natexlab{b}}.

\bibitem[Spirtes et~al.(2000)Spirtes, Glymour, and Scheines]{sgs}
P.~Spirtes, C.~Glymour, and R.~Scheines.
\newblock \emph{Causation, Prediction and Search, second ed.}
\newblock MIT Press, Cambridge, MA., 2000.

\bibitem[Studen{\'y}(1997)]{srs}
M.~Studen{\'y}.
\newblock A recovery algorithm for chain graphs.
\newblock \emph{International Journal of Approximate Reasoning}, 17:\penalty0
  265--293, 1997.

\bibitem[Tsamardinos et~al.(2003)Tsamardinos, Aliferis, and Statnikov]{tas}
I.~Tsamardinos, C.~F. Aliferis, and A.~Statnikov.
\newblock Time and sample efficient discovery of markov blankets and direct
  causal relations.
\newblock \emph{The Ninth ACM SIGKDD International Conference on Knowledge
  Discovery and Data Mining}, pages 673--678, 2003.

\bibitem[Tsamardinos et~al.(2006)Tsamardinos, , Brown, and
  Aliferis]{Tsamardinos2006}
I.~Tsamardinos, , L.~E. Brown, and C.~F. Aliferis.
\newblock The max-min hill-climbing bayesian network structure learning
  algorithm.
\newblock \emph{Machine Learning}, 65\penalty0 (1):\penalty0 31--78, Oct 2006.

\bibitem[Wermuth and Sadeghi(2012)]{ws}
N.~Wermuth and K.~Sadeghi.
\newblock Sequences of regressions and their independences.
\newblock \emph{Test}, 21:\penalty0 215--252, 2012.

\bibitem[Xie et~al.(2006)Xie, Zheng, and Zhao]{xie}
X.~Xie, Z.~Zheng, and Q.~Zhao.
\newblock Decomposition of structural learning about directed acyclic graphs.
\newblock \emph{Artificial Intelligence}, 170\penalty0 (4-5):\penalty0
  422--439, 2006.

\end{thebibliography}

%\begin{biography}[example-image-1x1]{A.~One}
%Please check with the journal's author guidelines whether author biographies are required. They are usually only included for review-type articles, and typically require photos and brief biographies (up to 75 words) for each author.
%\bigskip
%\bigskip
%\end{biography}

%\graphicalabstract{example-image-1x1}{Please check the journal's author guildines for whether a graphical abstract, key points, new findings, or other items are required for display in the Table of Contents.}

\end{document}